%% file: next_token_arxiv.tex
\newif\ifarxiv 
\arxivtrue
\ifarxiv

\documentclass[11pt]{article}
\usepackage[left=1in, top=1in, bottom=1in, right=1in]{geometry}
\usepackage{authblk}

\usepackage[utf8]{inputenc} 
\usepackage[T1]{fontenc}    
\usepackage{url}            
\usepackage{booktabs}       
\usepackage{amsfonts}       
\usepackage{nicefrac}       
\usepackage{microtype}      
\usepackage{xcolor}         
\title{Mechanics of Next Token Prediction with Self-Attention}
\author{Yingcong Li\thanks{Equal contribution.}\, $^1$ \quad Yixiao Huang$^*$\,$^1$ \quad M. Emrullah Ildiz$^1$ \\
{Ankit Singh Rawat$^2$ \quad Samet Oymak$^1$}
}
\affil{$^1$ University of Michigan, Ann Arbor \\
{\small \texttt{\{yingcong,eildiz,oymak\}@umich.edu, yixiao.huang@my.cityu.edu.hk} 
\vspace{-3mm}}
}
\affil{$^2$ Google Research NYC \\
{\small \texttt{ankitsrawat@google.com} 
\vspace{-3mm}}
} 
\else
\documentclass[twoside]{article}
\usepackage[accepted]{aistats2024}

\fi
\input{notation}

\date{}
\begin{document}

\ifarxiv

\maketitle
\else
\runningauthor{Yingcong Li, Yixiao Huang,  M. Emrullah Ildiz, Ankit Singh Rawat, Samet Oymak}
\twocolumn[

\aistatstitle{Mechanics of Next Token Prediction with Self-Attention}

\aistatsauthor{Yingcong Li$^{*\,1}$ \And Yixiao Huang$^{*\,1}$ \And  M. Emrullah Ildiz$^1$ \And Ankit Singh Rawat$^2$ \And Samet Oymak$^1$}

\aistatsaddress{University of Michigan, Ann Arbor$^1$ \And Google Research NYC$^2$
} ]
\fi
\vspace{-20pt}
\begin{abstract}

Transformer-based language models are trained on large datasets to predict the next token given an input sequence. Despite this simple training objective, they have led to revolutionary advances in natural language processing. Underlying this success is the self-attention mechanism. In this work, we ask: \emph{What does a single self-attention layer learn from next-token prediction?} We show that training self-attention with gradient descent learns an automaton which generates the next token in two distinct steps: \textbf{(1) Hard retrieval:} Given input sequence, self-attention precisely selects the \emph{high-priority input tokens} associated with the last input token. \textbf{(2) Soft composition:} It then creates a convex combination of the high-priority tokens from which the next token can be sampled. Under suitable conditions, we rigorously characterize these mechanics through a directed graph over tokens extracted from the training data. We prove that gradient descent implicitly discovers the strongly-connected components (SCC) of this graph and self-attention learns to retrieve the tokens that belong to the highest-priority SCC available in the context window. Our theory relies on decomposing the model weights into a directional component and a finite component that correspond to hard retrieval and soft composition steps respectively. This also formalizes a related implicit bias formula conjectured in [Tarzanagh et al.~2023]. We hope that these findings shed light on how self-attention processes sequential data and pave the path toward demystifying more complex architectures.
\vspace{-10pt}

\end{abstract}
\input{sec/intro}

\input{sec/setup}
\input{sec/global_gd}

\input{sec/reg_path}
\input{sec/fig_local}

\input{sec/local}
\input{sec/related}

\input{sec/discuss}



\bibliographystyle{iclr2024_conference}
\bibliography{references}

\onecolumn
\newpage 
\appendix
\newpage
\ifarxiv
\else
\aistatstitle{Mechanics of Next Token Prediction with Transformers \\
Supplementary Materials}
\input{supp/check_lst}
\newpage
\fi
\DoToC
\input{supp/support}
\input{supp/gd_proof}

\input{supp/reg_path_proof}

\input{supp/add_exp}

\end{document}

%% file: notation.tex
\usepackage[utf8]{inputenc} 
\usepackage{natbib}
\usepackage[T1]{fontenc}    
\usepackage{titletoc}
\usepackage{hyperref}       
\usepackage{url}            
\usepackage{booktabs}       
\usepackage{amsfonts,amsmath,amssymb}       
\usepackage{nicefrac}       
\usepackage{microtype}      
\usepackage{xcolor}         
\usepackage{enumitem}
\usepackage{mathtools}

\usepackage{graphicx}
\usepackage{subcaption}
\usepackage[normalem]{ulem}
\usepackage[textwidth=1.8cm]{todonotes}
\usepackage{ragged2e}

\usepackage{wrapfig,bm,comment,color}
\usepackage{breakurl,epsfig,epsf,fmtcount,semtrans,multirow,multicol,boldline}
\usepackage{tcolorbox}
\tcbuselibrary{skins}
\usepackage{tikz}
\usepackage{pifont}

\usepackage{appendix}
\usepackage{xcolor}

\definecolor{darkred}{RGB}{150,0,0}
\definecolor{darkgreen}{RGB}{0,150,0}
\definecolor{darkblue}{RGB}{0,0,200}
\hypersetup{colorlinks=true, linkcolor=darkred, citecolor=darkgreen, urlcolor=darkblue}

\newtheorem{theorem}{Theorem}

\newtheorem{assumption}{Assumption}

\newtheorem{lemma}{Lemma}

\newtheorem{definition}{Definition}

\def \endprf{\hfill {\vrule height6pt width6pt depth0pt}\medskip}

\newenvironment{proof}{\noindent {\bf Proof.} }{\endprf\par}
\newenvironment{proofof}[1]{\noindent {\bf Proof of {#1}.} }{\endprf\par}

\newcommand\DoToC{%
  \startcontents
  \printcontents{}{1}{\textbf{Contents}\vskip3pt\hrule\vskip5pt}
  \vskip3pt\hrule\vskip5pt
}

\newcommand{\cln}[1]{\red{}}

\newcommand{\ch}{c_\text{up}}
\newcommand{\cl}{c_\text{dn}}

\newcommand{\cmin}{c_\text{min}}
\newcommand{\cmax}{c_\text{max}}

\newcommand{\hb}{\vct{h}}

\newcommand{\st}{\star}

\newcommand{\pl}{\parallel}

\newcommand{\beq}{\begin{equation}}
\newcommand{\ba}{\begin{align}}
\newcommand{\ea}{\end{align}}

\newcommand{\eeq}{\end{equation}}

\newcommand{\nn}{\nonumber}

\newcommand{\V}{{\mtx{V}}}

\newcommand{\diag}[1]{\text{diag}(#1)}

\newcommand{\Lc}{{\cal{L}}}
\newcommand{\Oc}{{\cal{O}}}
\newcommand{\Ocb}{\bar{\Oc}}

\newcommand{\Lcb}{{\bar{\cal{L}}}}

\newcommand{\smx}{\s^{\max}}

\newcommand{\Tmax}{T_\text{max}}

\newcommand{\Cb}{{\mtx{C}}}

\newcommand{\Eb}{{\mtx{E}}}

\newcommand{\Gc}{{\cal{G}}}

\newcommand{\Gck}{{\cal{G}}^{(k)}}

\newcommand{\Iden}{{\mtx{I}}}
\newcommand{\M}{{\mtx{M}}}

\newcommand{\z}{{\vct{z}}}

\newcommand{\sft}[1]{\mathbb{S}(#1)}
\newcommand{\sfp}[1]{\mathbb{S}'(#1)}
\newcommand{\tn}[1]{\|{#1}\|_{\ell_2}}

\newcommand{\tf}[1]{\left\|{#1}\right\|_{F}}

\newcommand{\Cc}{\mathcal{C}}
\newcommand{\Cck}{\mathcal{C}^{(k)}}

\newcommand{\Rc}{\mathcal{R}}

\newcommand{\Rcb}{\bar{\mathcal{R}}}

\newcommand{\bt}{{\boldsymbol{\beta}}}

\newcommand{\bal}{{\boldsymbol{\alpha}}}
\newcommand{\bgam}{{\boldsymbol{\gamma}}}
\newcommand{\bgammax}{{\gamma}^{\text{max}}}

\newcommand{\Sc}{\mathcal{S}}

\newcommand{\vb}{\vct{v}}

\newcommand{\Ic}{{\mathcal{I}}}
\newcommand{\Icb}{\bar{\mathcal{I}}}

\newcommand{\Xb}{\bar{\X}}
\newcommand{\xb}{\vct{\bar{x}}}

\newcommand{\cb}{\mtx{c}}

\newcommand{\w}{\vct{w}}

\newcommand{\Wm}{\W^\text{svm}}
\newcommand{\tWm}{\widetilde\W^\text{svm}}
\newcommand{\Wgd}{\W^\text{GD}}

\newcommand{\li}{\left<}
\newcommand{\ri}{\right>}
\newcommand{\s}{\vct{s}}
\newcommand{\ab}{\vct{a}}
\newcommand{\bb}{\vct{b}}
\newcommand{\ub}{{\vct{u}}}

\newcommand{\Xc}{\mathcal{X}}
\newcommand{\Yc}{\mathcal{Y}}

\newcommand{\Wb}{\bar{\W}}

\newcommand{\op}{\texttt{opt}}

\newcommand{\data}{\texttt{DSET}}

\newcommand{\bdata}{\overline{\texttt{DSET}}}

\newcommand{\datak}{\texttt{DSET}^{(k)}}

\newcommand{\x}{\vct{x}}
\newcommand{\xl}{\vct{\bar{x}}}
\newcommand{\xli}{\vct{\bar{x}}_{i}}

\newcommand{\y}{\vct{y}}

\newcommand{\W}{\mtx{W}}
\newcommand{\Wf}{\mtx{W}^{\text{fin}}}
\newcommand{\Whard}{\W_{\text{hard}}}
\newcommand{\Wsoft}{\W_{\text{soft}}}
\newcommand{\hWf}{\widehat{\mtx{W}}^{\text{fin}}}
\newcommand{\bWf}{{\mtx{\bar W}}^{\text{fin}}}
\newcommand{\tWf}{\widetilde{\mtx{W}}^{\text{fin}}}

\newcommand{\Wc}{{\cal{W}}}

\newcommand{\prj}{\boldsymbol{\Pi}}
\newcommand{\Scf}{\Sc_{\text{fin}}}
\newcommand{\Scsvm}{{\Sc_{\text{svm}}}}
\newcommand{\Scall}{\Sc_{\text{active}}}
\newcommand{\tScf}{\widetilde{\Sc}_{\text{fin}}}

\newcommand{\Wcf}{\Wc^{\text{fin}}}

\definecolor{emmanuel}{RGB}{255,127,0}

\newcommand{\R}{\mathbb{R}}

\newcommand{\<}{\langle}
\renewcommand{\>}{\rangle}

\newcommand{\grad}[1]{{\nabla\Lc(#1)}}

\newcommand{\e}{\mathrm{e}}
\newcommand{\eb}{\vct{e}}

\newcommand{\vct}[1]{\bm{#1}}
\newcommand{\mtx}[1]{\bm{#1}}

\newcommand{\X}{{\mtx{X}}}

\newif\ifdraft
\draftfalse

\ifdraft
    \newcommand{\asrtodo}[1]{\todo[color=cyan,size=\tiny]{AR: #1}}
    \newcommand{\asr}[1]{\textsf{\textcolor{red}{[\textbf{Ankit}: #1]}}}
    \newcommand{\yl}[1]{\textcolor{orange}{#1}}
    \newcommand{\needreview}[1]{\textcolor{cyan}{#1}}
    \newcommand{\ylm}[1]{\marginpar{\color{orange}\tiny\ttfamily YL: #1}}
    \newcommand{\ankit}[1]{\textsf{\textcolor{red}{[\textbf{ASR}: #1]}}}
    \newcommand{\shaw}[1]{\textcolor{violet}{#1}}
    \newcommand{\so}[1]{\textcolor{darkblue}{#1}}
    \newcommand{\SO}[1]{\textcolor{darkred}{[SO: #1]}}
    \newcommand{\red}{\textcolor{darkred}}

\else 
    \newcommand{\asrtodo}[1]{}
    \newcommand{\asr}[1]{}
    \newcommand{\yl}[1]{{#1}}
    \newcommand{\needreview}[1]{{#1}}
    \newcommand{\ylm}[1]{}
    \newcommand{\ankit}[1]{}
    \newcommand{\shaw}[1]{}
    \newcommand{\so}[1]{}
    \newcommand{\SO}[1]{}
    \newcommand{\red}{}

\fi

%% file: sec/intro.tex


\input{sec/fig_gen}

\ifarxiv
\section{Introduction}
\else
\section{INTRODUCTION}
\fi
Language modeling as enabled by Transformer architecture~\citep{vaswani2017} and seemingly simple 
training objectives such as next-token prediction~\citep{radford2018_gpt1,radford2019_gpt2} have not only led to breakthroughs in the field of natural language processing (NLP)~\citep{brown2020_gpt3,chowdhery2022_palm,gpt4_techreport,touvron2023_llama}, but rather straightforward adaptations of this symbiosis between Transformers and next-token prediction tasks have also realized remarkable performance in other domains, including vision~\citep{chen2020_pixels}, speech~\citep{chung2020generative}, reinforcement learning~\citep{chen2021_decision}, and even protein design~\citep{ferruz2022_protgpt2,nijkamp2022_progen2}. This widespread empirical success is often attributed to the (self-)attention mechanism of Transformers that produces high-quality contextual representations needed to realize excellent prediction performance in a wide range of domains. However, a rigorous understanding of how Transformers 
can learn such high-quality representations by solving next-token prediction task via natural algorithms such as gradient descent is largely missing from the literature.

This work aims to bridge this gap between the empirical success and principled understanding of Transformer-based language modeling by shedding light on the optimization landscape and key implicit biases faced by the self-attention mechanism in solving the next-token prediction task. In particular, focusing on a \textit{single-layer} self-attention model with linear classification head, and solving the next-token prediction task, 
we consider the following questions:
\begin{itemize}
    \item \textit{What relationships in the training data are captured by the single-layer self-attention model?}
    \item \textit{How exactly do these relationships dictate the optimization geometry of natural algorithms such as gradient descent?}
\end{itemize}

We show that the answers to both of these questions are intertwined which we achieve by significantly expanding the recently proposed framework that connects learning with Transformers to the celebrated support vector machines (SVMs)~\citep{tarzanagh2023margin,tarzanagh2023transformers}. 

\yl{
As illustrated in Figure~\ref{fig:next-token}, given training data as a collection of (input sequence, next token) pairs, self-attention model learns to (1) retrieve the high-priority tokens (highlighted with red color) to the last input token; and then (2) build a convex combination of these high-priority tokens. The notion of \emph{high-priority} is dictated by a directed graph learned from training data. SGD training accomplishes this by learning hard and soft components of the attention weights $\W$ to execute (1) and (2) respectively. Concretely, the following theorem dictates the evolution of attention weights during gradient descent.
}
\begin{theorem}[\emph{informal}]\label{thm informal}Consider training a single-layer self-attention model with gradient descent. The combined attention weights $\W:=\W_K\W_Q^\top$ evolve as 
    \[
    \W_{\text{GD}}\approx C\cdot\Whard+\Wsoft,
    \]
    where $C\cdot\Whard$ is the hard retrieval component selecting the high-priority tokens when $C\to\infty$; and $\Wsoft$ is the soft composition component  allocating nonzero softmax probabilities over selected tokens. 
\end{theorem}

\yl{To capture the priority order among different tokens as observed in Figure~\ref{fig:next-token}, we construct directed graphs among the tokens in the vocabulary, namely \textit{token-priority graphs} (TPGs). An illustration is provided in Figure~\ref{fig:scc}, where a \textit{strongly connected component} (SCC, highlighted as dashed black rectangles) in a TPG corresponds to the tokens that are reachable from each other, indicating the absence of a strict priority among those tokens. 
The hard retrieval component $\W_{\text{hard}}$ captures the topological order of different SCCs (orange arrows) whereas the soft composition component $\W_{\text{soft}}$ captures the relationships of different tokens within each SCC (black arrows). These TPGs will geometrically capture the learning dynamics of self-attention. Specifically, we propose the SVM problem \eqref{graph svm}, solution of which describes the direction gradient descent converges to. This way, SGD asymptotically enforces the topological order between SCCs i.e. the $C\cdot\Whard$ term in Theorem \ref{thm informal} as $C\rightarrow \infty$. In practice, this implies that self-attention model favors suppressing lower priority tokens in favour of sampling higher priority tokens.}




A conjecture on the decomposition in Theorem \ref{thm informal} was first proposed in \citep{tarzanagh2023transformers}\footnote{Their conjecture aims to characterize the impact of the MLP layer that follow self-attention in a binary classification setting. However, the high-level claim is same.}. This decomposition is also related to the implicit bias of logistic regression on non-separable data \citep{ji2019implicit}. Our theory fully formalizes this decomposition under the next-token prediction setting and reveals fundamental connections to graphical structure in data (e.g.~through SCCs, TPGs).

\input{sec/fig_scc}

Overall, we carefully study the gradient descent and regularization path algorithms for attention-based next-token prediction and make the following contributions: 
\begin{enumerate}
    \item We study the optimization landscape of self-attention with log-loss and show that the problem is convex under suitable assumptions. We then establish a global convergence result to fully formalize Theorem \ref{thm informal} in terms of a directional component \eqref{graph svm} and a finite component (see~Sec~\ref{sec global gd}). Notably, results apply to arbitrary datasets as we don't require distributional assumptions.
    
    \item Our theory reveals insightful connections between continuous and discrete optimization, namely: \emph{Self-attention implicitly discovers the strongly-connected components of the TPGs during training}.
    \item  Under a general setting, we establish the implicit bias of the solution obtained by vanishing regularization~\citep{rosset2003margin, suggala2018connecting, ji2020gradient}. This yields a result similar to the gradient descent theory (see~Sec~\ref{sec rp}). We also show that, in general, gradient descent can exhibit local directional convergence rather than global. We characterize these local directions through the SVM solutions of \emph{pseudo TPGs} (see~Sec~\ref{sec:local-gd}).
    
\end{enumerate}

%% file: sec/fig_gen.tex
\ifarxiv
\begin{figure*}[ht]
\else
\begin{figure*}[tb]
\fi
    \centering    
    \includegraphics[width=0.8\textwidth]{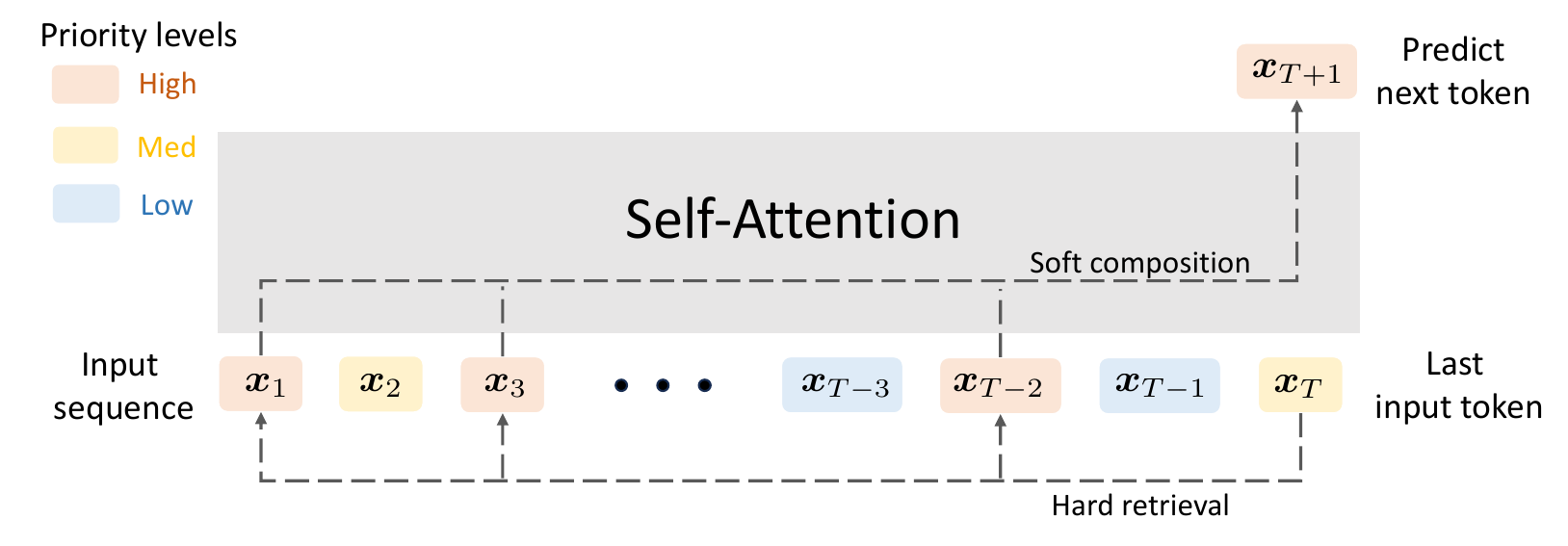}
    \vspace{-7pt}
    \caption{\small{Overview of our result on next-token prediction. We study the implicit bias of gradient descent where a 1-layer self-attention model is trained until convergence. We prove that, during test-time, this model implements a hard retrieval to precisely select the high-priority tokens and then outputs a convex combination of these as the output from which the next token can be sampled. The notion of \emph{high-priority} is formalized through the strongly-connected components of a directed graph associated to the last input token. 
    %
    }}
    \vspace{-6pt}
    \label{fig:next-token}
\end{figure*}

%% file: sec/fig_scc.tex
\begin{figure*}[tb] 
    \centering    \includegraphics[scale=0.4]{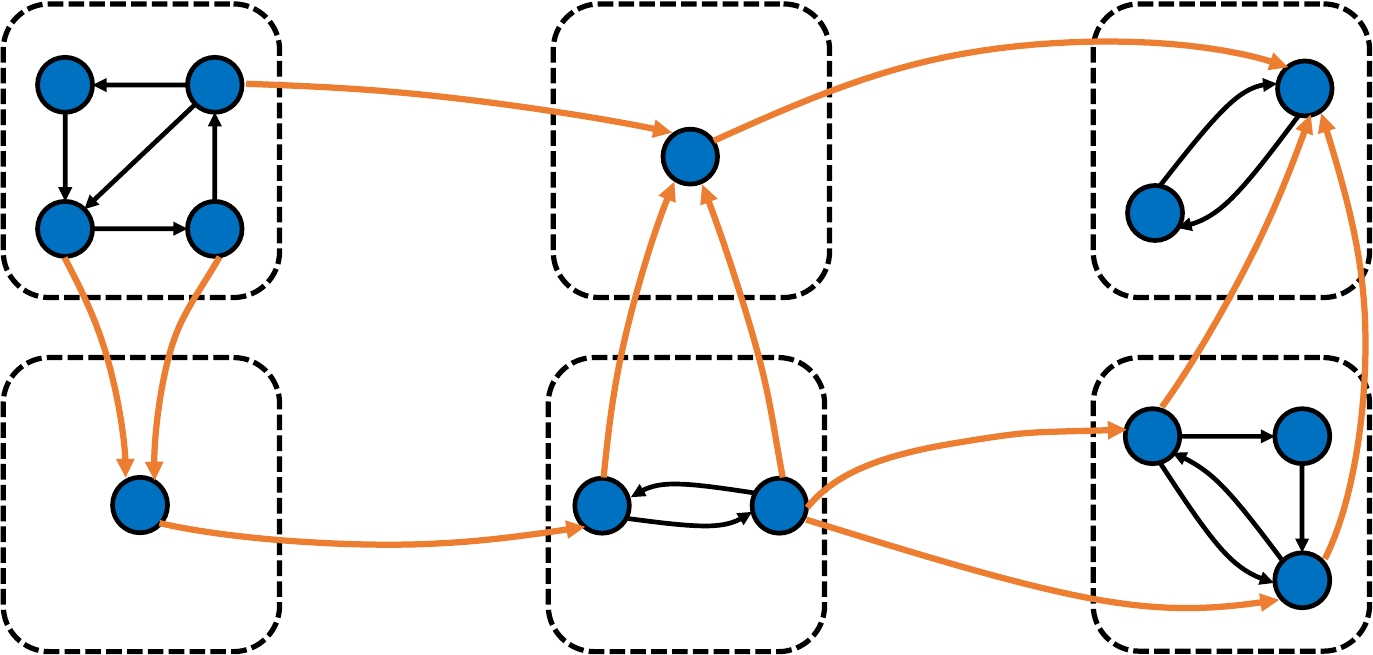}
    \caption{\small{\yl{A token-priority graph (TPG) is a directed graph derived from training data (see Sec \ref{sec ntg} for definition). The edges in TPG capture the input-output relationships between different tokens. A TPG can be partitioned into several SCCs depicted as dashed black squares. In light of Theorem \ref{thm informal}, black intra-SCC edges within each SCC induce the soft-composition component of the attention weights whereas the orange edges induce the hard-retrieval component enforcing the priority orders among various SCCs. }}}
    \label{fig:scc}
\end{figure*}

%% file: sec/setup.tex
\ifarxiv
\section{Problem Setup}\label{sec setup}
\else
\section{PROBLEM SETUP}\label{sec setup}
\fi

\noindent\emph{Notation.} Let $[n]$ denote the set $\{1,\cdots,  n\}$. For a space $\Sc$, let $\Sc^\perp$ denote the orthogonal complement of $\Sc$ and $\prj_\Sc$ denote the projection operator on $\Sc$ with respect to Euclidean distance.

\smallskip
\noindent\textbf{Next-token prediction problem.} Let $K$ be the vocabulary size with $\Eb=\begin{bmatrix}\eb_1~\dots~\eb_K\end{bmatrix}^\top\in\R^{K\times d}$ denoting the embedding matrix consisting of $d$-dimensional token embeddings for the $K$ tokens in the vocabulary. The next-token prediction is a multi-class classification problem and the goal is to predict the ID $y\in[K]$ of the next token given an input sequence $\X=[\x_1~\dots~\x_T]^\top\in\R^{T\times d}$, where $\x_t\in\Eb$ for all $t\in[T]$.

Suppose that we have a training dataset $\data=\big\{(\X_i,y_i)\in \R^{T_i\times d}\times [K]\big\}_{i=1}^n$ consisting of $n$ sequences where we allow the sequences to have different lengths $T_i$, $i\in[n]$. Throughout this paper, we use $x_{it} \in [K]$ to denote the scalar token ID corresponding to the $t$-th token $\x_{it} \in \R^{d}$ of the input sequence $\X_i$, i.e., $\x_{it}=\eb_{x_{it}}$.  


\smallskip
\noindent\textbf{Self-attention model.} We 
consider a \textit{single-layer} self-attention model when making a prediction on a given input sequence $\X\in\R^{T\times d}$. Following the previous work \citep{tarzanagh2023transformers}, we denote the combined key-query weights by a trainable $\W\in\R^{d\times d}$ matrix, and assume identity value matrix. Let $\xl:=\x_T$ be the last token of the input sequence $\X$. 
Then, the single-layer self-attention outputs the following embedding to predict the next-token ID $y$:
\begin{align} 
\label{eq:sattn-def}
f_{\W}(\X)=\X^\top \sft{\X\W\xl},
\end{align}
where $\sft{\cdot}$ denotes the softmax operation which facilitates weighing tokens of $\X$ based on the data-dependent probabilities $\sft{\X\W\xl}$. Note that the output embedding $f_{\W}(\X)\in\R^d$ in \eqref{eq:sattn-def} is a weighted linear combination of the input token embeddings in $\X$.

\smallskip
\noindent\textbf{Empirical risk minimization (ERM) problem.} 
Let $\ell:\R\to\R$ be the loss function. Given training dataset $\data$, we consider the ERM problem with the following objective:
\begin{align}
\Lc(\W)=\frac{1}{n}\sum_{i=1}^n \ell(\cb_{y_i}^\top\X_i^\top \sft{\X_i\W\xl_i}).\label{erm}\tag{ERM}
\end{align}
Throughout this paper, we fix the linear classification head $\cb_{k}$\footnote{{Specifically, we assume well-pretrained head $\cb_1,\cdots,\cb_K$ such that $\ell(\cb_y^\top\eb_k)$ returns the minimal risk when $k=y$.} } and assume $\|{\cb_k}\|$ is bounded for all $k\in[K]$. Note that even though the classification head is linear, the problem of learning attention parameters $\W$ via ERM is not necessarily convex due to the softmax operator. 
%
In this work, we focus on this exact problem and consider the following two algorithms to optimze for $\W$:
\begin{tcolorbox}[colback=white!5!white,colframe=black!5!black,colback=white!5!white,
                  interior hidden,
                  arc=0pt,
                  boxrule=1pt,
                  boxsep=0pt,
                  left=5pt,
                  right=5pt,
                  top=5pt,]

\noindent\textbf{1.~Gradient descent:} Given starting point $\W(0)\in\R^{d\times d}$ and step size $\eta>0$, for $\tau\geq0$, 
    \begin{align}\label{algo gd}\tag{Algo-GD}
    \W(\tau+1)=\W(\tau)-\eta\nabla\Lc(\W(\tau)).
    \end{align}
\noindent\textbf{2.~Regularization path:} Given $R>0$, $\W\in\R^{d\times d}$,
    \begin{align}\label{algo rp}\tag{Algo-RP}
    \bar\W_R=\arg\min_{\tf{\W}\leq R}\Lc(\W).
    \end{align}
\end{tcolorbox}

{The next-token prediction task aims to capture various patterns present in the underlying dataset. Towards this, we introduce \textit{token-priority graph} (TPG) in the following section that summarizes the sequential priority orders presented in the training data. As we will see later, TPGs play a crucial role in characterizing the optimization geometry for both \eqref{algo gd} 
 and \eqref{algo rp} algorithms.}

\input{sec/fig_intro}
\subsection{Token-priority Graph of the Dataset}\label{sec ntg}


A token-priority graph (TPG) is a directed graph with at most $K$ nodes corresponding to the elements in the vocabulary. We associate the dataset $\data=\{(\X_i,y_i)\}_{i=1}^n$ with multiple TPGs $\{\Gck\}_{k=1}^K$, with each TPG focusing on a subset of the dataset comprising of those input sequences that agree on the last token $\xl$. Concretely, we construct $\Gck$'s as follows:
%
%
\begin{enumerate}
    \item Split $\data$ into $K$ subsets $\{\datak\}_{k=1}^K$ with $\datak$ containing all input sequences that end with the same last token $\xl=\eb_k$.
    \item For each $(\X,y)\in\datak$ and for all $(x,y)$ pairs in $(\X,y)$ where $x$ is the corresponding token ID of $\x\in\X$, add a directed edge $(y\rightarrow x)$ to $\Gck$.
\end{enumerate}
{An illustration is provided in Fig.~\ref{fig:intro}, where we construct two {TPGs} ($\Gc^{(1)}$ and $\Gc^{(2)}$) based on the last tokens (depicted in yellow), and the directed edges are presented as arrows starting from labels (orange) to input tokens (blue/yellow) within each sequence.} Note that nodes of each $\Gck$ constitute a subset of the indices $[K]$. The edges in $\Gck$ capture the priorities across the tokens in {an extended data sequence,} conditioned on the last token of the input 
being $\xl=\eb_k$. 
We will see that if there is a cycle, i.e., {$y\rightarrow x$ and $x\rightarrow y$ are both directionally reachable in the graph}, 
then the self-attention {learnt via next-token prediction task} {can} assign comparable priorities to the tokens $x$ and $y$. In contrast, if $y$ always \emph{dominates} $x$, i.e.,~{$y\to x$ is reachable but $x\not\to y$,} 
then, when $x$ and $y$ are both present in an input sequence, self-attention will be learnt to suppress $x$ and select $y$ through an SVM mechanism along the line of \cite{tarzanagh2023transformers}. 


\smallskip
\noindent\textbf{Strongly-connected components in TPGs.}To formalize the aforementioned SVM mechanism, we need the notion of \emph{strongly-connected components} (SCCs). A directed graph is strongly connected if every node in the graph is reachable from every other node. SCCs of a directed graph form a partition into subgraphs that are themselves strongly connected. Given the {TPGs} $\{\Gck\}_{k=1}^{K}$ associated with the dataset $\data$, we can split the directed graph $\Gck$ into its SCCs, denoted by $\{\Cck_i\}_{i=1}^{N_k}$. Note that the number of SCCs in $\Gck$, as denoted by $N_k$, is at most the number of nodes in $\Gck$, which is upper bounded by the vocabulary size $K$. Furthermore, by definition, different SCCs within a graph consist of distinct nodes, i.e.,~$\Cck_{i}\bigcap\Cck_j=\emptyset$, for $i\neq j$. 
{Now, returning to Fig.~\ref{fig:intro}, {each of the dashed grey rectangle represents an SCC}. $\Gc^{(1)}$(left) contains four SCCs and therefore, all tokens within the graph have strict priority orders. In contrast, $\Gc^{(2)}$(right) consists of two SCCs, with one containing three nodes. Following the arrows, we can see that all the tokens/nodes within this specific SCC are directional reachable. }

{Before formally connecting {TPGs} and their SCCs to the SVM mechanism that enables next-token prediction, we introduce some necessary graph-related notation.} Given a directed graph $\Gc$, for $i,j\in[K]$ such that $i\neq j$:
\begin{itemize}
    \item $i\in \Gc$ denotes that the node $i$ belongs to $\Gc$.
    \item $(i\Rightarrow j)\in\Gc$ denotes that the directed edge $(i\rightarrow j)$ is present in $\Gc$ but {$j\rightarrow i$ is not reachable}.
    \item $(i \asymp j)\in\Gc$ means that both two nodes $(i,j)$ are in the same SCC of $\Gc$.
\end{itemize}
From the construction, for any two distinct nodes $i, j$ in the same TPG, they either satisfy $(i \Rightarrow j)$/$(j\Rightarrow i) \text{ or } (i \asymp j)$. 

\subsection{SVM Bias of Self-attention Learning}
\label{sec:graph-svm}

The main contribution of this paper is to establish the SVM equivalence that captures the optimization geometry of the next-token prediction problem. We will show that the self-attention model learnt via either \eqref{algo gd} or \eqref{algo rp} converges to the solution of an SVM defined by the {TPGs} of the underlying dataset $\data$. In particular, given  $(\Gck_k)_{k=1}^{K}$, 
we introduce the following SVM formulation:
\begin{align}\label{graph svm}\tag{Graph-SVM}
&\Wm=\arg\min_{\W}\tf{\W}\quad\\
&\text{s.t.}~~(\eb_i-\eb_j)^{\top}\W\eb_k
\begin{cases}=0\quad \forall(i\asymp j) \in\Gck\\
\geq 1\quad \forall(i\Rightarrow j)\in\Gck
\end{cases}
\ifarxiv
\text{for all}\quad k\in[K].\nn
\else
\hspace{-5pt}\forall k\in[K].\nn
\fi
\end{align}
Fix last token $\eb_k$, and consider any token IDs $i,j\in[K]$, $i\neq j$. When $(i\Rightarrow j)$, token ID $i$ has a higher \emph{priority} than $j$ and hence, the SVM problem \eqref{graph svm} aims to find a $\W$ such that $\W\eb_k$ achieves strictly higher correlation to token embedding $\eb_i$ than $\eb_j$, {that is, $\eb_i^\top\W\eb_k\geq\eb_j^\top\W\eb_k+1$, and then softmax operation will assign higher probability to the token $i$.} While if $(i\asymp j)$, there is not strict priority order between $i$ and $j$, and hence we set the correlation difference equal to zero to prevent the SVM solution $\W$ from distinguishing them. 
The existence of the solution $\Wm$ ensures the separability of tokens $i$'s from the $j$'s for all pairs $(i\Rightarrow j)\in\Gck$. Additionally, if for all $k\in[K]$, the number of SCCs\footnote{{Note that $N_k=0$ implies that within $\data$, there is not training sample whose input sequence has $\eb_k$ as its last token; or equivalently, $\datak=\emptyset$.}} $N_k\leq1$, then $\Wm=0$. 

\begin{lemma}\label{lemma feasible}
Suppose that the embedding matrix $\Eb$ is full row rank. Then, \eqref{graph svm} is feasible.
\end{lemma}

Next focusing on the nodes $i$, $j$, with $(i\asymp j)$, we introduce the following subspace definition.
\begin{definition}[Cyclic subspace]\label{def cyc subspace}
     Define cyclic subspace $\Scf$ as the span of all matrices $(\eb_i-\eb_j)\eb_k^\top$ for all $(i\asymp j)\in\Gck$ and $k\in[K]$.
\end{definition}
Note that since $\Wm$ satisfies all the ``$=0$'' constraints in \eqref{graph svm}, if \eqref{graph svm} is feasible and $\Wm\neq0$, $\Wm\perp\Scf$.

\subsection{Technical Assumptions}
\label{sec:technical-assumptions}
In what follows, we work with a few assumptions that will make the optimization landscape of the underlying learning problem more benign, and we introduce these assumptions along with their justifications.

\begin{assumption}\label{assume iden} 
For  $\forall y,k\in[K]$, $k\neq y$, $\cb_y^\top\eb_y=1$ and $\cb_y^\top\eb_{k}=0$.
\end{assumption}


This assumption essentially enforces that the rows of the prediction head $\Cb$ are aligned with the corresponding vocabulary embeddings in $\Eb$. This is a variation of the \emph{weight tying} strategy which is commonly employed in language models \citep{press2017using,vaswani2017}. It should be noted that Assumption~\ref{assume iden} implies $K\leq d$, which further establishes the feasibility of \eqref{graph svm} as demonstrated by Lemma~\ref{lemma feasible}. Given objective \eqref{erm} and a decreasing loss function $\ell$, our ideal goal is for the attention \eqref{eq:sattn-def} to output $\eb_y$ which minimizes the training risk.  
Recall that single-layer self-attention outputs a convex combination of the input tokens (cf.~\eqref{eq:sattn-def}). If tokens are linearly independent, the only way model can output {the embedding $\eb_y$ corresponding to the} target label $y$ would be if $\eb_y$ was among the input sequence. 
This motivates the following realizability assumption.
\begin{assumption}
\label{assume realizable} 
For any $(\X,y)\in\data$, the token $\eb_y$ is contained in the input sequence $\X$.
\end{assumption}

In the scenario where $(\X, y)$ is not realizable, 
self-attention would select $\eb_{\hat y}\neq\eb_y$ and the SVM formula would be established via separating $\hat y$ from the other tokens in the sequence instead of the true label $y$. Additionally, when Assumption~\ref{assume iden} holds, the model can only make a random prediction over the output labels (since any output of the self-attention model would result in the same training risk); consequently, such non-realizable examples will not play roles in optimizing $\W$, i.e. $\nabla_{\W} \ell(\cb_{y_i}^\top\X_i^\top \sft{\X_i\W\xli}) = 0$. 

%% file: sec/fig_intro.tex
\begin{figure*}[tb] 
    \centering    \includegraphics[width=0.95\textwidth,trim={0 1cm 0 1cm}, clip]{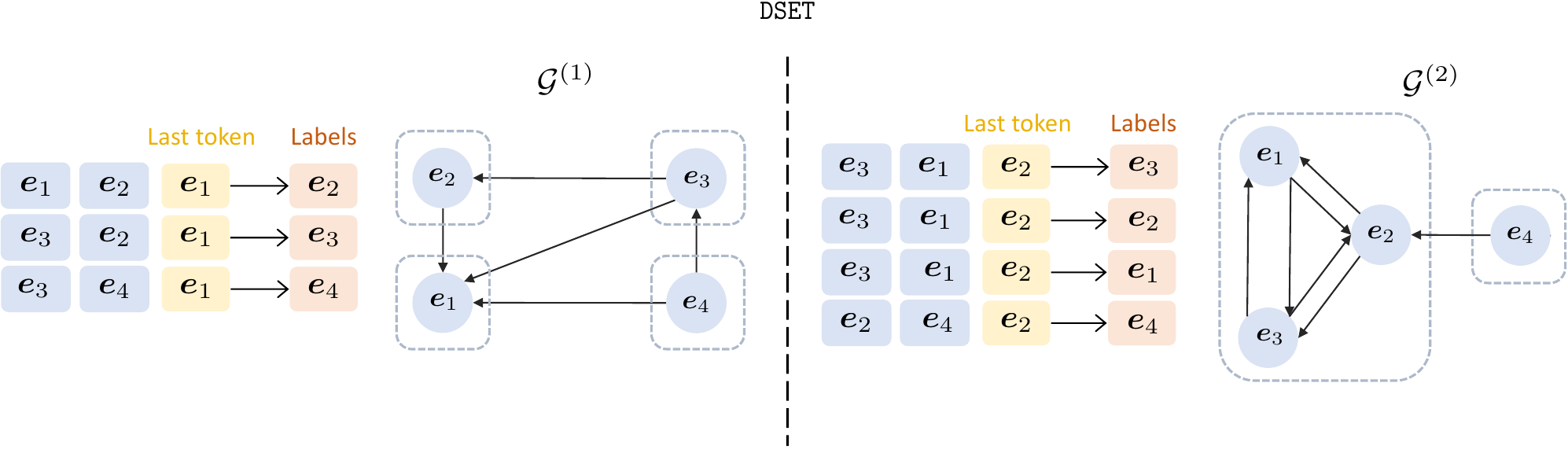}
    \caption{\small{{Illustration of token-priority graph (TPG).  Given the input sequences and labels (next tokens), we construct the TPGs $\{\Gc^{(k)}\}_{k=1}^K$ according to the last token.  
    Two TPGs $\Gc^{(1)}$ (left) and $\Gc^{(2)}$ (right) are constructed using the samples with $\eb_1$ and $\eb_2$ as the last tokens, respectively. In each graph, directed edges (label token$\to$input token) are added between tokens/nodes. Based on these directed edges, each graph can be partitioned into its strongly-connected components (SCCs, {highlighted as dashed grey rectangles}). {Each SCC is a set of tokens where each token is reachable from every other token within that SCC.} Further details are deferred to Section~\ref{sec ntg}.}
    }}
    \label{fig:intro}
\end{figure*}

%% file: sec/global_gd.tex
\ifarxiv
\section{Global Convergence of Gradient Descent}\label{sec global gd}
\else
\section{GLOBAL CONVERGENCE OF GRADIENT DESCENT}\label{sec global gd}
\fi
In this section, we assume the log-loss function, i.e., $\ell(u)=-\log(u)$, and establish the gradient descent convergence of attention weight $\W$ via the convexity of $\Lc(\W)$. Note that although loss function $\ell$ is convex and the classification head is linear, due to the non-convexity of softmax, the convexity of $\Lc(\W)$ is not immediately clear. {Towards this, we introduce the following lemma: }

\begin{lemma} \label{lemma cvx}
    Suppose Assumptions \ref{assume iden} and \ref{assume realizable} hold and consider the log-loss $\ell(u) = -\log(u)$, then $\Lc(\W)$ is convex. Furthermore, $\Lc(\W)$ is strictly convex on $ \Scf$.
\end{lemma}


Let $(\X,y)\in\data$ be any sample and set $\gamma_t=\cb_y^\top\x_t$. Assumption~\ref{assume iden} guarantees that $\gamma_t=1$ when $x_t=y$, otherwise $\gamma_t=0$. Consider the attention output $\cb_y^\top\X^\top\sft{\X\W\xl}$ (cf.~\eqref{erm}) and let softmax probabilities be $s_t=\sft{\X\W\xl}_t$, where $\sum_{t}s_t=1$. Then, the loss of this single sample $\X$ is $\ell(\bar\gamma)$ where $\bar\gamma=\sum_{t}\gamma_ts_t=\sum_{x_t=y}s_t$.  Given log-loss, note that when $\bar\gamma\to0^+$, $-\log(\bar\gamma)$ results in the infinite loss, which suggests that, once attention weight $\W$ diverges to saturate the softmax probability, the finite training risk is achievable only when $s_t\not\to0$ for all $t$ satisfying $x_t=y$. Given different label $y$ for different input and recalling the SCC definition in Section~\ref{sec ntg}, attention selects all $\x_t$'s within the same SCC as $y$.
The following result characterizes the global directional convergence of the GD iterates to the solution of \eqref{graph svm}.

\begin{theorem}\label{thm cyclic gd} Consider a dataset $\data$ and suppose Assumptions~\ref{assume iden} and \ref{assume realizable} hold. Set loss function as $\ell(u)=-\log(u)$. 
Let $\Wm{\in\Scf^\perp}$ be the solution of \eqref{graph svm}. Starting from any $\W(0)$ with constant step size $\eta$, the algorithm \ref{algo gd} 
satisfies {$\lim_{\tau\to\infty}\tf{\W(\tau)}=\infty$ and}
\begin{align}
\qquad\lim_{\tau\rightarrow\infty}\prj_{\Scf}(\W(\tau))=\Wf.\label{gd}
\end{align}
Here $\Wf$ is the unique finite minima of the loss $\tilde{\Lc}(\W):=\lim_{R\rightarrow\infty}\Lc(\W+R\cdot\Wm)$ over $\Scf$. Additionally, if $\Wm\neq 0$, 
\[
\lim_{\tau\rightarrow\infty}\frac{\W(\tau)}{\tf{\W(\tau)}}=\frac{\Wm}{\tf{\Wm}}.
\]
Otherwise, $\prj_{\Scf^\perp}(\W(\tau))$ remains unchanged throughout the optimization.
\end{theorem}
This theorem demonstrates the directional convergence of attention weight $\W$, and the limits imply the decomposition $\W(\tau)\approx C(\tau)\cdot\Wm+\Wf$ for an appropriate $C(\tau)>0$ with $\lim_{\tau\to\infty}C(\tau)=\infty$. {Note that, for directional convergence to happen, we need $\Wm\neq0$ which happens if and only if \eqref{graph svm} has $``\geq1''$ constraints. That is, the token graph contains a strict priority order.} {This is consistent with our Theorem~\ref{thm informal} where $\Wm$ corresponds to the hard retrieval component ($\Whard$) that selects the high-priority tokens and $\Wf$ is the soft composition component ($\Wsoft$) that determines the softmax probability assignments among these selected tokens.} 
Importantly, this is a global convergence result thanks to the convexity of the optimization problem, which is enabled by the log likelihood optimization combined with Assumption~\ref{assume iden}. 
In Section~\ref{sec:local-gd}, we will demonstrate that global convergence of gradient descent does not hold in general.


\input{sec/fig_asyc}

To illustrate Theorem~\ref{thm cyclic gd}, we conduct experiments with results presented in Figure~\ref{fig cyc}. We create embedding tables with $K=6,~d=8$ and randomly generate dataset with $n=6,~T=4$. Here we choose step size $\eta=0.01$ and perform the normalized gradient descent method to accelerate the increase in the norm of attention weight, so that softmax can easily saturate. Specifically, we update attention weight $\W$ via $\W(\tau+1)=\W(\tau)-\eta\frac{\nabla\Lc(\W(\tau))}{\tf{\nabla\Lc(\W(\tau))}}.$ 
At each iteration $\tau$, correlation coefficient is computed by $\li\W(\tau),\Wm\ri/(\tf{\W(\tau)}\tf{\Wm})$, and results averaged over $100$ random instances are displayed in Fig.~\ref{fig cyc Wm}, which end in correlation $\approx0.987$ after training with $4000$ iterations. In addition to the directional convergence of $\W(\tau)$, we also verify the convergence of finite component by tracking the matrix distance $\tf{\hWf-\Wf}$ where $\hWf=\prj_{\Scf}(\W(\tau))$, and results are displayed in Fig.~\ref{fig cyc Wf}, which obtains $<0.01$ final distance. Both results validate our Theorem~\ref{thm cyclic gd}.



%% file: sec/fig_asyc.tex
\begin{figure*}[!tb]
\centering

\begin{minipage}{0.67\linewidth}
\captionsetup{type=figure}
\begin{subfigure}{0.5\linewidth}
  \begin{tikzpicture}
  \node at (0,0){\includegraphics[scale=0.23, trim={1.3cm 1.3cm 0 0}, clip]{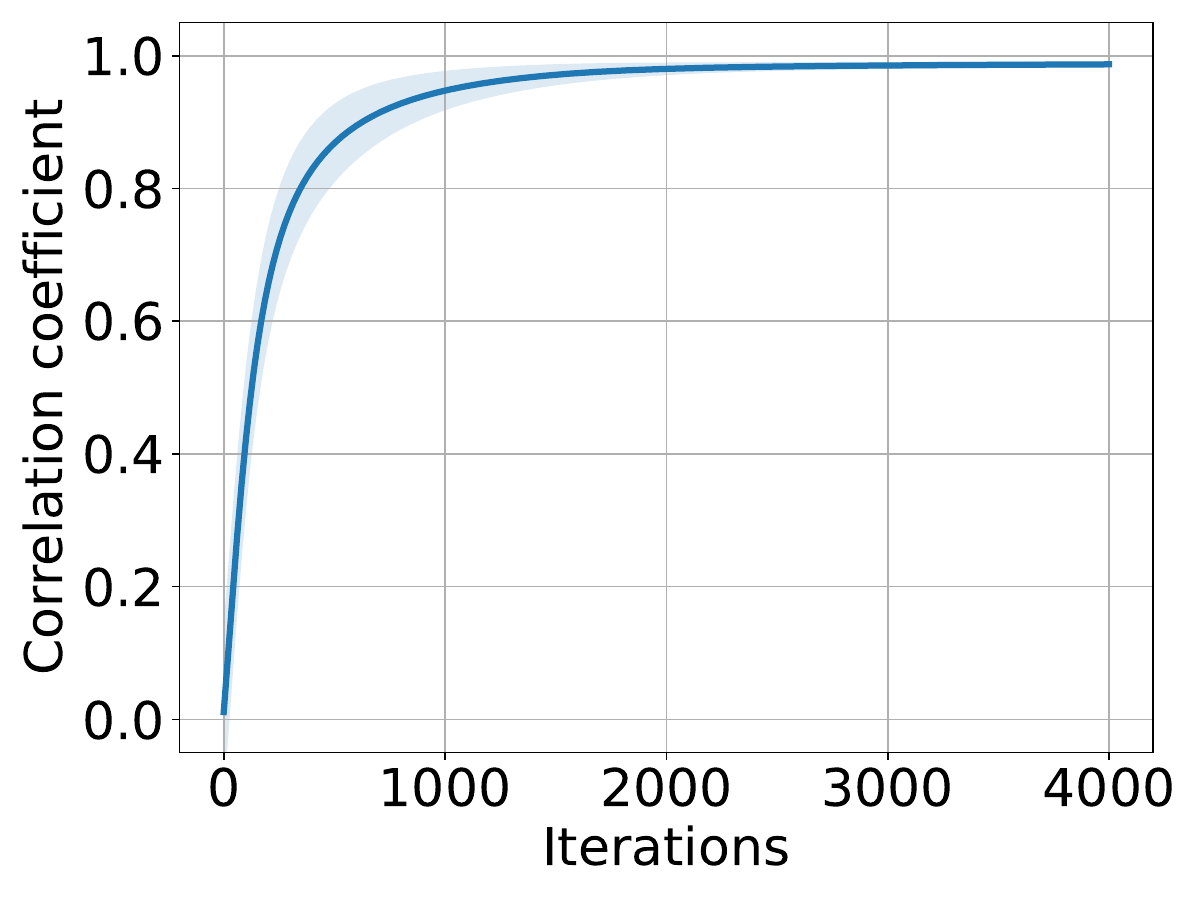}};
  \node at (0.2,-1.8) {\scriptsize{Iterations ($\tau$)}};
  \node[rotate=90] at (-2.45,0) {\scriptsize{Correlation coefficient}};
  \end{tikzpicture}
  \subcaption{\scriptsize{Evolution of $\frac{\W(\tau)}{\tf{\W(\tau)}}\to\frac{\Wm}{\tf{\Wm}}$}}\label{fig cyc Wm}
\end{subfigure}
\begin{subfigure}{0.5\linewidth}
  \begin{tikzpicture}
  \node at (0,0){\includegraphics[scale=0.23, trim={1.3cm 1.3cm 0 0}, clip]{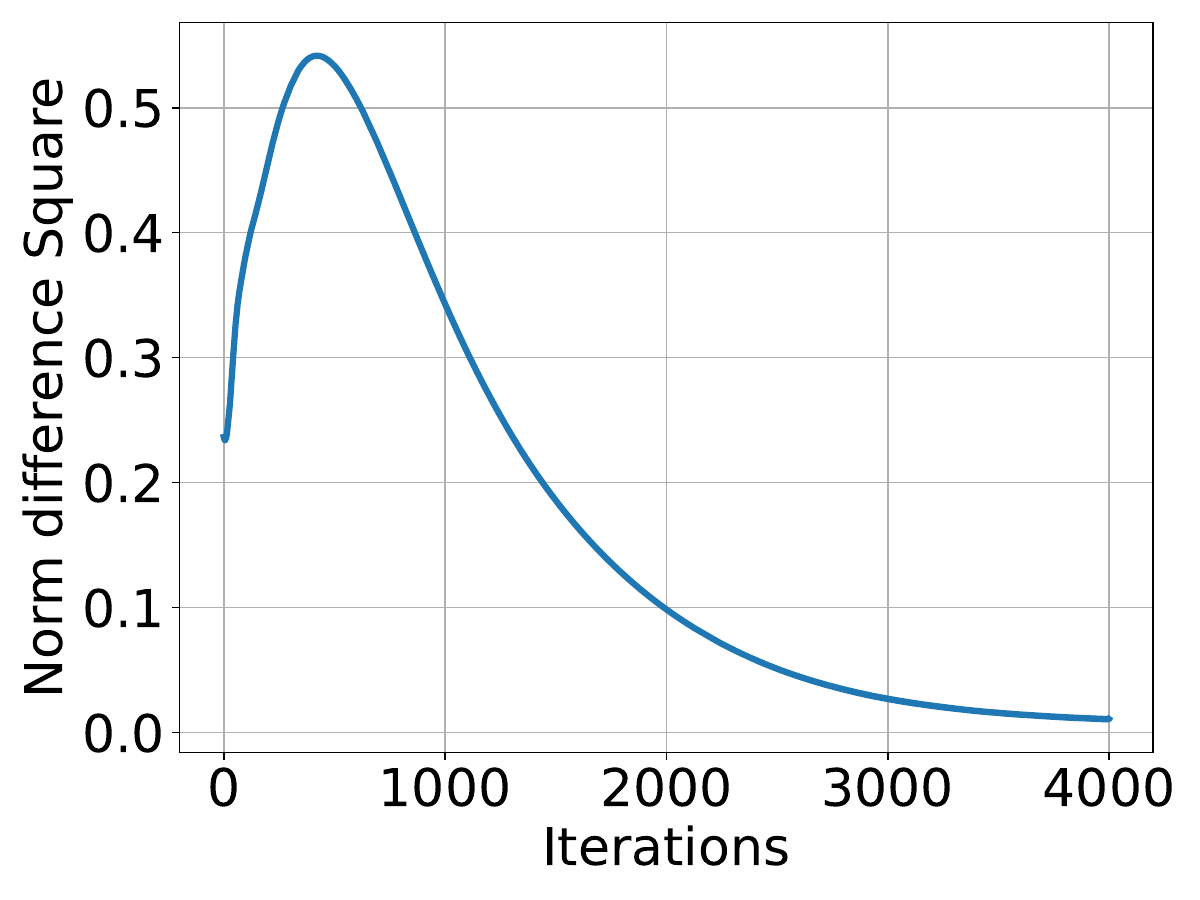}};
  \node at (0.2,-1.8) {\scriptsize{Iterations $(\tau)$}};
  \node[rotate=90] at (-2.45,0) {\scriptsize{$\tf{\hWf-\Wf}$}};
  \end{tikzpicture}
  \vspace{3pt}
  \subcaption{\scriptsize{Evolution of $\prj_{\Scf}(\W(\tau))\to\Wf$}}\label{fig cyc Wf}
\end{subfigure}
  \caption{\small{GD convergence of attention weight $\W$ when training with general dataset. 
  \textbf{(a)} shows the directional convergence of $\W(\tau)$; while \textbf{(b)} presents the convergence of $\prj_{\Scf}(\W(\tau))$.
  }}\label{fig cyc}
\end{minipage}
\hspace{5pt}
\begin{minipage}{0.3\linewidth}
  \begin{tikzpicture}
  \node at (0,0){\includegraphics[scale=0.23, trim={1.3cm 1.3cm 0 0}, clip]{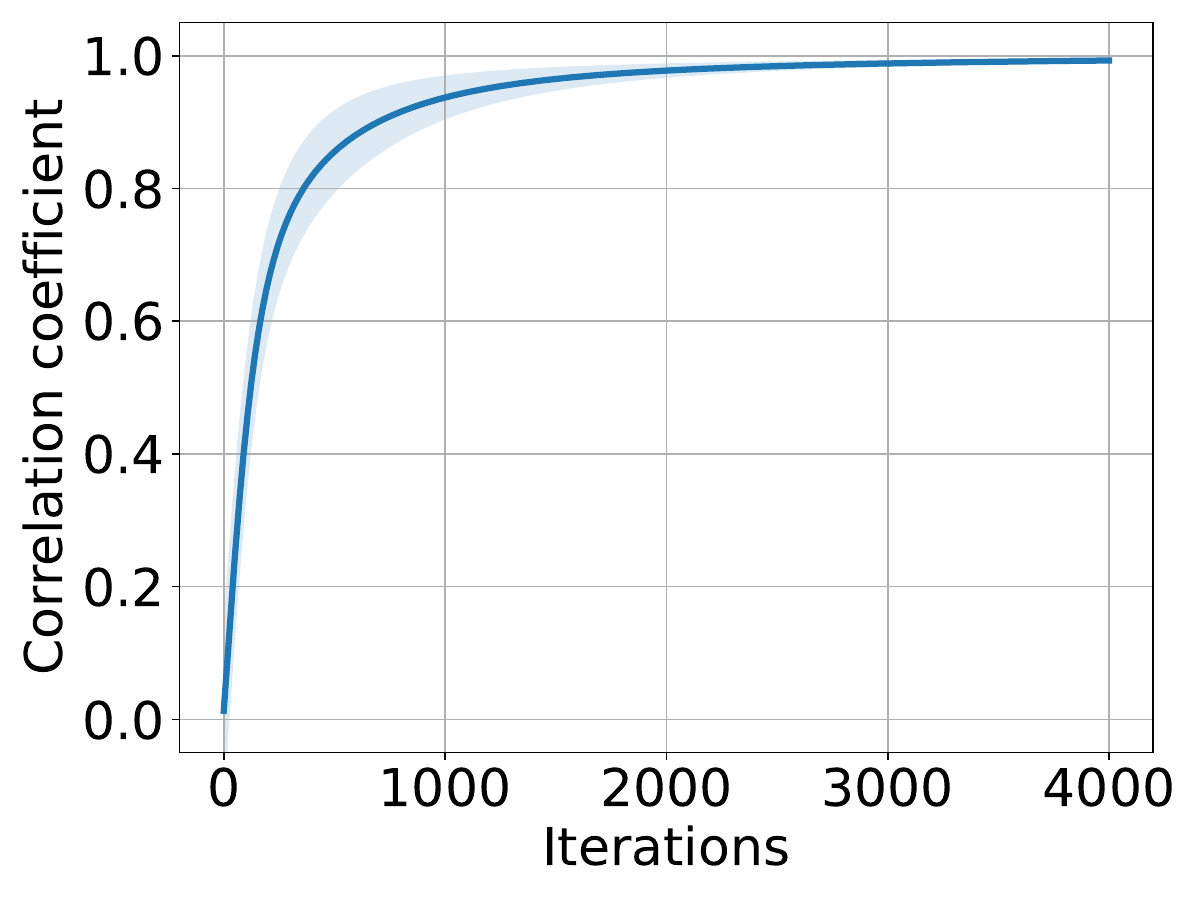}};
  \node at (0.2,-1.8) {\scriptsize{Iterations ($\tau$)}};
  \node[rotate=90] at (-2.45,0) {\scriptsize{Correlation coefficient}};
  \end{tikzpicture}  
\caption{\small{GD convergence of attention weight $\W$ when training with acyclic dataset (Def.~\ref{def acyc}). Correlation coefficient between $\W(\tau)$ and $\Wm$ are presented. 
}}\label{fig acyc}

\end{minipage}
\hfill
\end{figure*}

%% file: sec/reg_path.tex
\ifarxiv
\section{Implicit Bias of Self-attention}\label{sec rp}
\else
\section{IMPLICIT BIAS OF SELF-\\ATTENTION}\label{sec rp}
\fi
In Section~\ref{sec global gd}, we have discussed that gradient descent with log loss guarantees the global convergence. To proceed, in this section, we discuss the implicit bias of attention via analysis of regularization path (RP) as employed in \ref{algo rp} and identify the implicit bias of self-attention on more general next-token prediction problems. 




\begin{theorem}\label{thm general bias} 
Consider any dataset $\data$ and suppose Assumptions~\ref{assume iden} and \ref{assume realizable} hold. 
Additionally, assume loss $\ell:\R\rightarrow\R$ is strictly decreasing and $|\ell'|$ is bounded. 
Let $\Wm{\in\Scf^\perp}$ be the solution of  \eqref{graph svm}  and suppose $\Wm\neq0$.   
Then the solution of regularization path \ref{algo rp} obeys 
\begin{align*}
\lim_{R\rightarrow\infty}\frac{\bar\W_R}{R}= \frac{\Wm}{\tf{\Wm}}~~\text{and}~~\lim_{R\rightarrow\infty}\prj_{\Scf}(\bar\W_R)\in \Wcf.
\end{align*}
Here {$\Wcf=\arg\min_{\W\in\Scf}\lim_{R\rightarrow\infty}\Lc(\W+R\cdot\Wm)$} and we assume that $\Wcf$ \yl{is a bounded set}.
\end{theorem}
Here, we allow more general loss function $\ell$ which is different from the log-loss employed in Section~\ref{sec global gd}, and the ERM problem (cf. \eqref{erm}) is not guaranteed to be convex.  

\subsection{Acyclic Dataset}
\label{sec:acyclic-rp}
Below, {in contrast}, we introduce the concept of acyclic dataset which implies that the next-token prediction task always encounters a strict priority order among tokens {within each TPG}. This corresponds to the setting where all SCCs $((\Cck_i)_{i=1}^{N_k})_{k=1}^K$ are all singletons; or equivalently, 
$\bdata=\emptyset$. 
\begin{definition}[Acyclic dataset]\label{def acyc} We call $\data$ acyclic if all of its {TPGs} are directed acyclic graphs.
\end{definition}
For an acyclic dataset, there are not $i$, $j\in[K]$ satisfying $(i\asymp j)\in\Gck$, for all $k\in[K]$. Thus, SVM formulation \eqref{graph svm} reduces to the following 
simpler form:
\begin{align}\label{acyc svm}\tag{Acyc-SVM}
&\Wm=\arg\min_{\W}\tf{\W}\quad\\
&\text{s.t.}\quad (\eb_i-\eb_j)^\top\W\eb_k\geq 1\quad\forall(i\Rightarrow j)\in\Gck, ~k\in[K].\nn
\end{align}

We next make the following assumption on the linear head. Notably, Assumption~\ref{assume iden} is a special case of Assumption~\ref{assume relax}.
\begin{assumption}\label{assume relax} For $\forall y\in[K]$, $\arg\max_{k\in[K]}\cb_y^\top\eb_k=y$.
\end{assumption}

\begin{lemma}\label{lemma opt risk}
    Consider acyclic dataset $\data$ per Def.~\ref{def acyc} and suppose Assumptions~\ref{assume realizable} and \ref{assume relax} hold. Additionally, assume loss function $\ell$ is strictly decreasing and $|\ell'|$ is bounded. Then for any finite $\W\in\R^{d\times d}$, training risk obeys $\Lc(\W)>\Lc_\st:=\frac{1}{n}\sum_{i=1}^n\ell(\cb_{y_i}^\top\eb_{y_i})$. 
    Additionally, if \eqref{acyc svm} is feasible, then {for any $\W$ that satisfies the constraints in \eqref{acyc svm}}, $\lim_{R\to\infty}\Lc(R\cdot\W)=\Lc_\st$. 
\end{lemma}
Assumption~\ref{assume relax} and Lemma~\ref{lemma opt risk} ensure that 
the best way for attention to make a correct prediction on class $k$ is to output the vector $\eb_k$, i.e., $f_{\W}(\X) = \eb_k$.
The next theorem states the directional bias of self-attention on the acyclic dataset towards the solution of \eqref{acyc svm}.


\begin{theorem}\label{thm acyc bias} 
Suppose $\data$ is acyclic per Definition~\ref{def acyc} and Assumptions \ref{assume realizable} and \ref{assume relax} hold. Additionally, assume loss $\ell:\R\rightarrow\R$ is strictly decreasing and $|\ell'|$ is bounded. Suppose \eqref{acyc svm} is feasible with $\Wm$ denoting its solution. Then, Algorithm~\ref{algo rp} satisfies
\[
\lim_{R\to\infty}\frac{\bar\W_R}{R}=\frac{\Wm}{\tf{\Wm}}.
\]
\end{theorem}
{Recall that a dataset being acyclic implies that there is strict priority order among all tokens in each {TPG}.}
Theorem~\ref{thm acyc bias} establishes the implicit bias of self-attention model for next-token prediction problem 
{in the presence of such strict priority order.}
It demonstrates that {once the SVM problem \eqref{acyc svm} is feasible, the regularized path of optimizing \eqref{erm} converges directionally toward it solution $\Wm$.} 

Following the same implementation setting as in Section~\ref{sec global gd}, in Fig.~\ref{fig acyc}, we again conduct $100$ trials but with randomly generated acyclic dataset $\data$ under the setting of $K=d=8,n=4$ and $T=6$. The results averaged over $100$ random instances are presented in Fig.~\ref{fig acyc} with correlation coefficient exceeds $0.99$ after training with $4000$ iterations. 
{Note that in these experiments, log-loss is employed as loss function and Assumption~\ref{assume iden} is satisfied which guarantee the convexity of $\Lc(\W)$ following Lemma~\ref{lemma cvx} and hence, the connection between \ref{algo gd} and \ref{algo rp} is built by \cite{ji2020gradient}.}

%


%% file: sec/fig_local.tex
\begin{figure*}[!tb]
\vspace{-7pt}
\begin{minipage}{.49\linewidth}
\centering
\captionsetup{type=figure}
\begin{subfigure}{0.5\linewidth}
  \begin{tikzpicture}
  \node at (0,0){\includegraphics[width=0.9\linewidth, trim={1.3cm 1.3cm 0 0}, clip]{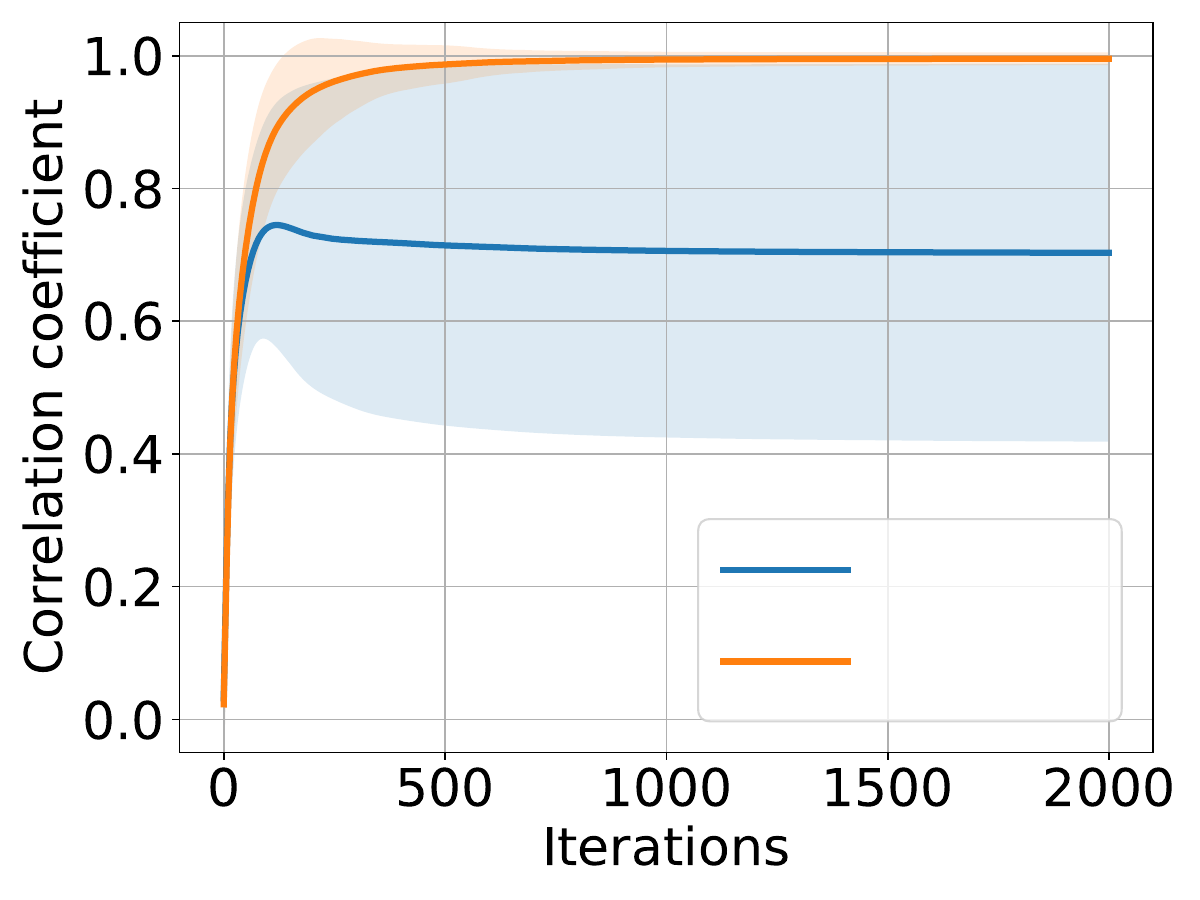}};
  \node at (0.2,-1.5) {\tiny{Iterations ($\tau$)}};
  \node[rotate=90] at (-2.,0) {\tiny{Correlation coefficient}};
  \node[right] at (.65,-0.52) {\tiny{Global}};
  \node[right] at (.7,-0.82) {\tiny{Local}};
  \end{tikzpicture}
  \subcaption{\scriptsize{$\frac{\W(\tau)}{\tf{\W(\tau)}}\to\frac{\tWm}{\tf{\tWm}}$}}\label{fig local ls Wm}
\end{subfigure}
\begin{subfigure}{0.5\linewidth}
  \begin{tikzpicture}
  \node at (0,0){\includegraphics[width=0.9\linewidth, trim={1.3cm 1.3cm 0 0}, clip]{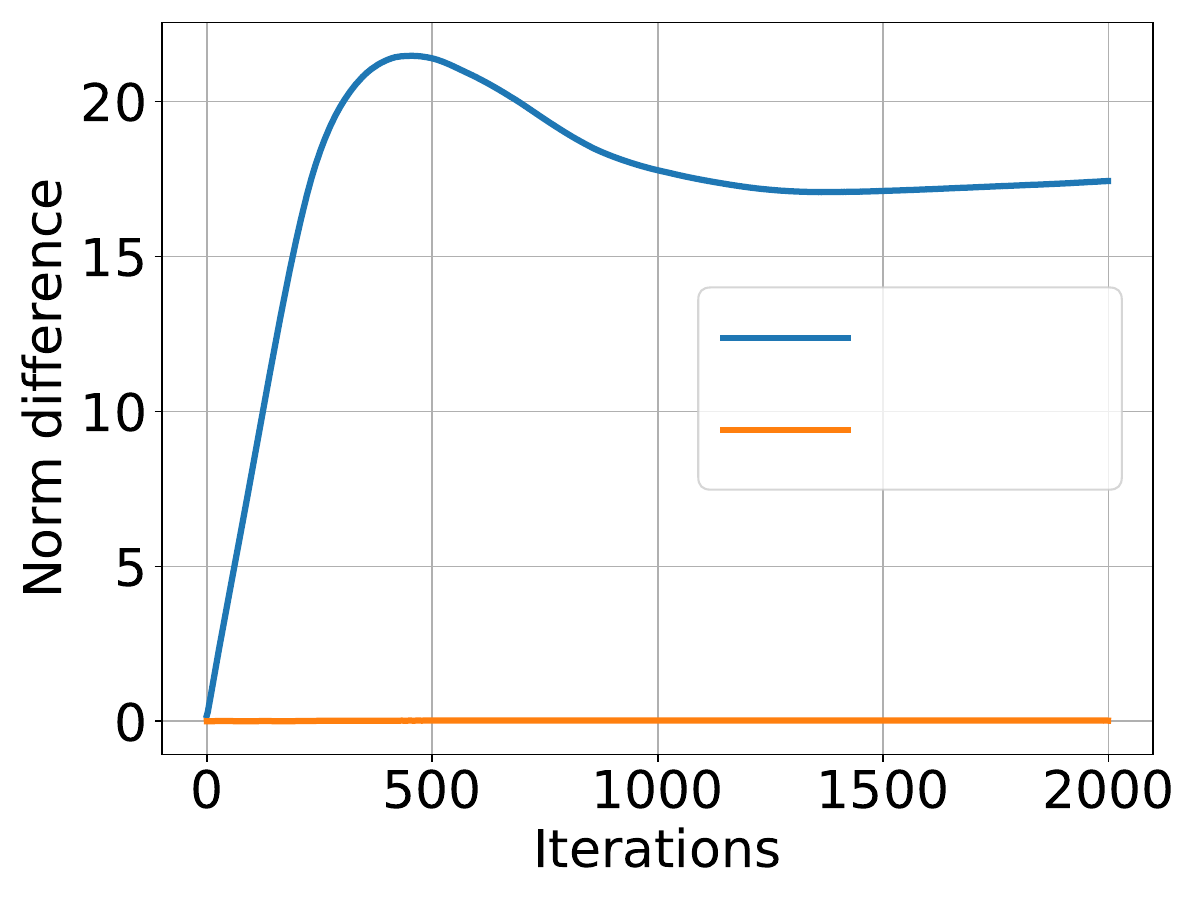}};
  \node at (0.2,-1.5) {\tiny{Iterations ($\tau$)}};
  \node[rotate=90] at (-2.1,0) {\tiny{$\tf{\hWf-\tWf}$}};
  \node[right] at (.65,0.25) {\tiny{Global}};
  \node[right] at (.7,-.05) {\tiny{Local}};
  \end{tikzpicture}
  \subcaption{\scriptsize{$\prj_{\tScf}(\W(\tau))\to\tWf$}}\label{fig local ls Wf}
\end{subfigure}
  \caption{\small{Squared loss with general classifier}}\label{fig local ls}
\end{minipage}\hfill
\begin{minipage}{.49\linewidth}
\centering
\captionsetup{type=figure}
\begin{subfigure}{0.5\linewidth}
  \begin{tikzpicture}
  \node at (0,0){\includegraphics[width=0.9\linewidth, trim={1.3cm 1.3cm 0 0}, clip]{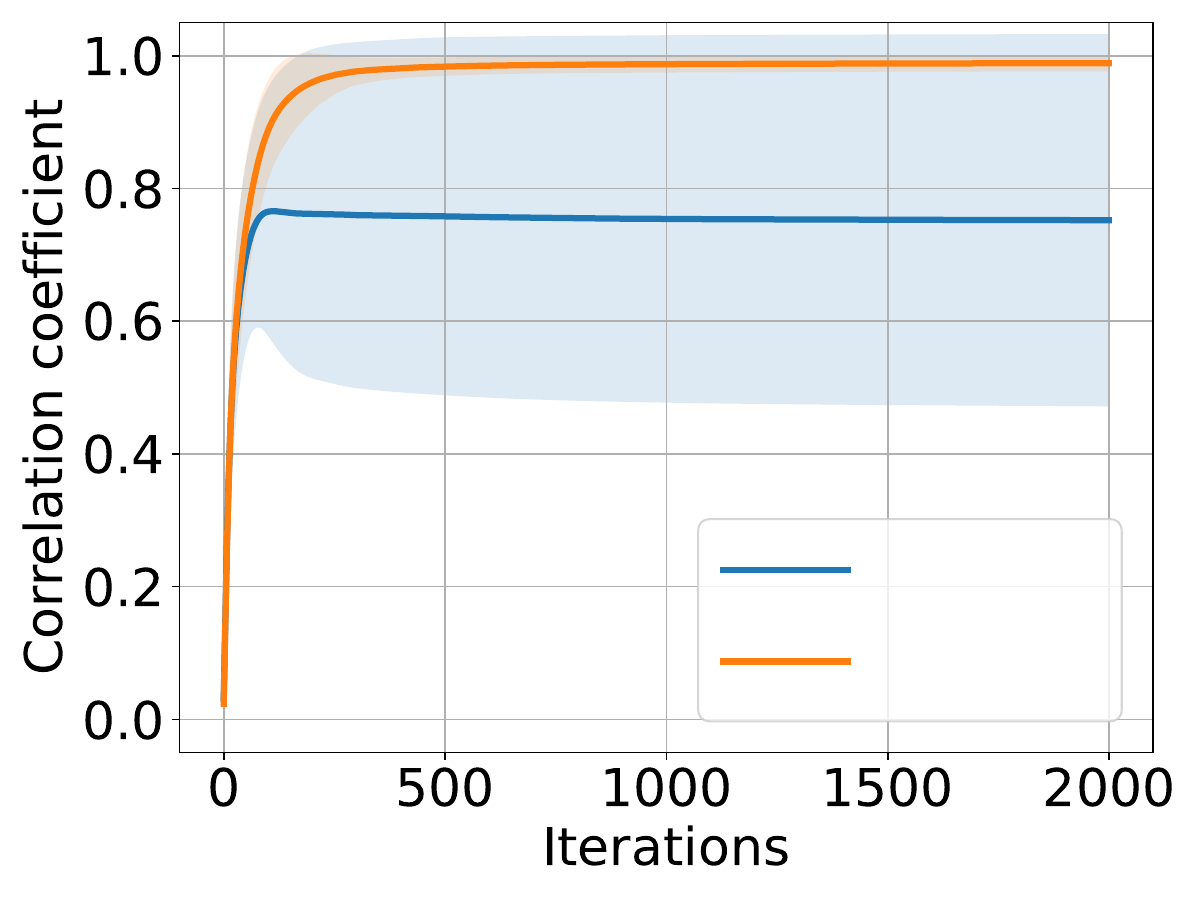}};
  \node at (0.2,-1.5) {\tiny{Iterations ($\tau$)}};
  \node[rotate=90] at (-2.,0) {\tiny{Correlation coefficient}};
  \node[right] at (.65,-0.52) {\tiny{Global}};
  \node[right] at (.7,-0.82) {\tiny{Local}};
  \end{tikzpicture}
  \subcaption{\scriptsize{$\frac{\W(\tau)}{\tf{\W(\tau)}}\to\frac{\tWm}{\tf{\tWm}}$}}\label{fig local ce Wm}
\end{subfigure}
\begin{subfigure}{0.5\linewidth}
  \begin{tikzpicture}
  \node at (0,0){\includegraphics[width=0.9\linewidth, trim={1.3cm 1.3cm 0 0}, clip]{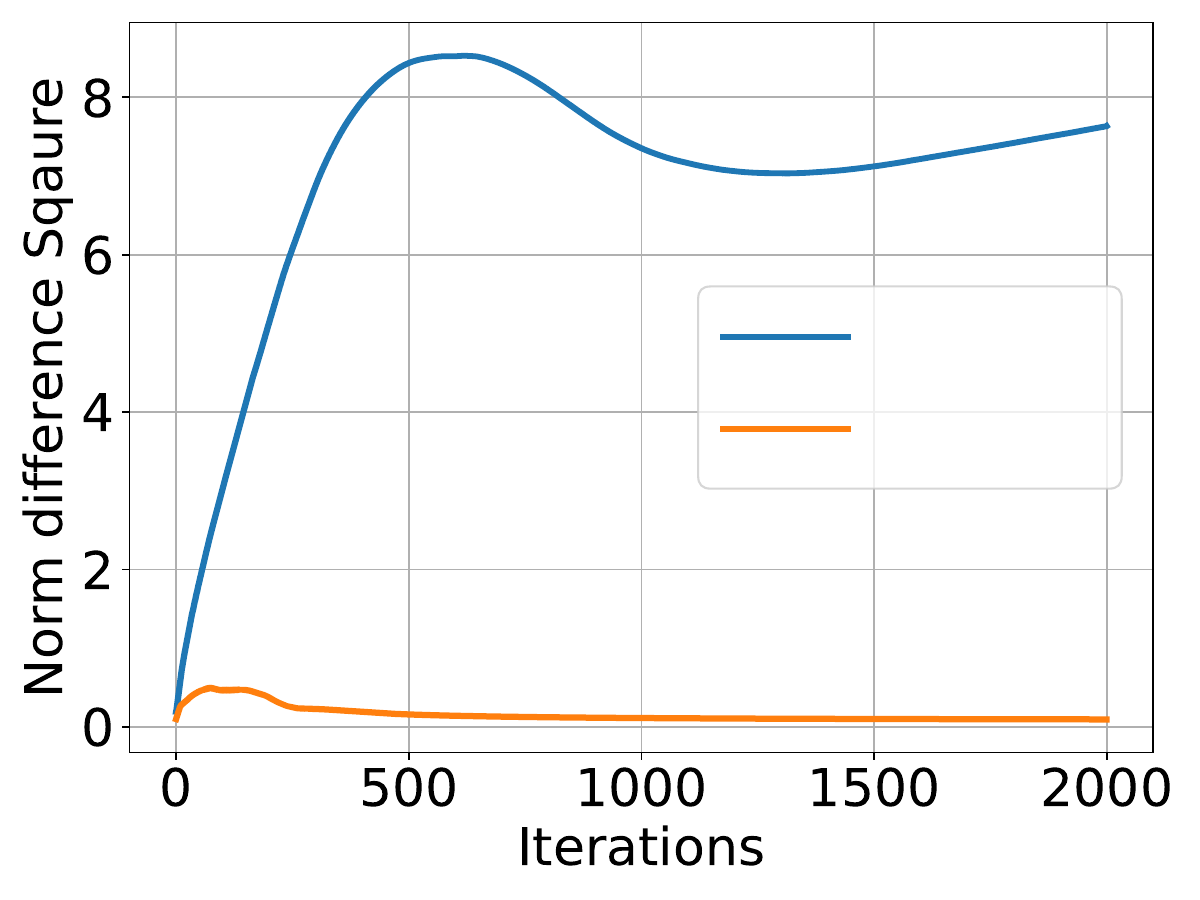}};
  \node at (0.2,-1.5) {\tiny{Iterations ($\tau$)}};
  \node[rotate=90] at (-2.1,0) {\tiny{$\tf{\hWf-\tWf}$}};
  \node[right] at (.65,0.25) {\tiny{Global}};
  \node[right] at (.7,-.05) {\tiny{Local}};
  \end{tikzpicture}
  \subcaption{\scriptsize{$\prj_{\tScf}(\W(\tau))\to\tWf$}}\label{fig local ce Wf}
\end{subfigure}
  \caption{\small{Cross-entropy loss with general classifier}}\label{fig local ce}
\end{minipage}\hfill
\end{figure*}

%% file: sec/local.tex
\ifarxiv
\section{Further Investigation on Local Convergence}
\else
\section{FURTHER INVESTIGATION ON LOCAL CONVERGENCE}
\fi
\label{sec:local-gd}

So far, we have proved the global GD convergence of attention weight when one employs the log-loss (Section~\ref{sec global gd}) and studied the implicit bias of self-attention over next-token prediction problem using RP analysis (Section~\ref{sec rp}). In this section, we investigate further on the convergence performance of GD and ask: 

\begin{center}
\emph{When does the GD exhibit local convergence rather than global?\\ Can we characterize its implicit bias? } 
\end{center}

Convergence performance of learning 1-layer attention has been analyzed in the previous work \cite{tarzanagh2023margin,tarzanagh2023transformers}, and they have observed the local convergence phenomenon, and also provided the theoretical explanation and empirical evidence. Inspired by their work, we define the \emph{pseudo} TPGs for obtaining \emph{locally-optimal} SVM equivalence $\tWm$ and cyclic component $\tWf$ as follows:
\begin{enumerate}
    \item Given any dataset $\data$, consider GD solution $\Wgd$. For each training example $(\X,y)\in\data$, let  
    $\s=\sft{\X\Wgd\xl}$.
    \item Construct TPGs based on $\s$ by adding directed edge $(x_{t_1}\to x_{t_2})$ to $\Gck$ if $\s_{t_1}>0$, where $k,x_t$ are the token IDs of last token and $\x_t$, respectively.
\end{enumerate}
Different from the TPGs defined in Section~\ref{sec ntg} which is uniquely determined by the dataset and the ground truth labels, pseudo-TPGs build edges based on the tokens selected by GD solution $\Wgd$. 
To further investigate under which scenarios local convergence phenomenon exists, we consider the following cases and provide experimental evidence.

\smallskip
\noindent\textbf{General loss function $\ell$.} In Section~\ref{sec global gd}, we analyze the convergence performance of gradient descent when employing log-loss. As we have discussed, such loss guarantees the convexity of the problem and therefore, GD of attention weight (directional) converges to its global minima. Here, we investigate the performance of more general loss function, i.e., squared loss, and find empirical evidence of local convergence (Figures~\ref{fig local ls} and \ref{fig local ce}). 


\smallskip
\noindent\textbf{Extended linear head.} In this work, GD experiments are conducted under Assumptions~\ref{assume iden} and \ref{assume realizable}, which implies the convex training loss $\Lc(\W)$ and global convergence performance. 
Now consider a more general linear head (i.e.,~Assumption~\ref{assume relax}). As have also been observed and discussed in \cite{tarzanagh2023margin,tarzanagh2023transformers}, \ref{algo gd} can converge to a locally-optimal solution. 
 


Figures~\ref{fig local ls} and \ref{fig local ce} display our local convergence results where Fig.~\ref{fig local ls} employs squared loss, i.e. $\ell(u)=(1-u)^2$ 
and Fig.~\ref{fig local ce} utilizes cross-entropy loss. Both apply general head following Assumption~\ref{assume relax}. Similar to Fig.~\ref{fig cyc}, we present (directional) convergence performance of $\W$ towards $\Wm$ and $\Wf$. Results indicate that instead of converging to the global solution (blue curves), attention weights trained via GD align more closely with the locally-optimal SVM solution defined via the pseudo TPGs constructed by $\Wgd$ (orange curves). In Fig.~\ref{fig local ls Wf}, the norm difference to $\tWf$ remains zero, indicating that all SCCs in the pseudo TPGs are singleton and GD optimizes attention weights towards selecting one token per sequence. While in Fig.~\ref{fig local ce}, multiple tokens can be selected by $\Wgd$. 
{Note that in Fig.~\ref{fig local ce Wf}, the norm of difference does not end with zero value on average. The potential explanations can be: 
Due to the non-convexity of training loss, training $\W$ with GD may not fully capture its RP solution $\tWf$ over the cyclic subspace, and general classification head induces correlation among tokens, leading the attention mechanism to generate more intricate composed tokens. Nevertheless, our empirical results indicate that $\W$ more closely aligns with the local $\tWf$ within its cyclic subspace. We defer a rigorous definition of local $\tWf$ and guarantees related to gradient descent for future exploration.} Experimental details are deferred to the appendix.

%% file: sec/related.tex
\ifarxiv
\section{Related Work}
\else
\section{RELATED WORK}
\fi


Inspired by the increasing popularity of Transformer-based models, a large number of research efforts have focused on developing theoretical understanding of various aspects of such models. 
\cite{yun2020_universal,fu2023_randomfeatures,bombari2024towards} studied the expressive power of Transformers and showed that they are universal approximators for sequence-to-sequence functions. A similar result for efficient variants of Transformers based on sparse attention was presented in \cite{yun2020_sparse}. 
\cite{edelman2022_inductive} studied bias of single attention layer towards representing sparse functions of input sequence with favourable generalization behaviour. 
Interestingly, 
\cite{baldi2023_quarks} explored key building blocks of attention mechanism beyond modern neural networks and studied the functional capacity of the resulting attention-based models. Other lines of theoretical efforts have focused on explaining various properties of Transformer-based models, including rank collapse~\citep{dong2021_rank} and realization of in-context learning~\citep{xie2022_incontext,garg2022_incontext,akyurek2023_incontext,oswald2023_incontext,li2023_incontext,huang2023context,li2023transformers,collins2024context,jeon2024information,chen2024training,li2024training}.

Unlike these prior work, we focus on optimization-theoretic analysis of attention-based models for the next-token prediction objective. Our work sheds light on the implicit bias of underlying optimization problem towards SVM formulations, which builds on the recent research efforts~\citep{tarzanagh2023margin, tarzanagh2023transformers}. However, different from these prior efforts that deal with traditional (supervised) classification tasks, we focus on next-token prediction task -- the main workhorse of Transformer-based language modeling. {The recent work \cite{thrampoulidis2024implicit} also explores the next-token prediction problem under a classification-like setting, employing a related SVM formulation. Since we study transformers, the main messages are fairly different, e.g., our theory relies on graph-theoretic concepts such as SCCs and token-priority graphs to capture the generative process learned by SGD.}  Notably, several recent efforts~\citep{jelassi2022_transformer,li2023_transformer,li2023transform,oymak23a_prompt,deora2023optimization,chen2024provably} have also analyzed optimization and generalization dynamics of attention-based models. However, these works again only focus on traditional classification tasks and consider simplifications of the attention mechanism~\citep{jelassi2022_transformer} or work with strict statistical data assumptions~\citep{jelassi2022_transformer,li2023_transformer,oymak23a_prompt}. In contrast, we provide a detailed optimization-theoretic treatment of the original (non-linear input dependent) attention mechanism without any statistical assumption on the underlying data. 
\yl{Related work by \cite{tian2023_scan} studies the training dynamics of next-token prediction. Compared to us, their analysis is restricted to a specific statistical data model, including the requirement of long input sequences ($T \to \infty$). \cite{ildiz2024self,makkuva2024attention} build connections between self-attention and Markov chains.} In contrast, we characterize the implicit bias of self-attention learning to novel SVM formulations without any such assumptions on the data model or sequence lengths.

We would also like to note the rich literature on studying implicit bias of gradient-based optimization methods (see, e.g.,~\cite{soudry2018_implicit,gunasekar2018_implicit,ji2020gradient,ji2021_characterizing,kini2021_label,li2019_towards,blanc2020_implicit,qian2019_implicit,wang2021_implicit} and references therein). However, this prior work does not focus on the optimization landscape of learning Transformer-based models and thus, does not provide specific insights into their inner-workings, which is the main objective of our work.






%% file: sec/discuss.tex
\ifarxiv
\section{Discussion}
\else
\section{DISCUSSION}
\fi

{In this work we set out to demystify Transformer-based language modeling via next-token prediction task. We established that single-layer self-attention learning has implicit bias towards the solution of a support vector machine (SVM) formulation based on token-priority graphs which encode the priority order among the tokens as per the training data. Our analysis shows that a self-attention model learned via next-token prediction objective implements a selection mechanism to suppress the lower priority tokens in order to predict the higher priority tokens as the next-token for an input sequence. At the same time, such an attention model would distribute its softmax probabilities among all equal priority tokens as modeled by the strongly-connected components of the next-token graph. Ultimately, our results comprehensively capture the automaton implemented by a 1-layer self-attention under realistic assumptions. }

\yl{
A natural future direction is relaxing our assumptions in SGD analysis and providing a comprehensive characterization of the training dynamics, accounting for non-convexities. It would also be interesting to extend our analysis to multi-layer multi-head self-attention models or explore how feed-forward layers (a.k.a.~MLP layers) in Transformers affect the optimization dynamics and aid in the aforementioned token selection and composition mechanisms during next-token prediction.
}

\section*{Acknowledgements}
This work was supported in part by the NSF grants CCF-2046816, CCF-2212426, CNS-1932254, UMich's MIDAS PODS program, a Google Research Scholar award, and an Adobe Data Science Research award.%

%% file: supp/check_lst.tex
\section*{Checklist}



 \begin{enumerate}

 \item For all models and algorithms presented, check if you include:
 \begin{enumerate}
   \item A clear description of the mathematical setting, assumptions, algorithm, and/or model. [Yes]
   \item An analysis of the properties and complexity (time, space, sample size) of any algorithm. [Not Applicable]
   \item (Optional) Anonymized source code, with specification of all dependencies, including external libraries. [Yes]
 \end{enumerate}

 \item For any theoretical claim, check if you include:
 \begin{enumerate}
   \item Statements of the full set of assumptions of all theoretical results. [Yes]
   \item Complete proofs of all theoretical results. [Yes]
   \item Clear explanations of any assumptions. [Yes]     
 \end{enumerate}

 \item For all figures and tables that present empirical results, check if you include:
 \begin{enumerate}
   \item The code, data, and instructions needed to reproduce the main experimental results (either in the supplemental material or as a URL). [Yes]
   \item All the training details (e.g., data splits, hyperparameters, how they were chosen). [Yes]
         \item A clear definition of the specific measure or statistics and error bars (e.g., with respect to the random seed after running experiments multiple times). [Not Applicable]
         \item A description of the computing infrastructure used. (e.g., type of GPUs, internal cluster, or cloud provider). [Not Applicable]
 \end{enumerate}

 \item If you are using existing assets (e.g., code, data, models) or curating/releasing new assets, check if you include:
 \begin{enumerate}
   \item Citations of the creator If your work uses existing assets. [Not Applicable]
   \item The license information of the assets, if applicable. [Not Applicable]
   \item New assets either in the supplemental material or as a URL, if applicable. [Not Applicable]
   \item Information about consent from data providers/curators. [Not Applicable]
   \item Discussion of sensible content if applicable, e.g., personally identifiable information or offensive content. [Not Applicable]
 \end{enumerate}

 \item If you used crowdsourcing or conducted research with human subjects, check if you include:
 \begin{enumerate}
   \item The full text of instructions given to participants and screenshots. [Not Applicable]
   \item Descriptions of potential participant risks, with links to Institutional Review Board (IRB) approvals if applicable. [Not Applicable]
   \item The estimated hourly wage paid to participants and the total amount spent on participant compensation. [Not Applicable]
 \end{enumerate}

 \end{enumerate}

%% file: supp/support.tex

\input{sec/fig_intro_app}
\ifarxiv
\section{Auxiliary Results}
\else
\section{AUXILIARY RESULTS}
\fi

\subsection{Soft-composition Component}\label{app soft}
In Section~\ref{sec global gd}, we have theoretically shown that when training a single-layer self-attention model with gradient descent and log-loss function, once $\Wm\neq0$, the composed attention weight $\W(\tau)$ will diverge in Frobenius norm and $\W(\tau)$ converges towards direction $\Wm/\tf{\Wm}$; while in the subspace of $\Scf$, $\prj_{\Scf}(\W(\tau))$ converges to a finite $\Wf$ which is the unique solution minimizing the training loss over subspace $\Scf$ as described in Theorem~\ref{thm cyclic gd}. Here, $\Wm$ follows the solution of \eqref{graph svm} and plays a role in separating tokens from different SCCs within the same TPG. Specifically, nodes satisfy $(i\Rightarrow j)\in\Gck$.  

As for the nodes contained within the same SCC (e.g., $(i\asymp j)$), to ensure that $i$ and $j$ will not suppress each other, \eqref{graph svm} solves the SVM problem with the constraint $(\eb_i-\eb_j)^\top\W\eb_k=0$.
This essentially disregards the influence of distinct tokens within the same SCC.
Consequently, $\Wm$ does not truly capture the essence of the ERM solution. In the following, we introduce cyclic subdataset and the so-called \emph{cyclic-component}, and an equivalence between the cyclic term and $\Wf$ can be established under mild assumptions. 



\begin{definition}[Cyclic subdataset]\label{cyc sub} 
Given any training sample $(\X,y)\in\data$, we obtain the corresponding sample $(\X', y) \in \bdata$ by 
removing all tokens in $\X$ that satisfy $(y\Rightarrow x)$ in the corresponding {TPG}.
\end{definition}

In short, cyclic subdataset focuses on the input tokens that are part of the same SCC as the label token. {Fig.~\ref{fig:intro app}(Right) presents the cyclic subdataset $\bdata$ of $\data$ given in Fig.~\ref{fig:intro app}(Left), which is the same as Fig.~\ref{fig:intro}. In $\Gc^{(1)}$, all nodes are separated into different SCCs, and therefore, none of them is present in $\bdata$; while in $\Gc^{(2)}$, token $\eb_1,\eb_2$ and $\eb_3$ are reachable from each other, and then are utilized to construct $\bdata$ while $\eb_4$ is removed from the dataset.} 
Note that $\bdata$ provides a {self-contained sub-problem} that solely focuses on intra-SCC edges.


\begin{definition} [Cyclic component] \label{def finite correct} $\Wcf$ is obtained as the solution set of the ERM problem over the cyclic subdataset $\bdata$ per Definition~\ref{cyc sub}. Concretely, 
\begin{align*}
&\Wcf=\arg\min_{\W\in\Scf}\bar\Lc(\W)\quad\\
&\text{where}\quad\bar\Lc(\W)=\frac{1}{n}\sum_{(\X,y)\in\bdata} \ell(\cb_{y}^\top\X^\top \sft{\X\W\xl}).
\end{align*}
\end{definition}
\begin{lemma}\label{lemma ortho} 
Consider a dataset $\data$ and let $\Wm$ be the corresponding SVM solution of \eqref{graph svm} with $\Wm\neq0$. Then we have $\Wm\perp\Scf$, and for any $\bWf\neq0\in\Wcf$, $\bWf$ and $\Wm$ are orthogonal.
\end{lemma}

\begin{lemma}\label{lemma Lcb eq}
    Let $\W\in\R^{d\times d}$ be an arbitrary matrix, then we have $\Lcb(\W)=\Lcb(\prj_{\Scf}(\W))$.
\end{lemma}

{\begin{lemma}\label{lemma Wf eq}
    Suppose Assumptions~\ref{assume iden} and \ref{assume realizable} hold, and loss function $\ell(u)=-\log(u)$, 
    then for any finite $\W$, $\tilde\Lc(\W)=\Lcb(\W)$ where $\tilde\Lc(\W)$ and $\Lcb(\W)$ are defined in Theorem~\ref{thm cyclic gd} and Definition~\ref{def finite correct}, respectively. 
\end{lemma}}

\subsection{Useful Notations}
In this section, we introduce additional notations used in the subsequent proofs.

\smallskip
\noindent$\bullet$ \textbf{Token index sets $\Oc_i$, $\Ocb_i$, $\Rc_i$, $\Rcb_i$, $i\in[n]$.} Consider dataset $\data$. Throughout, for any sample $(\X_i,y_i)\in\data$, $i\in[n]$, we define
\begin{subequations}\label{def Oc Rc}
\begin{align}
&\Oc_i:=\left\{t~\Big|~x_{it}=y_i, \forall t\in[T_i]\right\}\quad\text{and}\quad\Ocb_i=[T_i]-\Oc_i,\\
&\Rc_i:=\Oc_i\bigcup\left\{t~\Big|~(x_{it}\asymp y_i)\in\Gc^{(\bar x_i)}, \forall t\in[T_i]\right\}\quad\text{and}\quad\Rcb_i=[T_i]-\Rc_i
\end{align}
\end{subequations}
where $x_{it}$ is the token ID of $\x_{it}$, $T_i$ is the number of tokens in the input sequence $\X_i$ and $\Gc^{(\bar x_i)}$ is the corresponding token-priority graph (TPG) associated with the last/query token of $\X_i$. Concretely, $\Oc_i$ returns the token indices of $i$-th input that have the same token ID as label $y_i$, while $\Rc_i$ returns the token indices of $i$-th input that are included in the same strongly-connected component (SCC) as label $y_i$ in the corresponding TPG. Then for any $t\in\Rcb_i$, we have $(y_i\Rightarrow x_{it})\in\Gc^{(\bar x_i)}$. Take the last input sequence in Figure~\ref{fig:intro app}(left) as an example, where $\Oc=\{3\}$ and $\Rc=\{1,3\}$.

\smallskip
\noindent$\bullet$ \textbf{Datasets $\data$, $\bdata$ and sample index set $\Ic$, $\Icb$.} Recap the training dataset $\data=(\X_i,y_i)_{i=1}^n$. Based on the relationships between input tokens and label token, following instructions in Section~\ref{sec ntg} we can construct the TPGs of dataset $\data$. Then, let $\Ic\subseteq[n]$ be the sample index set such that for any $i\in\Ic$, $\X_i$ contains distinct tokens from the same SCC as label $y_i$ in their corresponding TPG. Or equivalently, 
\begin{align}\label{def Ic}
    \Ic=\left\{ i~\Big|~\Rc_i-\Oc_i\neq\emptyset,i\in[n]\right\}\quad\text{and}\quad\Icb=[n]-\Ic.
\end{align}
Then the cyclic subset defined in Definition~\ref{cyc sub} can be written by 
\begin{align}
\bdata=(\bar\X_i,y_i)_{i\in\Ic},\label{def bdata}
\end{align}
where $\bar\X_i$ is obtained by removing all input tokens of $\X_i$ that are in the different SCCs from the label token $y_i$, or equivalently, removing $\x_{it}$, $t\in\Rcb_i$. 
Hence, for all $i\in\Icb$, $\X_i$ only contains input tokens (ignoring the ones with the same token ID as label) that have strictly lower priority than its label token, i.e., $(y_i\Rightarrow x_{it})\in\Gc^{(\bar x_i)}$ for $t\in\Ocb_i$. In Figure~\ref{fig:intro app}(left), we have $\Ic=\{4,5,6,7\}$ and $\Icb=\{1,2,3\}$.

\smallskip
\noindent $\bullet$ \textbf{Token scores $\bgam_i,i\in[n]$ and loss $\Lc(\W)$ under Assumption~\ref{assume iden}.}
Let $\bgam_i=\X_i\cb_{y_i}$ be the token score vectors. Then under Assumption~\ref{assume iden}, we have
\begin{align}
\gamma_{it}=\begin{cases}
    1,\quad t\in\Oc_i\\
    0,\quad t\in\Ocb_i
\end{cases}\quad\text{for all }i\in[n].\label{def score}
\end{align}
Additionally, letting $\s_i^{\W}=\sft{\X_i\W\xli}$, we can rewrite the training risk as follows:
\begin{align}
\Lc(\W)=\frac{1}{n}\sum_{i=1}^n\ell\left(\sum_{t\in\Oc_i}s_{it}^{\W}\right).\label{def erm under assum iden}
\end{align}

\subsection{Proof of Lemma~\ref{lemma feasible}}

\begin{proof}
    Recap the constraints in \eqref{graph svm} problem: 
\begin{align}
(\eb_i-\eb_j)^{\top}\W\eb_k
\begin{cases}=0\quad \forall(i\asymp j) \in\Gck\\
\geq 1\quad \forall(i\Rightarrow j)\in\Gck
\end{cases}
\text{for all}\quad k\in[K].\label{con 1}
\end{align}
Since $\Eb=[\eb_1~\eb_2~\cdots~\eb_K]^\top\in\R^{K\times d}$ is full row rank, then $K\leq d$ and $\eb_k,k\in[K]$ are linearly independent. 
Let $\bar\Eb\in\R^{K\times d}$ satisfying $\bar\Eb\Eb^\top=\Iden$. 
Then for any $\bar\W\in\R^{K\times K}$, we get
\begin{align}
(\eb_i-\eb_j)^{\top}\bar\Eb^\top\bar\W\bar\Eb\eb_k
\begin{cases}=0\quad \forall(i\asymp j) \in\Gck\\
\geq 1\quad \forall(i\Rightarrow j)\in\Gck
\end{cases}
\text{for all}\quad k\in[K]\label{con 2}
\end{align}
and feasibility of \eqref{con 2} implies $\W\in\R^{d\times d}$ in \eqref{con 1} is feasible. Since we can set $\W=\bar\Eb^\top\bar\W\bar\Eb$. Next let $\ub_i=\bar\Eb\eb_i$ $i\in[K]$ be $K$-dimensional one-hot vectors. Then it remains to show that there exists $\bar\W\in\R^{K\times K}$ such that 
\begin{align}
(\ub_i-\ub_j)^{\top}\bar\W\ub_k
\begin{cases}=0\quad \forall(i\asymp j) \in\Gck\\
\geq 1\quad \forall(i\Rightarrow j)\in\Gck
\end{cases}
\text{for all}\quad k\in[K]\label{con 3}
\end{align}
is feasible. Additionally, it is equivalent with showing that for any $k\in[K]$, there exists $\w\in\R^K$, such that
\begin{align}\label{res w}
(\ub_i-\ub_j)^{\top}\w
\begin{cases}=0\quad \forall(i\asymp j) \in\Gck\\
\geq 1\quad \forall(i\Rightarrow j)\in\Gck.
\end{cases}
\end{align}
To start with, we first derive the priority order of each graph (referring to the topological sorting of directed graph). Let $M_i$ be the order of $\ub_i$ where $M_i$'s $i\in[K]$ are positive integers. Then if $(i\asymp j)$, $M_i=M_j$; if $(i\Rightarrow j)$, $M_i>M_j$. Then let $\w=\sum_{i\in[K]}M_i\ub_i$. We obtain that for any $k\in[K]$,
\begin{align*}
    &\forall(i\asymp j)\in\Gck,~(\ub_i-\ub_j)^\top\w=(\ub_i-\ub_j)^\top(M_i\ub_i+M_j\ub_j)=M_i-M_j=0\\
    &\forall(i\Rightarrow j)\in\Gck,~(\ub_i-\ub_j)^\top\w=(\ub_i-\ub_j)^\top(M_i\ub_i+M_j\ub_j)=M_i-M_j\geq1
\end{align*}
which indicates that \eqref{res w} is feasible for any $k\in[K]$ and it completes the proof. 
\end{proof}

\subsection{Proof of Lemma~\ref{lemma ortho}}
\begin{proof}
Recall from Definition~\ref{def cyc subspace} that $\Scf$ is the span of all matrices $(\eb_i-\eb_j)\eb_k^\top$ for all $(i\asymp j)\in\Gck$ and $k\in[K]$. Then for any matrix $\W\in\Scf$, there exist $a_{ijk}$'s satisfying
\[
\W=\sum_{i,j,k}a_{ijk}(\eb_i-\eb_j)\eb_k^\top
\]
where $(i\asymp j)\in\Gck$ and $k\in[K]$. 
Since $\Wm$ is the solution of \eqref{graph svm} that satisfies the all ``$=0$'' constraints, for any matrix $\W\in\Scf$, we have 
\[
\li\Wm,\W\ri=\sum_{i,j,k}a_{ijk}\li\Wm,(\eb_i-\eb_j)\eb_k^\top\ri=0.
\]
Therefore, $\Wm\perp\Scf$. 
%
%
\end{proof}

\subsection{Proof of Lemma~\ref{lemma Lcb eq}}
\begin{proof} Recap the definition of $\Lcb(\W)$ from Def.~\ref{def finite correct} and $\Ic$, $\bar\X_i$ from \eqref{def Ic}, \eqref{def bdata}. Then  
    \[
\Lcb(\W)=\frac{1}{n}\sum_{i\in\Ic} \ell(\cb_{y_i}^\top\bar\X_i^\top \sft{\bar\X_i\W\xli}).
\]
Let $\W^\perp=\prj_{\Scf^\perp}(\W)$ and $\W^\pl=\prj_{\Scf}(\W)$. Then it remains to show that for any $(\bar\X,y)\in\bdata$, $\sft{\bar\X\W\xl}=\sft{\bar\X\W^\pl\xl}$. 

For simplification, let $\xl=\eb_k$, and following the definition of TPG, SCC and $\bdata$, we have that all tokens $\x\in\bar\X$ are in the same SCC and denote the token set as $\Cck$. Then $\Scf$ spans the matrices $(\eb_i-\eb_j)\eb_k^\top$ for $i,j\in\Cck$. 
For any $i\in\Cck$, we get 
\[
\eb_{i}^\top\W\eb_k=\eb_{i}^\top\W^\pl\eb_k+\eb_i^\top\W^\perp\eb_k.
\]
{Next, let $a_{ik}=\eb_i^\top\W^\perp\eb_k$, $i\in\Cck$. Since $\W^\perp\perp\Scf$, and $(\eb_i-\eb_j)\eb_k^\top\in\Scf$, we obtain}
\begin{align}
&(\eb_i-\eb_j)^\top\W^\perp\eb_k=0\label{ortho W}\\
\Longrightarrow~&\eb_i^\top\W^\perp\eb_k-\eb_{j}^\top\W^\perp\eb_k=0\nn\\
\Longrightarrow~&a_{ik}-a_{jk}=0\nn\\
\Longrightarrow~&a_{ik}=a_{jk}=:\bar a_k.\nn
\end{align}
Then we have that for any $\x\in\bar\X$, $\x^\top\W^\perp\xl=\bar a_k$ where $\bar a_k$ is associated with the last/query token $\xl$ and hence 
\begin{align*}
&\Xb\W\xl=\Xb\W^\pl\xl+\Xb\W^\perp\xl=\Xb\W^\pl\xl+\bar a_k\mathbf{1}\\
&\sft{\Xb\W\xl}=\sft{\Xb\W^\pl\xl+\bar a_k\mathbf{1}}=\sft{\Xb\W^\pl\xl},
\end{align*}
which completes the proof.
\end{proof}

\subsection{Proof of Lemma~\ref{lemma Wf eq}}
In the following, we present an additional lemma that incorporates Lemma~\ref{lemma Wf eq}.
\begin{lemma}\label{lemma loss eq}
    Suppose Assumptions~\ref{assume iden} and \ref{assume realizable} hold and loss function $\ell:\R\to\R$ is strictly decreasing. For any finite $\W$, once $\bdata\neq\data$, we have that
    \begin{align}
    \Lc(\W)>\frac{|\Icb|}{n}\ell(1)+\Lcb(\W).\label{loss vs cyc loss}
    \end{align}
    Additionally, we have
    \begin{subequations}
    \begin{align}
        &\min_{\W'\in\Scf^\perp}\Lc(\W'+\W)=\tilde\Lc(\W)=\frac{|\Icb|}{n}\ell(1)+\Lcb(\W),\label{eq cyc 1}\\
        &\min_{\W}\Lc(\W)=\min_{\W}\tilde\Lc(\W)=\frac{|\Icb|}{n}\ell(1)+\min_{\W}\Lcb(\W).\label{eq cyc 2}
    \end{align}
    \end{subequations}  
\end{lemma}
\begin{proof}
We start with proving that for any finite $\W$, $\Lc(\W)\geq\frac{|\Icb|}{n}\ell(1)+\Lcb(\W)$. Let $\W^\perp=\prj_{\Scf^\perp}(\W)$, $\W^\pl=\prj_{\Scf}(\W)$ where we have $\W=\W^\perp+\W^\pl$. Let $\ab_i=\X_i\W^\perp\xli$, $\bb_i=\X_i\W^\pl\xli$ and $\s_i=\sft{\X_i\W\xli}=\sft{\ab_i+\bb_i}$. Following \eqref{def erm under assum iden} and the definition of $\Lcb(\W)$, training losses obey
\[
\Lc(\W)=\frac{1}{n}\sum_{i=1}^n\ell\left(\sum_{t\in\Oc_i}s_{it}\right)\quad\text{and}\quad\Lcb(\W)=\frac{1}{n}\sum_{i\in\Ic}\ell\left(\frac{\sum_{t\in\Oc_i}s_{it}}{\sum_{t\in\Rc_i}s_{it}}\right).
\]
From proof of Lemma~\ref{lemma Lcb eq} (more specifically \eqref{ortho W}), we have that for any $t\in\Rc_i$, 
\[
\x_{it}^\top\W^\perp\xli=\bar a_i~\Longrightarrow a_{it}=\bar a_i,~\forall t\in\Rc_i
\]
where $\bar a_i$ is some constant associated with $\W^\perp$. Then for any $i\in[n]$, we get
\begin{align*}
&\sum_{t\in\Oc_i}s_{it}=\frac{\sum_{t\in\Oc_i}e^{\bar a_{i}+b_{it}}}{\sum_{t\in\Rc_i}e^{\bar a_{i}+b_{it}}+\sum_{t\in\Rcb_i}e^{a_{it}+b_{it}}}=\frac{\sum_{t\in\Oc_i}e^{b_{it}}}{\sum_{t\in\Rc_i}e^{b_{it}}+\sum_{t\in\Rcb_i}e^{b_{it}+a_{it}-\bar a_i}}\leq\frac{\sum_{t\in\Oc_i}e^{b_{it}}}{\sum_{t\in\Rc_i}e^{b_{it}}},\\
&\frac{\sum_{t\in\Oc_i}s_{it}}{\sum_{t\in\Rc_i}s_{it}}=\frac{\sum_{t\in\Oc_i}e^{\bar a_{i}+b_{it}}}{\sum_{t\in\Rc_i}e^{\bar a_{i}+b_{it}}}=\frac{\sum_{t\in\Oc_i}e^{b_{it}}}{\sum_{t\in\Rc_i}e^{b_{it}}}.
\end{align*}
Next following \eqref{def Ic} we have that for $i\in\Icb$, $\Rc_i=\Oc_i$ and therefore 
\[
\frac{\sum_{t\in\Oc_i}s_{it}}{\sum_{t\in\Rc_i}s_{it}}=1\quad\text{for all}\quad i\in\Icb.
\]
Note that since $a_{it},b_{it}$ are finite and $\bdata\neq\data$, there exists $i\in[n]$ such that $\sum_{t\in\Oc_i}s_{it}<\frac{\sum_{t\in\Oc_i}s_{it}}{\sum_{t\in\Rc_i}s_{it}}$.
Given strictly decreasing loss function $\ell$ and any finite $\W$, the training risks obey
\begin{align}
\Lc(\W)=\frac{1}{n}\sum_{i=1}^n\ell\left(\sum_{i\in\Oc_i}s_{it}\right)>\frac{1}{n}\sum_{i\in\Icb}\ell(1) + \frac{1}{n}\sum_{i\in\Ic}\ell\left(\frac{\sum_{t\in\Oc_i}s_{it}}{\sum_{t\in\Rc_i}s_{it}}\right)=\frac{|\Icb|}{n}\ell(1)+\Lcb(\W).\nn
\end{align}
It completes the proof of \eqref{loss vs cyc loss}.

We next show that 
\[
\min_{\W'\in\Scf^\perp}\Lc(\W'+\W)=\tilde\Lc(\W)=\frac{|\Icb|}{n}\ell(1)+\Lcb(\W).
\]
Recap from Theorem~\ref{thm cyclic gd} that $\tilde\Lc(\W)=\lim_{R\to\infty}\Lc(\W+R\cdot\Wm)$. 
Let $\W^\perp=\prj_{\Scf^\perp}(\W)$, $\W^\pl=\prj_{\Scf}(\W)$ where we have $\W=\W^\perp+\W^\pl$. Let $\ab_i=\X_i\W^\perp\xli$, $\bb_i=\X_i\W^\pl\xli$, $\cb_i=\X_i\Wm\xli$ and $\s^R_i=\sft{\X_i(R\cdot\Wm+\W)\xli}=\sft{\ab_i+\bb_i+R\cdot \cb_i}$. Similarly, for any $t\in\Rc_i$, 
\[
\x_{it}^\top\W^\perp\xli=\bar a_i~\Longrightarrow a_{it}=\bar a_i,~\forall t\in\Rc_i
\]
where $\bar a_i$ is some constant associated with $\W^\perp$. Additionally, since $\Wm$ follows \eqref{graph svm}, we have
\begin{align}\label{cit}
(\x_{i\tau}-\x_{it})^\top\Wm\xli\begin{cases}
    =0 &  \forall t\in\Rc_i\\
    \geq1 & \forall t\in\Rcb_i 
\end{cases},\quad\text{for all }\tau\in\Rc_i\quad\Longrightarrow\quad \begin{cases}
    c_{it}=\bar c_i,&~\forall t\in\Rc_i,\\
    c_{it}\leq \bar c_{i}-1,&~\forall t\in\Rcb_i.
\end{cases}
\end{align}
Then for any $i\in[n]$, we get
\[
\sum_{t\in\Oc_i}s_{it}^R=\frac{\sum_{t\in\Oc_i}e^{\bar a_{i}+b_{it}+R\bar c_{i}}}{\sum_{t\in\Rc_i}e^{\bar a_{i}+b_{it}+R\bar c_{i}}+\sum_{t\in\Rcb_i}e^{a_{it}+b_{it}+Rc_{it}}}=\frac{\sum_{t\in\Oc_i}e^{b_{it}}}{\sum_{t\in\Rc_i}e^{b_{it}}+\sum_{t\in\Rcb_i}e^{b_{it}+a_{it}-\bar a_i+R(c_{it}-\bar c_i)}}\leq\frac{\sum_{t\in\Oc_i}e^{b_{it}}}{\sum_{t\in\Rc_i}e^{b_{it}}}.
\]

\noindent\textbf{Case 1: $\Wm=0$.} Then for all $i\in[n]$, $\Rcb_i=\emptyset$ and the equality holds for all $i\in[n]$.  

\noindent\textbf{Case 2: $\Wm\neq0$.}
Since $a_{it},b_{it}$ are finite and $c_{it}-\bar c_i\leq-1$ for $t\in\Rcb_i$ following \eqref{cit}, the equality holds when $R\to\infty$, and therefore we have for any $i\in[n]$,
\[
\lim_{R\to\infty}\sum_{t\in\Oc_i}s_{it}^R=\frac{\sum_{t\in\Oc_i}e^{b_{it}}}{\sum_{t\in\Rc_i}e^{b_{it}}}
\]
and
\begin{align}
\tilde\Lc(\W)=\lim_{R\to\infty}\Lc(R\cdot\Wm+\W)&=\lim_{R\to\infty}\frac{1}{n}\sum_{i=1}^n\ell\left(\sum_{i\in\Oc_i}s_{it}^R\right)\nn\\
&=\frac{|\Icb|}{n}\ell(1)+\frac{1}{n}\sum_{i\in\Ic}\ell\left(\frac{\sum_{t\in\Oc_i}e^{b_{it}}}{\sum_{t\in\Rc_i}e^{b_{it}}}\right)\nn\\
&=\frac{|\Icb|}{n}\ell(1)+\Lcb(\W).\label{eq cyc 1 sub}
\end{align}
Additionally, we have for any $\W'\in\Scf^\perp$,
\[
\Lc(\W'+\W)\geq\frac{|\Icb|}{n}\ell(1)+\Lcb(\W'+\W)=\frac{|\Icb|}{n}\ell(1)+\Lcb(\W)
\]
where the inequality uses \eqref{loss vs cyc loss} and the equality comes from Lemma~\ref{lemma Lcb eq}. Since bound is achievable (by choosing $\W'=\lim_{R\to\infty}R\cdot\Wm$ as in \eqref{eq cyc 1 sub}), then combining it with \eqref{eq cyc 1 sub} completes the proof of \eqref{eq cyc 1}. 
\eqref{eq cyc 2} is directly obtained from \eqref{eq cyc 1}.
\end{proof}

%% file: sec/fig_intro_app.tex
\begin{figure*}[tb] 
    \centering    \includegraphics[width=0.89\textwidth]{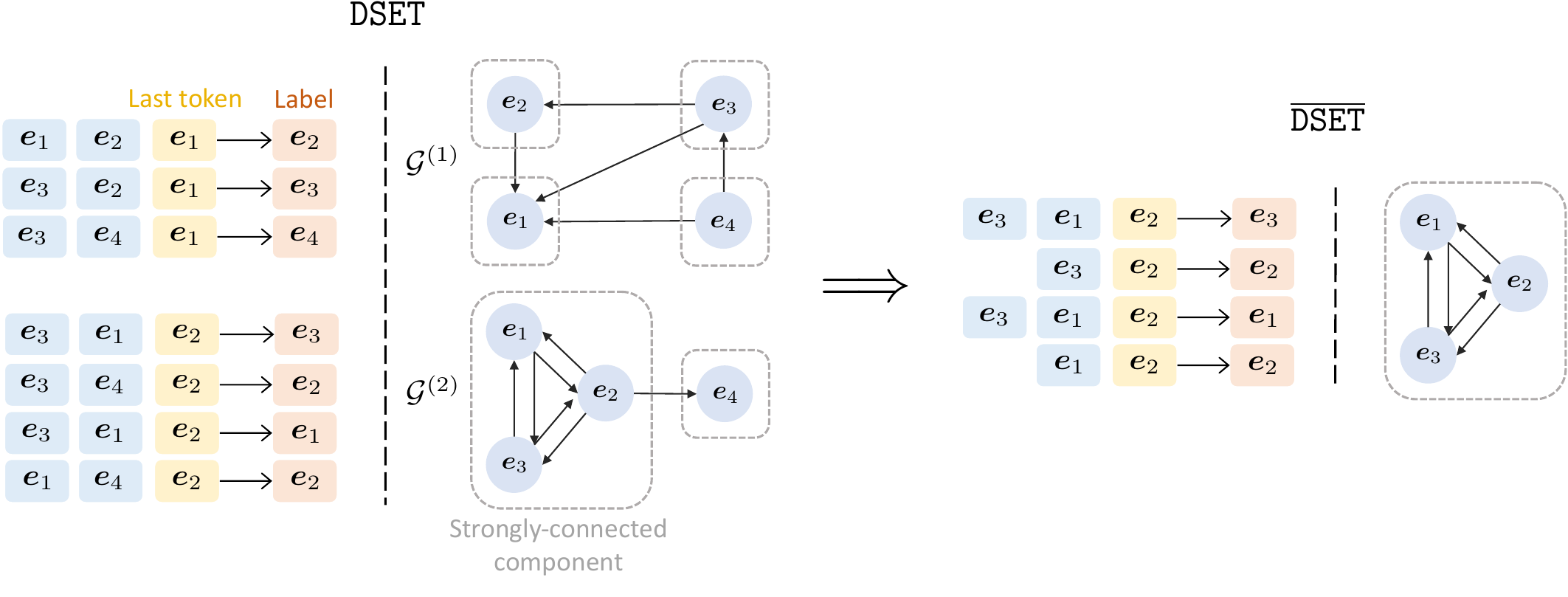}
    \caption{\small{{Illustration of token-priority graph (TPG).  Given the input sequences and labels (next tokens), we construct the TPGs $\{\Gc^{(k)}\}_{k=1}^K$ according to the last token. \textbf{Left hand side:} 
    Two TPGs $\Gc^{(1)}$ and $\Gc^{(2)}$ are constructed using the samples with $\eb_1$ and $\eb_2$ as the last tokens, respectively. In each graph, directed edges (label$\to$input token) are added, and based on the token relations, each graph can be partitioned into different strongly-connected components (SCCs, highlighted as dashed grey rectangles). \textbf{Right hand side:} We consider a cyclic subdataset $\bdata$ by removing all the singleton SCCs, as well as their corresponding edges (see Definition~\ref{cyc sub}). Then, $\bdata$ contains non-singleton SCCs only as shown on the right.}
    }}
    \label{fig:intro app}
\end{figure*}

%% file: supp/gd_proof.tex
\ifarxiv
\section{Global Convergence of Gradient Descent}
\else
\section{GLOBAL CONVERGENCE OF GRADIENT DESCENT}
\fi

\subsection{Supporting Results under the Setting of Theorem~\ref{thm cyclic gd}}
In this section, we introduce results useful for the main proof. 
Recap the setting of Theorem~\ref{thm cyclic gd} where $\ell(u)=-\log(u)$. Therefore loss defined in \eqref{def erm under assum iden} is 
\begin{align}\label{def erm loss}
\Lc(\W)=-\frac{1}{n}\sum_{i=1}^n\log\left(\sum_{t\in\Oc_i}s_{it}^{\W}\right)
\end{align}
where $\s_i^{\W}=\sft{\X_i\W\xli}$ and $\Oc_i$'s follow \eqref{def Oc Rc}.

\smallskip
\noindent$\bullet$ \textbf{$\nabla\Lc(\W)$ under the setting of Theorem~\ref{thm cyclic gd}.} 
For any $\W\in\R^{d\times d}$, let $\hb_i=\X_i\W\xli,\s_{i} = \sft{\hb_i},\bgam_i=\X_i\cb_{y_i}$. 
\begin{equation}\label{gradient}
\begin{split}
   \nabla \Lc(\W) 
   &= \frac{1}{n} \sum_{i=1}^n \ell'(\bgam_i^\top \s_i) \X_i^{\top} \sfp{\hb_i}\bgam_i \xli^{\top} \\ 
   &= \frac{1}{n} \sum_{i=1}^n -\frac{1}{\bgam_i^\top\s_i} \X_i^{\top} (\diag{\s_i}-\s_i\s_i^\top)\bgam_i \xli^{\top} \\ 
   &= \frac{1}{n}\sum_{i=1}^n\sum_{t\in \Ocb_i}s_{it}(\x_{it}-\eb_{y_i})\xli^\top
\end{split}
\end{equation}
where the last equation uses the fact that for any example $(\X_i,y_i)\in\data$, $i\in[n]$,
\begin{equation}
\begin{split}
    \frac{\X_i^\top(\diag{\s_i} - \s_i\s_i^{\top})\bgam_i}{\bgam_i^\top\s_i}
    &= \frac{\X_i^\top\diag{\s_i}\bgam_i}{\bgam_i^\top\s_i}-\X_i^\top\s_i\\ &=\frac{\sum_{t\in\Oc_i}s_{it}\eb_{y_i}}{\sum_{t\in\Oc_i}s_{it}}-\X_i^\top\s_i\\
    &=\eb_{y_i}-\X_i^\top\s_i\\
    &=\sum_{t\in \Ocb_i}s_{it}(\eb_{y_i}-\x_{it}).\nn
\end{split}
\end{equation}
Here, the second equality comes from \eqref{def score}.

\smallskip
\noindent$\bullet$ \textbf{Lipschitzness of $\nabla\Lc(\W)$ in \eqref{gradient}.} 
For any $\W, \dot{\W} \in \R^{d \times d}$, let $\s_i=\sft{\X_i\W\xli}$ and $\dot\s_i=\sft{\X_i\dot\W\xli}$. Consider bounded tokens and let $M:=\max_{k\in[K]}\|\eb_k\|$. Following \eqref{gradient}, we have:
\begin{equation}
\begin{split} \label{eq:gradlip}
    \tf{\nabla \mathcal{L}(\W) - \nabla \mathcal{L}(\dot{\W})}
        &=\tf{\frac{1}{n}\sum_{i=1}^n\sum_{t\in\Ocb_i}(s_{it}-\dot s_{it})(\x_{it}-\eb_{y_i})\xli^\top}\\
        & \leq \frac{1}{n}\sum_{i=1}^n\sum_{t\in \Ocb_i}|s_{it}-\dot s_{it}|~\tf{(\x_{it}-\eb_{y_i})\xli^\top}\\
        & \leq \frac{2M^2}{n}\sum_{i=1}^n\sum_{t\in \Ocb_i}|s_{it}-\dot s_{it}| \\
        & \leq \frac{2M^2}{n}\sum_{i=1}^n \|\s_{i}-\dot \s_{i}\|_1 \\
        & \leq \frac{2M^2}{n}\sum_{i=1}^n\sqrt{T_i}\cdot\|\s_{i}-\dot \s_{i}\|.
\end{split}
\end{equation}
Next for any $\s,\dot\s$, we get
\begin{equation}
    \begin{split}
        \|\s-\dot\s\|&=\|\sft{\X\W\xl}-\sft{\X\dot\W\xl}\|\\
        &\leq \|\X\W\xl-\X\dot\W\xl\|\\
        &\leq M^2\tf{\W-\dot\W}.
    \end{split}
\end{equation}
Combining results in that 
\begin{equation}
    \begin{split}
        \tf{\nabla\Lc(\W)-\nabla\Lc(\dot\W)}\leq 2M^4\sqrt{\Tmax}\cdot\tf{\W-\dot\W}
    \end{split}
\end{equation}
where $\Tmax:= \max_{i\in[n]}T_i$. Then let 
\begin{align}
    L:=2M^4\sqrt{\Tmax}\label{lip term}
\end{align}
and $\nabla\Lc(\W)$ is $L$-Lipschitz continuous. 

\subsection{Proof of Lemma~\ref{lemma cvx}}
In this subsection, we provide and prove a general version of Lemma \ref{lemma cvx}. To to that, we first introduce the following new subspaces.
\begin{definition}\label{def svm + cyc subspace}
    Define the subspace $\Scall$ as the span of all matrices $(\eb_i-\eb_j)\eb_k^\top$ for all $(i \rightarrow j)\in\Gck$ and $k\in[K]$.
\end{definition}
\begin{definition}[Cyclic subspace (Restated)]\label{def cyc subspace restated}
     Define cyclic subspace $\Scf$ as the span of all matrices $(\eb_i-\eb_j)\eb_k^\top$ for all $(i\asymp j)\in\Gck$ and $k\in[K]$.
\end{definition}
Observe that $\Scf$ is a subspace of $\Scall$. This is because if two nodes $i,j \in \Gck$ and $i \asymp j$, this also implies that $i \rightarrow j$ and $j \rightarrow i$.
\begin{definition}[SVM subspace]\label{def svm subspace}
     Define svm subspace $\Scsvm$ as the orthogonal complement of the subspace $\Scf$ inside the subspace $\Scall$.
\end{definition}
\begin{lemma}
We have the following:
\begin{enumerate}[label=(\roman*)]
    \item Recall the definition of $\Wm$ in \eqref{graph svm}. $\Wm \in \Scsvm$.
    \item Let $\Scall^\perp$ be the orthogonal complement of $\Scall$ inside $\R^{d \times d}$. Then, 
    \begin{align*}
\tf{\prj_{\Scall^\perp} \left(\nabla \Lc(\W)\right)} = 0, \quad \forall \W \in \R^{d \times d}.
    \end{align*}
\end{enumerate}
\end{lemma}

\begin{proof}

\noindent $\bullet$ \textbf{(i):} Recall the definition of $\Wm$:
    \begin{align*}
&\Wm=\arg\min_{\W}\tf{\W}\quad\\
&\text{s.t.}~~(\eb_i-\eb_j)^{\top}\W\eb_k
\begin{cases}=0\quad \forall(i\asymp j) \in\Gck\\
\geq 1\quad \forall(i\Rightarrow j)\in\Gck
\end{cases}
\text{for all}\quad k\in[K].\nn
\end{align*}
Assume that the statement is not correct. Then, either $\tf{\prj_{\Scf}(\Wm)} >0$ or $\tf{\prj_{\Scall^\perp}(\Wm)} > 0$. 

By definition, $\tf{\prj_{\Scf}(\Wm)}=0$ since for all $(i \asymp j) \in \Gck$, $(\eb_i - \eb_j)^\top \Wm \eb_k = 0$.

On the other hand, if $\tf{\prj_{\Scall^\perp}(\Wm)} > 0$, then $\Wm - \prj_{\Scall^\perp}(\Wm)$ also satisfies all of the constraints of \eqref{graph svm}, and 
\begin{align*}
    \tf{\Wm - \prj_{\Scall^\perp}(\Wm)} < \tf{\Wm},
\end{align*}
which is a contradiction. Therefore, $\Wm \in \Scsvm$.

\noindent $\bullet$ \textbf{(ii):} From \eqref{gradient}, we know that
\begin{align*}
    \nabla \Lc(\W) = \frac{1}{n}\sum_{i=1}^n\sum_{t\in \Ocb_i}s_{it}(\x_{it}-\eb_{y_i})\xli^\top
\end{align*}
where $\Ocb_i$ is given by \eqref{def Oc Rc}. By definition of $\Scall$, $\tf{\prj_{\Scall^\perp}(\x_{it} - \eb_{y_i})\xli^\top} = 0$ for any $i \in [n]$ and $t \in \Ocb_i$. As $\nabla \Lc(\W)$ is the summation of these terms, the advertised result is proved.
\end{proof}

Now, we are ready to prove a stronger version of Lemma \ref{lemma cvx}.
\begin{lemma}[Stronger version of Lemma \ref{lemma cvx}]\label{lemma cvx app}
    Suppose Assumptions \ref{assume iden} and \ref{assume realizable} hold and consider the log-loss $\ell(u) = -\log(u)$, then $\Lc(\W)$ is convex. Furthermore, $\Lc(\W)$ is {strictly convex} on $ \Scall$.
\end{lemma}
\begin{proof}
Let $\Sc_K$ be the span of all $\eb_i \eb_j^\top$ where $i, j \in [K]$.

\noindent$\bullet$ \textbf{First Case: $\W \in \Sc_K$.} Let $g: \Sc_K \xrightarrow[]{} \R^{K \times K}$ such that $g(\W) = \Eb \W \Eb^{\top}$. By definition, this function is linear. In addition to that, this function $g$ is invertible by Assumption \ref{assume iden} and the domain of the function is $\Sc_K$. Note that Assumption \ref{assume iden} ensures $\text{rank}(\Eb) = K$.

Let  $\Eb' = \Cb' = \Iden_k$, $(\X'_i, y_i')$ be a $\data$ such that $y_i' = y_i$ and $\X_i' = \X_i \Eb^{\dagger}$. Then, for any $\W' \in \R^{K \times K}$, we have the following:
\begin{align*}
    \Lc \circ g^{-1}(\W') = \frac{1}{n}\sum_{i=1}^n \ell\left((\cb'_{y_i})^\top (\X_i')^\top \sft{\X_i' \W' \bar{\x_i}'}\right)
\end{align*}
Using Lemma \ref{lemma linear map strict convexity} and \ref{lemma vect strict convexity}, we know that $\Lc \circ g^{-1}(\W')$ is convex on $\R^{K \times K}$ and strictly convex on $g(\Scall)$. Using these two facts and Lemma \ref{lemma linear map strict convexity}, we have $\Lc(\W)$ is convex on $\Sc_K $ and strictly convex on $\Scall \cap \Sc_K = \Scall$. 

\noindent$\bullet$ \textbf{Second Case: $\W \not\in \Sc_K$.} By definition of loss function in \eqref{erm}, we have 
\begin{align}\label{loss equivalent projected}
    \Lc(\W) = \frac{1}{n}\sum_{i=1}^n \ell(\cb_{y_i}^\top\X_i^\top \sft{\X_i\W\xl_i}) = \frac{1}{n}\sum_{i=1}^n \ell(\cb_{y_i}^\top\X_i^\top \sft{\X_i \prj_{\Sc_K}(\W)\xl_i}) = \Lc(\prj_{\Sc_K}(\W)) 
\end{align}
Let $\W_1, \W_2 \in \R^{d \times d}$ be arbitrary variables. For any $0 < \lambda < 1$, we have the following:
\begin{align}\label{linearity of projection}
    \Lc(\lambda \W_1 + (1-\lambda)\W_2 ) = \Lc(\lambda \prj_{\Sc_K}(\W_1) + \lambda \prj_{\Sc_K}(\W_2))  
\end{align}
Then, using \eqref{loss equivalent projected} and \eqref{linearity of projection}, we have the following:
\begin{align*}
    \lambda \Lc(\W_1) + (1-\lambda) \Lc(\W_2) &= \lambda \Lc(\prj_{\Sc_K} (\W_1)) + (1-\lambda) \Lc(\prj_{\Sc_K}(\W_2)) \\
    & \stackrel{(a)}{\geq} \Lc(\lambda \prj_{\Sc_K}(\W_1) + \lambda \prj_{\Sc_K}(\W_2)) = \Lc(\lambda \W_1 + (1-\lambda)\W_2 )
\end{align*}
where (a) follows from the convexity of $\Lc(\W)$ inside $\Sc_K$. This implies that $\Lc(\W)$ is convex when $\W \not\in \Sc_K$. Note that $\Scall \subset \Sc_K$, therefore we do not look at the strict convexity in this case. 
\end{proof}

\begin{lemma}\label{lemma linear map strict convexity}
    Let $T : \Xc \xrightarrow[]{} \Yc$ be an invertible linear map. If a function $f : \Yc \xrightarrow[]{} \R$ is convex/strictly convex on $\Yc$, then $f \circ T (x)$ is a convex/strictly convex function on $\Xc$. 
\end{lemma}
\begin{proof}
    Let $x_1 \neq x_2 \in \Xc$ be arbitrary variables. Let $y_1 = T(x_1)$ and $y_2 = T(x_2)$. Since $T$ is an invertible map, $y_1 \neq y_2$. Since $T$ is a linear map, $ T(\lambda x_1 + (1-\lambda)x_2) = \lambda y_1 + (1-\lambda) y_2$ for $0 < \lambda < 1$. Then, we obtain the following
    \begin{align*}
        \lambda (f \circ T(x_1)) + (1-\lambda) (f \circ T(x_2)) &= \lambda f (y_1) + (1-\lambda) f(y_2) \\
        & \stackrel{(a)}{>} f(\lambda y_1 + (1-\lambda) y_2) \\ 
        &= f \circ T(\lambda x_1 + (1-\lambda) x_2)
    \end{align*}
    where (a) follows from the strict convexity of the function $f$. This implies that $f \circ T (x)$ is a strictly convex function on $\Xc$. Note that if $y_1 = y_2$, then we cannot achieve (a). Additionally, if $f$ is convex instead of strictly convex, then $>$ in (a) is changed to $\geq$, and $f \circ T(x)$ is convex.
\end{proof}
\begin{lemma}\label{lemma vect strict convexity}
    Suppose that Assumption \ref{assume realizable} holds and $\Eb = \boldsymbol{I}_d$. Let $f : \R^{d \times d} \xrightarrow[]{} \R^{d^2}$ be a linear transformation defined as $f(\W) = \vb$ where $v_{i\times d + j} = {\eb_i^\top \W \eb_j}$. Then, $\Lc \circ f^{-1}(\vb)$ is convex. Furthermore, $\Lc \circ f^{-1}(\vb)$ is strictly convex on $f(\Scall)$, where $\Scall$ is defined in Definition~\ref{def svm + cyc subspace}.
\end{lemma}
\begin{proof}
    \noindent$\bullet$ \textbf{We first prove that  $\Lc \circ f^{-1}(\vb)$  is convex. }
Let $\bar{\ell}: \R^{d^2} \times \R^{T \times d} \times \R \xrightarrow[]{} \R$ be defined as follows:
    \begin{align*}
        \bar{\ell}(\vb, \X, y) := \ell\left(\cb_{y}^\top \X^\top \sft{\X \big(f^{-1}(\vb)\big) \bar{\x}}\right).  
    \end{align*}
    Then, using \eqref{erm}, we have the following:
    \begin{align}\label{defn ellbar strict convexity}
        \Lc \circ f^{-1}(\vb) = \frac{1}{n}\sum_{i=1}^n \ell\left(\cb_{y_i}^\top \X_i^\top \sft{\X_i \big(f^{-1}(\vb)\big) \bar{\x_i}}\right) = \frac{1}{n}\sum_{i=1}^n \bar{\ell}(\vb, \X_i, y_i).  
    \end{align}
    Note that the summation of convex functions is convex. Therefore, it is sufficient to prove the convexity of $\Lc \circ f^{-1}(\vb)$ by proving the convexity of $\bar{\ell}(\vb, \X, y)$ for an arbitrary pair of input sequence and label $(\X, y)$. For the simplicity of notation, we use $\bar{\ell}(\vb)$ instead of $\bar{\ell}(\vb, \X, y)$. Let $m_j$ be the number of token ID $j$ inside input sequence $\X$ for $j \in [K]$. Let $k$ be the last token of $\X$. 
    By Assumption \ref{assume iden} and log-loss, we know that
    \begin{align*}
       \bar{\ell}(\vb) := \bar{\ell}(\vb, \X, y) = -\log\left(\frac{m_y\cdot e^{\vb_{y\times d + k}}}{\sum_{j\in[K]} m_j\cdot e^{\vb_{j\times d + k}}} \right) = \log\left(\sum_{j \in [K]} m_j\cdot e^{\vb_{j\times d + k}}\right) -\log (m_y\cdot e^{v_{y\times d + k}}).
   \end{align*}   
   Let $\z \in \R^{d^2}$ be a vector such that the $(j\times d + k)^{\text{th}}$ element of $\z$ is $z_{j \times d +k}=m_j\cdot e^{v_{j\times d +k}}$ for $k \in [K]$, otherwise $z_i = 0$. Then,
   the Hessian matrix of $\bar{\ell}(\vb)$ is 
\begin{align*}
    \nabla^2 \bar{\ell}(\vb) = \frac{1}{(\boldsymbol{1}^\top \z)^2}\left((\boldsymbol{1}^\top \z) \text{diag}(\z) - \z \z^\top \right)
\end{align*}
For any $\ub \in \R^{d^2}$, we obtain that 
\begin{align}\label{CSI}
    \ub^\top \nabla^2 \bar{\ell}(\vb) \ub = \frac{1}{(\boldsymbol{1}^\top \z)^2} \left( \left(\sum_{j=1}^{d^2} z_j \right) \left(\sum_{j=1}^{d^2} u_j^2 z_j\right)  - \left(\sum_{j=1}^{d^2} u_j z_j \right)^2 \right) \geq 0  .
\end{align}
Since $\z_i\geq0$, $i\in[d^2]$, \eqref{CSI} follows from the Cauchy-Schwarz inequality $(\bal^\top \bal)(\bt ^\top \bt) \geq (\bal^\top \bt)^2$ applied to the vectors with $\alpha_i = u_i \sqrt{z_i}$ and $\beta_i = \sqrt{z_i}$. The equality condition holds $k \bal = \bt$ for $k \neq 0$. This means that $\bar{\ell}(\vb)$ is convex.

\noindent$\bullet$ \textbf{Next, we will show that  $\Lc \circ f^{-1}(\vb)$ is strictly convex on $f(\Scall)$.} Assume that $\Lc \circ f^{-1}(\vb)$ is not strictly convex on $f(\Scall)$. Using the convexity of $\Lc \circ f^{-1}(\vb)$, this implies that there exist $\ub, \vb \in f(\Scall)$, $\|{\ub}\|_2 > 0$ such that 
\begin{align*}
    \ub^\top \left( \nabla^2 \Lc \circ f^{-1}(\vb) \right) \ub = 0
\end{align*}
Using the convexity of $\bar{\ell}(\vb)$ and \eqref{defn ellbar strict convexity}, we have the following:
\begin{align}\label{strict convexity necessary condition}
    \ub^{\top} \left( \nabla^2 \bar{\ell}(\vb, \X_i, \y_i) \right) \ub = 0 \qquad \forall i \in [n]
\end{align}

Now, we are going to prove that $\|\ub\|_2 = 0$ if \eqref{strict convexity necessary condition} holds. As $\ub \in f(\Scall)$, there exists $\W \in \Scall$ such that $f(\W) = \ub$. As the function $f$ preserves the norm, $\tf{\W} > 0$. By definition of $\Scall$, there exist $\bar{i},\bar{j}, \bar{k} \in [K]$ and $(\X_{\bar{n}}, y_{\bar{n}}) \in \data$ such that $\<(\eb_{\bar{i}} - \eb_{\bar{j}}) \eb_{\bar{k}}^\top, \W \> > 0$, $\X_{\bar{n}}$ includes the $\bar{j}^{\text{th}}$ token, the last token of $\X_{\bar{n}}$ is the $\bar{k}^{\text{th}}$ token, and $y_{\bar{n}} = \bar{i}$. On the other hand, by Assumption \ref{assume realizable}, $z_{\bar{i}\times d + \bar{k}}$ and $z_{\bar{j} \times d + \bar{k}}$ in \eqref{CSI} are non-zero for this input sequence $\X_{\bar{n}}$. Using the equality condition of Cauchy-Schwartz Inequality in \eqref{CSI}, we obtain that $u_{\bar{i} \times d + \bar{k}} - u_{\bar{j} \times d + \bar{k}} = 0$. This implies that 
\begin{align*}
    0 &= u_{\bar{i} \times d + \bar{k}} - u_{\bar{j} \times d + \bar{k}} \\
    &= \eb_{\bar{i}}^\top \W \eb_{\bar{k}} - \eb_{\bar{j}}^\top \W \eb_{\bar{k}} \\
    &= (\eb_{\bar{i}} - \eb_{\bar{j}})^\top \W \eb_{\bar{k}} = \<(\eb_{\bar{i}} - \eb_{\bar{j}}) \eb_{\bar{k}}^\top, \W \>
\end{align*}
which contradicts with the fact that $\| \ub \|_2 > 0$. This completes the proof.
\end{proof}
\subsection{Divergence of $\tf{\bf{\Pi}_{\Scsvm}({\bf{W}}(\tau))}$}
We first introduce the following lemmas establishing the descent property of gradient descent for $\Lc(\W)$ (Lemma~\ref{lemma des}) and the correlation between $\nabla\Lc(\W)$ and the solution of \eqref{graph svm} $\Wm$ (Lemma~\ref{lemma cyc neg corr}) under the setting of Theorem~\ref{thm cyclic gd}. The proofs in this section follow Appendix~B.1 of \cite{tarzanagh2023transformers}.
\begin{lemma}[Descent Lemma]\label{lemma des} 
Consider the loss in \eqref{def erm loss} and choose step size {$\eta \leq 1 / L$} where $L$ is the Lipschitzness of $\nabla\Lc(\W)$ defined in \eqref{lip term}. Then from any initialization $\W(0)$, Algorithm \ref{algo gd} satisfies: 
\begin{equation}
     \Lc(\W(\tau+1)) - \Lc(\W(\tau)) \leq -\frac{\eta}{2} \tf{\nabla \Lc(\W(\tau))}^2\nn
\end{equation}
for all $\tau \geq 0$. Additionally, it holds that $\sum_{\tau=0}^{\infty} \tf {\nabla \Lc(\W(\tau))}^2 < \infty$, and $\lim_{\tau \to \infty} \tf{\nabla \Lc(\W(\tau))}^2 = 0$
\end{lemma}
    \begin{proof}
From \eqref{algo gd}, for $\tau\geq0$, we have that $\W(\tau+1)=\W(\tau)-\eta\nabla\Lc(\W(\tau))$. Since $\mathcal{L}(\W)$ is $L$-smooth with $L$ defined in \eqref{lip term}, we get
    \begin{equation*}
    \begin{split}
        \Lc(\W(\tau+1)) 
        & \leq \Lc(\W(\tau)) + \langle {\nabla \Lc(\W(\tau)), \W(\tau+1) - \W(\tau) }\rangle + \frac{L}{2}\tf{\W(\tau+1) - \W(\tau)}^2 \\ 
        & = \Lc (\W(\tau)) - \eta\cdot \tf{\nabla \Lc(\W(\tau))}^2 + \frac{L \eta^2}{2}\tf{\nabla \Lc(\W(\tau))}^2 \\ 
        & = \Lc (\W(\tau)) - \eta\left(1 -\frac{L \eta}{2}\right)  \tf{\nabla \Lc(\W(\tau))}^2 \\ 
        & \leq \Lc (\W(\tau)) - \frac{\eta}{2}  \tf{\nabla \Lc(\W(\tau))}^2.  
    \end{split}
    \end{equation*}

The inequality above also indicates that
\begin{equation}
    \sum_{\tau = 0}^{\infty} \tf{\nabla \Lc(\W(\tau))}^2 \leq \frac{2}{\eta} (\mathcal{L}(\W(0)) - \mathcal{L}^*)<\infty,~~~\text{and}~~~ \lim_{\tau \to \infty} \tf{\nabla \Lc(\W(\tau))}^2 = 0.\nn
\end{equation}
\end{proof}
\begin{lemma} \label{lemma cyc neg corr} 
Let $\Wm$ be the SVM solution of \eqref{graph svm} and suppose $\Wm\neq0$. For any $\W$ with $\tf{\W}<\infty$, the training loss $\Lc(\W)$ in \eqref{def erm loss} obeys $\langle {\nabla \mathcal{L}(\W), \Wm} \rangle < 0$. Equivalently, $\langle {\prj_{\Scsvm}(\nabla \mathcal{L}(\W)), \Wm} \rangle < 0$.
\end{lemma}

\begin{proof} 
Recap $\Oc_i,\Ocb_i,\Rc_i,\Rcb_i,i\in[n]$ in \eqref{def Oc Rc}. From \eqref{gradient}, for any $\W\in\R^{d\times d}$, we obtain the gradient
\[
\nabla\Lc(\W)=\frac{1}{n}\sum_{i=1}^n\sum_{t\in\Ocb_i}s_{it}(\x_{it}-\eb_{y_i})\xli^\top.
\]
Then
\begin{align*}
    \li\nabla\Lc(\W),\Wm\ri&=\frac{1}{n}\sum_{i\in[n]}\sum_{t\in\Ocb_i}\li s_{it}(\x_{it}-\eb_{y_i})\xli^\top,\Wm\ri\\
    &=\frac{1}{n}\sum_{i\in[n]}\sum_{t\in\Ocb_i}s_{it}\cdot\text{trace}\left((\Wm)^\top (\x_{it}-\eb_{y_i})\xli^\top\right)\\
    &=\frac{1}{n}\sum_{i\in[n]}\sum_{t\in\Ocb_i}s_{it}\cdot (\x_{it}-\eb_{y_i})^\top\Wm\xli.
\end{align*}
From the \eqref{graph svm} formulation, we have that $(\x_{it}-\eb_{y_i})^\top\Wm\xli=0$ for $t\in\Rc_i$ and $(\x_{it}-\eb_{y_i})^\top\Wm\xli \leq -1$ for $t\in\Rcb_i$. 
Then $\Wm\neq0$ ensures that there exists $i\in[n]$ such that $\Rcb_i\neq\emptyset$, which implies that 
\[
\li\nabla\Lc(\W),\Wm\ri<0.
\]
Using the fact that $\Wm\in\Scsvm$ (Lemma~\ref{lemma ortho}) completes the proof.
\end{proof}
The next theorem proves the divergence of norm of the iterates $\W(\tau)$.
\begin{theorem}\label{thm diverge}
    Consider the same setting as in Theorem~\ref{thm cyclic gd}, then there is no finite $\W \in \R^{d \times d}$ satisfying $\grad{\W} = 0$. Furthermore, Algorithm~\ref{algo gd} with the step size $\eta\leq1/L$ where $L$ is the Lipschitzness of $\nabla\Lc(\W)$ defined in \eqref{lip term} and any starting point $\W(0)$ satisfies $\lim_{\tau\to\infty}\tf{\prj_{\Scsvm}(\W(\tau))}=\infty$. 
\end{theorem}
\begin{proof}
Following Lemma \ref{lemma des}, when using log-loss $\ell(u) = -\log(u)$, for any starting point $\W(0)$, the Algorithm~\ref{algo gd} satisfies $\lim_{\tau \to \infty} \tf{\nabla \Lc(\W(\tau))}^2 = 0$. Moreover, assume that the first claim is wrong and that there is a finite critical point $\W$ that satisfies $\grad{\W} = 0$. We then have $\langle \prj_{\Scsvm}(\grad{\W}), \Wm \rangle = 0$. This leads to a contradiction with Lemma \ref{lemma cyc neg corr} which says that for any finite $\W$, $ \langle \prj_{\Scsvm}(\nabla \mathcal{L}(\W)), \Wm \rangle < 0$. This implies that $\tf{\prj_{\Scsvm}(\W(\tau))} \to \infty$.
\end{proof}

\subsection{Uniqueness and Finiteness of ${{{\bf{W^{\text{fin}}}}}}$}

\begin{lemma}\label{lemma finite wfin}
    Consider the setting of Theorem~\ref{thm cyclic gd}. $\Wf$ defined in Theorem~\ref{thm cyclic gd} is unique and finite. 
\end{lemma}
\begin{proof} Following Lemma~\ref{lemma Wf eq}, it is equivalent to show $\Wcf$ defined in Def.~\ref{cyc sub} has unique element $\bWf$ and $\bWf=\Wf$ is unique and finite. To start with, recap the definition of $\bdata$ (Definition~\ref{cyc sub}). Denote $\Ic\subset[n]$ as in \eqref{def bdata}, and let $\s_i=\sft{\bar\X_i\W\xli}$ where $(\bar\X_i,y_i)\in\bdata$. What's more, recap the ERM loss from \eqref{def erm loss} and loss function $\ell(u)=-\log(u)$. Then we have 
\begin{align}
\bWf=\arg\min_{\W\in\Scf}\Lcb(\W)\quad\text{where}\quad\Lcb(\W)=\frac{1}{n}\sum_{i\in\Ic}-\log\left(\sum_{t\in\Oc_i}s_{it}\right).\label{wfin def supp}
\end{align}
Different from \eqref{def Oc Rc}, $\Oc_i$ for dataset $\bdata$ is defined as follows:
\[
\Oc_i:=\left\{t~|~x_{it}={y_i},~\x_{it}\in\bar\X_i,~t\in[\bar T_i]\right\},
\]
where $\bar T_i$ is the number of tokens -- tokens that are in the same SCC as label token ${y_i}$ within their corresponding TPG -- in $\bar\X_i$ and if recap the notation of $\Rc_i$ in \eqref{def Oc Rc}, here we have $|\Rc_i|=\bar T_i$.

We will first prove that $\bWf$ is finite by contradiction. Specifically, we will show that for any $\W\in\Scf$ with $\tf{\W}\neq0$, $\lim_{R\to\infty}\Lcb(R\cdot\W)=\infty$ which implies that the optimal solution $\bWf$ has to be finite.

Let $\W\in\Scf$ be arbitrary attention weight. Following the definition of $\Scf$ as in Def.~\ref{def cyc subspace}, we have that $(\eb_{i}-\eb_j)\eb_k^\top\in\Scf$ for all $(i\asymp j)\in\Gck$ and $k\in[K]$. For any $\bar\X$, let $\Oc,\Ocb$ correspond to the token index sets. Then we have
\[
\sum_{t\in\Oc}s_t=\frac{|\Oc|e^{\eb_y^\top(R\cdot\W)\eb_k}}{|\Oc|e^{\eb_y^\top(R\cdot\W)\eb_k}+\sum_{t\in\Ocb}e^{\eb_t^\top(R\cdot\W)\eb_k}}=\frac{1}{1+\sum_{t\in\Ocb}e^{(\eb_t-\eb_y)^\top(R\cdot\W)\eb_k}/|\Oc|}.
\]
Given the sample loss $\ell=-\log\left(\sum_{t\in\Oc}s_t\right)$ and to prevent it from divergence as $R\to\infty$, that is, $\sum_{t\in\Oc}s_t\not\to0$, we have
\[
\sum_{t\in\Ocb}e^{(\eb_t-\eb_y)^\top(R\cdot\W)\eb_k}\not\to\infty\quad\Longrightarrow\quad(\eb_y-\eb_t)^\top \W\eb_k\geq0 \text{ for all }t\in\Ocb
\]
where $\eb_y$ is the label token and $\eb_t$ is any other token in $\Ocb$. Recap from the construction of TPG in Section~\ref{sec ntg}, the directed edge $y\to t$ exists in the graph $\Gck$. Since SCC is bidirectionally reachable, which means there exists route from $t$ to $y$, e.g., $t\to p_1\to p_2\to\cdots p_m\to y$, similarly we have
\[
(\eb_{t}-\eb_{p_1})^\top \W\eb_k,~(\eb_{p_1}-\eb_{p_2})^\top \W\eb_k,~\cdots,(\eb_{p_m}-\eb_{y})^\top \W\eb_k\geq0\Longrightarrow(\eb_{t}-\eb_{y})^\top \W\eb_k\geq0.
\]
Combining results in that $(\eb_{t}-\eb_{y})^\top \W\eb_k=0$. This implies that for all $(i\asymp j)\in\Gck$ and $R\to\infty$, to ensure the training loss $\Lcb(R\cdot\W)$ finite, $(\eb_{i}-\eb_{j})^\top \W\eb_k=0$, which contradicts the facts that $\W\in\Scf$ and $\W\neq0$.

Next, we prove that there is at most one local minimum for $\Lcb(\W)$ based on Lemma \ref{lemma cvx}. Suppose to the contrary that we have two optimal solutions satisfying $\min_{\W} \Lcb(\W) = \Lcb(\Wf_1) = \Lcb(\Wf_2), \Wf_1 \neq \Wf_2$. From Lemma \ref{lemma cvx}, since $\Lc(\W)$ is strictly convex over subspace $\Scf$, for any $\W_1, \W_2 \in \Scf, \lambda \in (0,1)$, if $\W_1 \neq \W_2$, we have
\begin{equation}
  \Lcb((1-\lambda)\W_1 + \lambda\W_2) < (1 - \lambda) \Lcb(\W_1) + \lambda \Lcb(\W_2)  
\end{equation}
Substitute $\W_1 = \Wf_1, \W_2 = \Wf_2$, we get
\begin{equation}
\Lcb((1-\lambda)\Wf_1 + \lambda \Wf_2) < (1 - \lambda) \Lcb(\Wf_1) + \lambda \Lcb(\Wf_2) = \min_{\W} \Lcb(\W)
\end{equation}
which leads to a contradiction to the assumption that $\Wf_1 \text{ and } \Wf_2$ are both optimal solutions. Combining this with the fact that $\Wf$ is not attained at infinity, there exists one unique and finite solution with $\bWf = \arg\min_{\W\in\Scf}\Lcb(\W)$.
\end{proof}

\subsection{Proof of Theorem~\ref{thm cyclic gd}}
\begin{lemma}\label{lemma smaller loss}
    Consider the same setting of Theorem~\ref{thm cyclic gd}. 
    Given any $\pi>0$, there exists $R_\pi>0$ such that for any $\W$ with $\tf{\prj_{\Scf}(\W)}<\infty$ and $\tf{\prj_{\Scf^\perp}(\W)}>R_\pi$, 
    \[
    \Lc(\W)\geq\Lc\left((1+\pi)\frac{\tf{\prj_{\Scf^\perp}(\W)}}{\tf{\Wm}}\Wm+\prj_{\Scf}(\W)\right).
    \]
\end{lemma}
\begin{proof} 
Recap $\Oc_i,\Ocb_i,\Rc_i,\Rcb_i$, $i\in[n]$ from \eqref{def Oc Rc} and recap from \eqref{def erm loss}, we get that for any $\W$, 
\[
\Lc(\W)=-\frac{1}{n}\sum_{i=1}^n\log\left(\sum_{t\in\Oc_i}s_{it}\right)
\]
where $\s_i=\sft{\X_i\W\xli}$. 

To obtain the result, we establish a refined softmax probability control by studying the distance to $\Lcb(\W)$ as defined in Definition~\ref{def finite correct}. Let $\W^\pl=\prj_{\Scf}(\W)$, $\W^\perp=\W-\W^\pl$, $\tf{\W^\perp}=R$, and $\Theta=1/\tf{\Wm}$. Let $\bb_i=\X_i\W^{\pl}\xli$, $\ab_i^\st=\X_i((1+\pi)R\Theta\cdot\Wm)\xli$, and  $\s_i^\st=\sft{\X_i((1+\pi)R\Theta\cdot\Wm+\W^\pl)\xli}=\sft{\ab_i^\st+\bb_i}$. Additionally, let $\ab_i=\X_i\W^\perp\xli$, $\s_i=\sft{\X_i\W\xli}=\sft{\ab_i+\bb_i}$, $\gamma_i^\st:=\cb_{y_i}^\top\X_i^\top\s_i^\st$, and $\gamma_i:=\cb_{y_i}^\top\X_i^\top\s_i$. 

From proof of Lemma~\ref{lemma Lcb eq} (more specifically \eqref{ortho W}), we get for all $t,t'\in\Rc_i$
\[
(\x_{it}-\x_{it'})^\top\V\xli=0\quad\text{for any }\V\perp\Scf\Longrightarrow a^\st_{it}-a^\st_{it'}=a_{it}-a_{it'}=0.
\]
Additionally, since $\frac{\W}{\tf{\W}}\neq\Theta\Wm$, there exist $i\in[n],t\in\Oc_i,t'\in\Rcb_i$ such that $(\x_{it}-\x_{it'})^\top\W\xli<R\Theta$. Then,
\begin{align*}
    \sum_{t\in\Oc_i}s_{it}&=\frac{\sum_{t\in\Oc_i}e^{a_{it}+b_{it}}}{\sum_{t\in[T_i]}e^{a_{it}+b_{it}}}\leq\frac{\sum_{t\in\Oc_i}e^{b_{it}}}{\sum_{t\in\Rc_i}e^{b_{it}}+e^{-R\Theta+b_{it'}}}\leq\frac{c_i}{d_i+e^{-R\Theta-\bar b}},\quad\exists i\in[n]\\    \sum_{t\in\Oc_i}s_{it}^\st&=\frac{\sum_{t\in\Oc_i}e^{a^\st_{it}+b_{it}}}{\sum_{t\in[T_i]}e^{a^\st_{it}+b_{it}}}\geq\frac{\sum_{t\in\Oc_i}e^{b_{it}}}{\sum_{t\in\Rc_i}e^{b_{it}}+\sum_{t\in\Rcb_i}e^{-(1+\pi)R\Theta+b_{it}}}\geq\frac{c_i}{d_i+Te^{-(1+\pi)R\Theta+\bar b}},\quad\forall i\in[n],
\end{align*}
where $c_{i}=\sum_{t\in\Oc_i}e^{b_{it}}$, $d_i=\sum_{t\in\Rc_i}e^{b_{it}}$, and $\bar b:=\max_{t\in\Rcb_i,i\in[n]}|b_{it}|$, and we have $$\Lcb(\W)=\Lcb(\W^\pl)=-\frac{1}{n}\sum_{i\in\Ic}\log\left(\frac{c_i}{d_i}\right).$$ 
Then we obtain
\begin{align*}
    \Lc(\W)-\Lcb(\W)&\geq-\frac{1}{n}\left(\log\left(\frac{c_i}{d_i+e^{-R\Theta-\bar b}}\right)-\log\left(\frac{c_i}{d_i}\right)\right)\\
    &\geq\frac{1}{n}\log\left(1+e^{-R\Theta-\bar b}/d_i\right)
\end{align*}
and let $j:=\arg\max_{i\in[n]}\left(-\log\left(\sum_{t\in\Oc_i}s_{it}^\st\right)+\log\left(\frac{c_i}{d_i}\right)\right)$. We can upper-bound the loss difference for $(1+\pi)R\Theta\cdot\Wm+\W^\pl$ as follows:
\begin{align*}
    \Lc((1+\pi)R\Theta\cdot\Wm+\W^\pl)-\Lcb(\W)&\leq\max_{i\in[n]}\left(-\log\left(\sum_{t\in\Oc_i}s_{it}^\st\right)+\log\left(\frac{c_i}{d_i}\right)\right)\\
    &=\log\left(1+Te^{-(1+\pi)R\Theta+\bar b}/d_j\right).
\end{align*}
Combining them together results in that, $\Lc(\W)>\Lc((1+\pi)R\Theta\cdot\Wm+\W^\pl)$ whenever
\begin{align}
\frac{1}{n}\log\left(1+e^{-R\Theta-\bar b}/d_i\right)\geq\log\left(1+Te^{-(1+\pi)R\Theta+\bar b}/d_j\right).\nn
\end{align}
Given that $x/2\leq\log(1+x)$ for any $0\leq x\leq1$, we get $\Lc(\W)>\Lc((1+\pi)R\Theta\cdot\Wm+\W^\pl)$ whenever
\begin{align}
&\frac{e^{-R\Theta-\bar b}}{2nd_i}\geq\frac{Te^{-(1+\pi)R\Theta+\bar b}}{d_j}\quad\text{when}\quad R\geq-\frac{\log(d_j)+\bar b}{\Theta}\nn\\
\Longrightarrow& R>R_\pi:=\max\left\{\frac{1}{\pi\Theta}\log\left(\frac{2nTd_i}{d_j}\right)+\frac{2\bar b}{\pi\Theta},-\frac{\log(d_j)+\bar b}{\Theta}\right\}.\label{R pi}
\end{align}
Here, since $\tf{\W^\pl}<\infty$, $d_i,d_j,\bar b<\infty$. 
\end{proof}

\begin{proofof}{Theorem~\ref{thm cyclic gd}}
Now gathering all the results so far, we are ready to prove the gradient descent convergence. The divergence of $\tf{\W(\tau)}$ as $\tau\to\infty$ has been proven by Theorem~\ref{thm diverge}.

\smallskip
\noindent$\bullet$ \textbf{We first show that $\prj_{\Scf}(\W(\tau))\to\Wf$.} Lemma~\ref{lemma Wf eq} has established the equivalence between $\bWf$ and $\Wf$. Since $\Lc(\W)$ is convex following Lemma~\ref{lemma cvx}, we have that $\Lc(\W(\tau))\to\Lc_\st:=\min_{\W}\Lc(\W)$. Additionally, Lemma~\ref{lemma cvx} shows that $\Lc(\W)$ is strictly convex on $\Scf$ and Lemma~\ref{lemma finite wfin} shows that $\Wf$ is the unique finite solution. Suppose $\prj_{\Scf}(\W(\tau))\not\to\Wf$. Let $\W$ be any matrix with $\prj_{\Scf}(\W)\neq\Wf$. Then Lemma \ref{lemma loss eq} and Lemma~\ref{lemma Lcb eq} give that $\Lc(\W)\geq \Lcb(\W)=\Lcb(\prj_{\Scf}(\W))>\Lc_\st$ where $\Lc_\st=\min_{\W}\Lc(\W)=\min_{\W}\Lcb(\W)$. Given that $\Lc_\st$ is achievable, the proof is done by contradiction and we have that $\prj_{\Scf}(\W(\tau))\to\Wf$.

\smallskip
\noindent$\bullet$ \textbf{We next prove that when $\Wm=0$, $\prj_{\Scf^\perp}(\W(\tau))=\prj_{\Scf^\perp}(\W(0))$.} Recap the gradient in \eqref{gradient} where
\[
\nabla\Lc(\W)=\frac{1}{n}\sum_{i=1}^n\sum_{t\in \Ocb_i}s_{it}(\x_{it}-\eb_{y_i})\xli^\top.
\]
Since $\Wm=0$ implies that for all $i\in[n]$, $\Rcb_i=\emptyset$, then $\Ocb_i\subseteq\Rc_i$. Additionally, since following definition of $\Scf$ from Def.~\ref{def cyc subspace}, for any $t\in\Rc_i$, $(\x_{it}-\eb_{y_i})\xli^\top\in\Scf$. Then we obtain
\[
\prj_{\Scf}(\nabla\Lc(\W))=\frac{1}{n}\sum_{i=1}^n\sum_{t\in \Ocb_i}s_{it}\prj_{\Scf}\left((\x_{it}-\eb_{y_i})\xli^\top\right)=\frac{1}{n}\sum_{i=1}^n\sum_{t\in \Ocb_i}s_{it}(\x_{it}-\eb_{y_i})\xli^\top=\nabla\Lc(\W).
\]
Therefore, $\prj_{\Scf^\perp}(\nabla\Lc(\W))=0$ for any $\W$, which completes the proof.

\smallskip
\noindent$\bullet$ \textbf{Last, we show that when $\Wm\neq0$, $\W(\tau)/\tf{\W(\tau)}\to\Wm/\tf{\Wm}$.} 




Consider any $\W\in\R^{d\times d}$, and let $\W^\perp=\prj_{\Scf^\perp}(\W)$, $\W^\pl=\prj_{\Scf}(\W)$, $R=\tf{\W^\perp}$, and $\Theta=1/\tf{\Wm}$. 
Next, from Lemma~\ref{lemma des}, for any $\tau\geq0$, $\Lc(\W(\tau+1))\leq\Lc(\W(\tau))$. 
Let $\W^\perp(\tau)=\prj_{\Scf^\perp}(\W(\tau))$ and $\W^\pl(\tau)=\prj_{\Scf}(\W(\tau))$. Following Lemma~\ref{lemma loss eq}, since loss satisfies $\ell(1)=-\log(1)=0$, we obtain 
\begin{align*}
\Lc(\W(\tau))\geq\Lcb(\W^\pl(\tau)).
\end{align*}
Since following Lemma~\ref{lemma finite wfin}, the training risk $\Lcb(\W^\pl(\tau))$ is infinite if $\tf{\W^\pl(\tau)}\to\infty$, which implies $\tf{\W^\pl(\tau)}<\infty$ for any $\tau\geq0$. 
Additionally, Theorem~\ref{thm diverge} proves the divergence of $\W(\tau)$ as $\tau\to\infty$, hence we have $\tf{\W^\perp(\tau)}\to\infty$.

Applying Lemma~\ref{lemma smaller loss}, as well as the fact that $\tf{\W^\pl(\tau)}<\infty$ and $\tf{\W^\perp(\tau)}\to\infty$, there exists sufficiently large $R_\pi$ as defined in \eqref{R pi} such that once $\tf{\W^\perp(\tau)}=R> R_\pi$, $\Lc(\W(\tau))-\Lc((1+\pi)R\Theta\cdot\Wm+\W^\pl(\tau))\ge 0$. Since $\Lc(\W)$ is convex following Lemma~\ref{lemma cvx}, we have that 
\begin{equation}\label{prj cvx ineq}
    \begin{split}
        \Lc(\W)&\leq\Lc((1+\pi)R\Theta\cdot\Wm+\W^\pl)+\li\nabla\Lc(\W),\W-\left((1+\pi)R\Theta\cdot\Wm+\W^\pl\right)\ri\\
        &=\Lc((1+\pi)R\Theta\cdot\Wm+\W^\pl)+\li\nabla\Lc(\W),\left(\W^\perp -(1+\pi)R\Theta\cdot\Wm\right)\ri\\
    &=\Lc((1+\pi)R\Theta\cdot\Wm+\W^\pl)+\li\prj_{\Scf^\perp}(\nabla\Lc(\W)),\left(\W^\perp-(1+\pi)R\Theta\cdot\Wm\right)\ri.
    \end{split}
\end{equation}
Here, the first inequality uses the convexity of $\Lc(\W)$ and last equation is obtained from the fact that $\Wm\perp\Scf$. It implies that once $\tf{\W^\perp(\tau)}=R> R_\pi$, 
\[
\li\prj_{\Scf^\perp}(\nabla\Lc(\W)),\left(\W^\perp-(1+\pi)R\Theta\cdot\Wm\right)\ri\geq0.
\]
Now we choose $\tau_0$ such that for all $\tau\geq\tau_0$, $\tf{\W^\perp(\tau)}>R_\pi$. Then for $\tau>\tau_0$, we get
\begin{align}
    \li\W^\perp(\tau+1)-\W^\perp(\tau), \frac{\Wm}{\tf{\Wm}}\ri
    \geq\frac{1}{1+\pi}\li\W^\perp(\tau+1)-\W^\perp(\tau),\frac{\W^\perp(\tau)}{\tf{\W^\perp(\tau)}}\ri\label{gd Wm ieq1}
\end{align}
where 
\begin{align}
&\li\W^\perp(\tau+1)-\W^\perp(\tau),\frac{\W^\perp(\tau)}{\tf{\W^\perp(\tau)}}\ri\nn\\
&=\frac{1}{2\tf{\W^\perp(\tau)}}\left(\tf{\W^\perp(\tau+1)}^2-\tf{\W^\perp(\tau)}^2-\tf{\W^\perp(\tau+1)-\W^\perp(\tau)}^2\right)\nn\\
    &\geq\frac{\tf{\W^\perp(\tau+1)}^2-\tf{\W^\perp(\tau)}^2}{2\tf{\W^\perp(\tau)}}-\tf{\W^\perp(\tau+1)-\W^\perp(\tau)}^2\label{gd Wm ieq2}\\
    &\geq\tf{\W^\perp(\tau+1)}-\tf{\W^\perp(\tau)}-\tf{\W^\perp(\tau+1)-\W^\perp(\tau)}^2\label{gd Wm ieq3}\\
    &\geq\tf{\W^\perp(\tau+1)}-\tf{\W^\perp(\tau)}-\tf{\W(\tau+1)-\W(\tau)}^2\label{gd Wm ieq4}\\
    &\geq\tf{\W^\perp(\tau+1)}-\tf{\W^\perp(\tau)}+2\eta\left(\Lc(\W(\tau+1))-\Lc(\W(\tau))\right).\label{gd Wm ieq5}
\end{align}
Here, \eqref{gd Wm ieq1} is obtained from \eqref{prj cvx ineq} and holds for all $\tau>\tau_0$; 
\eqref{gd Wm ieq2} comes from the fact that $\tf{\W^\perp(\tau)}>0.5$; \eqref{gd Wm ieq3} follows that for any $a,b>0$, $(a^2-b^2)/2b>a-b$; \eqref{gd Wm ieq4} follows the projection property that $\tf{\W(\tau+1)-\W(\tau)}^2=\tf{\W^\perp(\tau+1)-\W^\perp(\tau)}^2+\tf{\W^\pl(\tau+1)-\W^\pl(\tau)}^2$; and  \eqref{gd Wm ieq5} is obtained via Lemma~\ref{lemma des}.

Summing the above inequality over $\tau\geq\tau_0$ obtains 
\begin{align*}
    \li\W^\perp(\tau)-\W^\perp(\tau_0), \frac{\Wm}{\tf{\Wm}}\ri\geq\frac{1}{1+\pi}\left(\tf{\W^\perp(\tau)}-\tf{\W^\perp(\tau_0)}+2\eta\left(\Lc(\W(\tau))-\Lc(\W(\tau_0))\right)\right)
\end{align*}
\begin{align*}
    \Longrightarrow\li\frac{\W^\perp(\tau)}{\tf{\W^\perp(\tau)}},\frac{\Wm}{\tf{\Wm}}\ri\geq\frac{1}{1+\pi}\left(1+\frac{C}{\tf{\W^\perp(\tau)}}\right)
\end{align*}
where $$C:=\li\W^\perp(\tau_0),\frac{\Wm}{\tf{\Wm}}\ri-\tf{\W^\perp(\tau_0)}+2\eta\left(\Lc(\W(\tau))-\Lc(\W(\tau_0))\right).$$
Since $\tf{\W^\perp(\tau)}\to\infty$ and $0<\Lc(\W(\tau))\leq\Lc(\W(0))<\infty$, we get
\begin{align}
\lim_{\tau\to\infty}\li\frac{\W^\perp(\tau)}{\tf{\W^\perp(\tau)}},\frac{\Wm}{\tf{\Wm}}\ri\geq\frac{1}{1+\pi}.\label{gd Wm 1 corr}
\end{align}
Choosing $\pi\to0$ and combining \eqref{gd Wm 1 corr} with the fact that $\lim_{\tau\to\infty}\tf{\W^\pl(\tau)}<\infty$ completes the proof.

\end{proofof}

%% file: supp/reg_path_proof.tex
\ifarxiv
\section{Global Convergence of Regularization Path}
\else
\section{GLOBAL CONVERGENCE OF REGULARIZATION PATH}
\fi

\subsection{Proof of Theorem~\ref{thm general bias}}

\begin{lemma}\label{lemma reg orth dir} Suppose Assumptions~\ref{assume iden} and \ref{assume realizable} hold. Additionally, assume loss $\ell:\R\rightarrow\R$ is strictly decreasing and $|\ell'|$ is bounded.  Define $\bar\W_R^\perp:=\bar\W_R^\perp(\W^\pl)\in\Scf^\perp$ by
\begin{align}
\bar\W_R^\perp:=\arg\min_{\W\in\Scf^\perp,\tf{\W}\leq R}\Lc(\W+\W^\pl).\label{any wfin subj}
\end{align}
Let $\Wm\neq0$ denote the solution of \eqref{graph svm}. Then we have that for any $\W^\pl\in\Scf$ with $\tf{\W^\pl}<\infty$
\[
\lim_{R\to\infty}\frac{\bar\W_R^\perp}{R}=\frac{\Wm}{\tf{\Wm}}.
\]
\end{lemma}
\begin{proof} 
%
%
%
%
Recap $\Oc_i,\Ocb_i,\Rc_i,\Rcb_i$, $i\in[n]$ from \eqref{def Oc Rc}. Let $\bgam_i=\X_i\cb_{y_i}$ where following Assumption~\ref{assume iden} we have 
\[
\gamma_{it}=\x_{it}^\top\cb_{y_i}=\begin{cases}
    1,\quad t\in\Oc_i\\
    0,\quad t\in\Ocb_i.
\end{cases}
\]
Recap $\Ic,\Icb$ from \eqref{def Ic}. To proceed, define $\bgammax_i$ as follows:
\begin{itemize}
    \item Consider $i\in\Icb$. Given $\min_{\W}\ell(\cb_{y_i}^\top\X_i^\top\sft{\X_i\W\xli})\geq\ell(\cb_{y_i}^\top\eb_{y_i})$, we define the maximal score $\bgammax_i:=\cb_{y_i}^\top\eb_{y_i}=1$.
    \item Consider $i\in\Ic$.  
    \begin{enumerate}
        \item Assumption~\ref{assume iden} ensures that all tokens, excluding the ones with token ID $x_{it}=y_i$,  return zero score, that is, $\cb_{y_i}^\top\eb_{k}=0$ for $k\neq y_i$.
        \item From proof of Lemma~\ref{lemma ortho}, for any $\W^\perp\in\Scf^\perp$ and $t\in\Rc_i$, $\x_{it}^\top(\W^\perp+\W^\pl)\xl_i=\x_{it}^\top\W^\pl\xl_i+\bar a_i$, where $\bar a_i$ is some constant associated with $\W^\perp$ and remains the same value within the same SCC. Let $\bb_i=\X_i\W^\pl\xli$. Then the probabilities for $t\in\Rc_i$ (if denoted by $s_{it}$) obey
        \begin{align}
        \frac{s_{it}}{\sum_{t'\in\Rc_i}s_{it'}}=\frac{e^{b_{it}+\bar a_i}}{\sum_{t'\in\Rc_i}e^{b_{it'}+\bar a_i}}=\frac{e^{b_{it}}}{\sum_{t'\in\Rc_i}e^{b_{it'}}},\label{score over Rc}
        \end{align}
        which means that 
        the probability distribution over set $\Rc_i$ 
        remains the same with varying $\W^\perp$.
    \end{enumerate}
    Combining both, we define the maximal score as follows:
    \[
    \bgammax_i:=|\Oc_i|\cdot\bar s_i\quad\text{where}\quad \bar s_i=\frac{e^{\eb_{y_i}^\top\W^\pl\xli}}{\sum_{t'\in\Rc_i}e^{b_{it'}}}.
    \]
    Note that if consider the cyclic subdataset $\bdata$ as in \eqref{def bdata}. Let $(\Xb_i,y_i)\in\bdata$ where $\Xb_i$ is the corresponding sequence by removing the tokens in $\Rcb_i$. Then we have $\bgammax_i=\cb_{y_i}^\top\bar\X_i^\top\sft{\bar\X_i\W^\pl\xli}$. 
\end{itemize}
Hence, given $\bb_i=\X_i\W^\pl\xli$, we obtain
\[
\bgammax_i=\begin{cases}
    1,& i\in\Icb\\
    \frac{|\Oc_i|e^{\eb_{y_i}^\top\W^\pl\xli}}{\sum_{t'\in\Rc_i}e^{b_{it'}}},& i\in\Ic.
\end{cases}
\]
Then we define the optimal risk of \eqref{any wfin subj} and its corresponding softmax probabilities $\smx_i$, $i\in[n]$ as follows: 
\[
\Lc_\st^{\W^\pl}:=\frac{1}{n}\sum_{i=1}^n\ell(\bgammax_i),\quad \text{and}\quad s^{\max}_{it}=\begin{cases}
    0,& t\in\Rcb_i\\
    \frac{e^{b_{it}}}{\sum_{t'\in\Rc_i}e^{b_{it'}}},&t\in\Rc_i
\end{cases} \quad\text{for all }i\in[n].
\]
Note that we also have 
\begin{align}\bgammax_i=\cb_{y_i}^\top\X_i^\top\smx_i=\sum_{t\in\Oc_i}s^{\max}_{it}=\frac{\sum_{t\in\Oc_i}e^{b_{it}}}{\sum_{t\in\Rc_i}e^{b_{it}}}\quad\text{and}\quad \Lc_\st^{\W^\pl}=\frac{1}{n}\sum_{i\in\Icb}\ell(1)+\Lcb(\W^\pl)\label{Lc opt}
\end{align}
where $\Lcb(\W^\pl)$ is the empirical risk over cyclic subdataset defined in Definition~\ref{def finite correct}.

In the following, we will complete the proof in three steps.

\noindent\textbf{Step 1: } We first show that $\lim_{R\to\infty}\Lc(R\cdot\Wm+\W^\pl)=\Lc_\st^{\W^\pl}$. 
It can be easily proven using Lemma~\ref{lemma loss eq} and \eqref{Lc opt} by showing that for any $\W^\pl$ with $\tf{\W^\pl}<\infty$, 
\[
\lim_{R\to\infty}\Lc(R\cdot\Wm+\W^\pl)=\min_{\W\in\Scf^\perp}\Lc(\W+\W^\pl)=\frac{|\Icb|}{n}\ell(1)+\Lcb(\W^\pl)=\Lc_\st^{\W^\pl}.
\]

\noindent\textbf{Step 2: }Next, we will prove that for any $\W^\pl\in\Scf$ with $\tf{\W^\pl}<\infty$, $\Wb_R^\perp$ achieves the optimal risk as $R\to\infty$ -- rather than problem having finite optima. It is to show that there is no finite $R$ can achieve optimal risk. Consider any $\W\in\Scf^\perp$ with $\tf{\W}<\infty$. Let $\s_i:=\sft{\X_i(\W+\W^\pl)\xli}$ and $\gamma_i^{\W}:=\cb_{y_i}^\top\X_i^\top\s_i$. Then we have that
\begin{align*}
\gamma_i^{\W}=\sum_{t\in\Oc_i}s_{it}=\sum_{t\in\Oc_i}\frac{\sum_{t'\in\Rc_i}s_{it'}}{\sum_{t'\in[T_i]}s_{it'}}s^{\max}_{it}\leq\sum_{t\in\Oc_i}s^{\max}_{it}=\bgammax_i
\end{align*}
where the equality holds when $\Rc_i=[T_i]$. Since $\Wm\neq0$, then there exists some $i\in[n]$ such that $\gamma_i^{\W}<\bgammax_i$. 
Therefore, for any finite $\W\in\R^{d\times d}$ and $\W\in\Scf^\perp$, since loss function is strictly decreasing
\[
\Lc(\W+\W^\pl)=\frac{1}{n}\sum_{i=1}^n\ell(\gamma_i^{\W}) > \frac{1}{n}\sum_{i=1}^n\ell(\bgammax)=\Lc_\st^{\W^\pl}.
\]

\noindent\textbf{Step 3:} Now, it remains to show that $\Wb_R^\perp$ converges in direction to $\Wm$. Suppose convergence fails. We will obtain a contradiction by showing that $R\cdot\Wm/\tf{\Wm}$ achieves a strictly superior loss compared to $\bar\W_R^\perp$ given sufficiently large $R$. Since $\bar\W_R^\perp$ fails to converge to $\Wm$, for some $\delta>0$, there exists arbitrarily large $R>0$ such that 
\[
\tf{\bar\W_R^\perp\cdot\tf{\Wm}/R-\Wm}\geq \delta.
\]
Let $\W'=\bar\W_R^\perp\cdot\tf{\Wm}/R$ where we have $\tf{\W'}\leq\tf{\Wm}$ and $\W'\neq\Wm$. 
Following Definition~\ref{def cyc subspace}, we obtain
\[
(\eb_i-\eb_j)^\top\W'\eb_k= 0\quad\text{where}\quad(i\asymp j)\in\Gck.
\]
Then for some $\epsilon:=\epsilon(\delta)$, there exists $i,j,k$ such that 
\[
(\eb_i-\eb_j)^\top\W'\eb_k\leq 1-\epsilon\quad\text{where}\quad(i\Rightarrow j)\in\Gck.
\]

Now, we will argue that this leads to a contradiction by proving $\Lc(R\cdot\Wm/\tf{\Wm}+\W^\pl)<\Lc(\Wb_R+\W^\pl)$ for sufficiently large $R$. Let $\Theta=1/\tf{\Wm}$ and we will show that $\Lc(R\Theta\cdot\Wm+\W^\pl)<\Lc(R\Theta\cdot\W'+\W^\pl)$ for sufficiently large $R$.

To obtain the result, we establish a refined softmax probability control by studying the distance to $\Lc_\st^{\W^\pl}$. Recap the definitions of $\bgammax_i$ and $\smx_i$, and let $\bb_i=\X_i\W^{\pl}\xli$, $\ab_i^\st=\X_i(R\Theta\cdot\Wm)\xli$,  $\s_i^\st=\sft{\X_i(R\Theta\cdot\Wm+\W^\pl)\xli}=\sft{\ab_i^\st+\bb_i}$, and $\gamma_i^\st:=\cb_{y_i}^\top\X_i^\top\s_i^\st$. Additionally, let $\ab_i^R=\X_i(R\Theta\cdot\W')\xli$, $\s_i^R=\sft{\X_i(R\Theta\cdot\W'+\W^\pl)\xli}=\sft{\ab_i^R+\bb_i}$, and $\gamma_i^R:=\cb_{y_i}^\top\X_i^\top\s_i^R$. 

Following Definition~\ref{def cyc subspace}, we get for all $t,t'\in\Rc_i$
\[
(\x_{it}-\x_{it'})^\top\W\xli=0\quad\text{for any }\W\perp\Scf\Longrightarrow a^\st_{it}-a^\st_{it'}=a^R_{it}-a^R_{it'}=0.
\]
%
Then
\begin{align*}
    \sum_{t\in\Oc_i}s_{it}^R&=\frac{\sum_{t\in\Oc_i}e^{a^R_{it}+b_{it}}}{\sum_{t\in[T_i]}e^{a^R_{it}+b_{it}}}\leq\frac{\sum_{t\in\Oc_i}e^{b_{it}}}{\sum_{t\in\Rc_i}e^{b_{it}}+e^{-(1-\epsilon)R\Theta-\bar b}}\leq\frac{c_i}{d_i+e^{-(1-\epsilon)R\Theta-\bar b}},\quad\exists i\in[n]\\
    \sum_{t\in\Oc_i}s_{it}^\st&=\frac{\sum_{t\in\Oc_i}e^{a^\st_{it}+b_{it}}}{\sum_{t\in[T_i]}e^{a^\st_{it}+b_{it}}}\geq\frac{\sum_{t\in\Oc_i}e^{b_{it}}}{\sum_{t\in\Rc_i}e^{b_{it}}+\sum_{t\in\Rcb_i}e^{-R\Theta+b_{it}}}\geq\frac{c_i}{d_i+Te^{-R\Theta+\bar b}},\quad\forall i\in[n],
\end{align*}
where $c_{i}=\sum_{t\in\Oc_i}e^{b_{it}}$, $d_i=\sum_{t\in\Rc_i}e^{b_{it}}$, and $\bar b:=\max_{t\in\Rc_i,i\in[n]}|b_{it}|$, and we have $\bgammax_i=c_i/d_i$.

Since $\ell$ is strictly decreasing and $|\ell'|$ is bounded, let $\cl\leq-\ell'\leq\ch$ for some constants $\cl,\ch>0$. Note that $\cl,\ch$ are data-dependent. 
Then we have
\begin{align*}
    \Lc(R\Theta\cdot\W'+\W^\pl)-\Lc_\st^{\W^\pl}&\geq\frac{1}{n}\left(\ell(\gamma_i^R)-\ell(\bgammax_i)\right)\geq\frac{\cl}{n}\left(\bgammax_i-\gamma_i^R\right)\\
    &=\frac{\cl}{n}\left(\bgammax_i-\sum_{t\in\Oc_i}s_{it}^R\right)\\
    &\geq\frac{\cl \gamma_i^{\max}}{n}\left(1-\frac{1}{1+e^{-(1-\epsilon)R\Theta-\bar b}/d_i}\right)
\end{align*}
and let $j:=\max_{i\in[n]}\left(\ell(\gamma_i^\st)-\ell(\bgammax_i)\right)$. We can upper-bound the loss difference for $R\Theta\cdot\Wm+\W^\pl$ as follows:
\begin{align*}
    \Lc(R\Theta\cdot\Wm+\W^\pl)-\Lc_\st^{\W^\pl}&\leq\max_{i\in[n]}\left(\ell(\gamma_i^\st)-\ell(\bgammax_i)\right)\leq\ch\left(\bgammax_j-\gamma_j^\st\right)\\
    &=\ch\left(\bgammax_j-\sum_{t\in\Oc_j}s_{it}^\st\right)\\
    &\leq\ch\gamma_j^{\max}\left(1-\frac{1}{1+Te^{-R\Theta+\bar b}/d_j}\right)\\
    &\leq\frac{\ch\gamma_j^{\max}T}{d_j}e^{-R\Theta+\bar b}.
\end{align*}
Combining them together results in that, $\Lc(R\Theta\cdot\W'+\W^\pl)>\Lc(R\Theta\cdot\Wm+\W^\pl)$ whenever
\begin{align}
&\frac{\cl \gamma_i^{\max}}{n}\left(1-\frac{1}{1+e^{-(1-\epsilon)R\Theta-\bar b}/d_i}\right)>\frac{\ch\gamma_j^{\max}T}{d_j}e^{-R\Theta+\bar b}\nn\\
&\Longrightarrow R>R_\epsilon:=\frac{1}{\Theta\cdot\min(\epsilon,1)}\log\left(\frac{2nT\ch \gamma_j^{\max}\cdot\max(d_i,1)}{\cl \gamma_i^{\max}d_j}\right)+\frac{2\bar b}{\Theta\cdot\min(\epsilon,1)}.\label{cyc reg R bound}
\end{align}
Note that since $\W^\pl$ is finite, $b_{it}$ for all $i\in[n],t\in[T_i]$ are bounded and fixed, and therefore, $0< d_i<\infty$, for all $i\in[n]$ and $\bar b<\infty$. \eqref{cyc reg R bound} completes the proof by contradiction. 
\end{proof}


Now, gathering all the results we have obtained so far, we are ready to prove Theorem~\ref{thm general bias}. 

\begin{proofof}{Theorem~\ref{thm general bias}}
Recap the dataset $\data=(\X_i,y_i)_{i=1}^n$ and index sets $\Oc_i,\Ocb_i,\Rc_i,\Rcb_i$, $i\in[n]$ from \eqref{def Oc Rc}. Let $\bgam_i=\X_i\cb_{y_i}$ denote the score vector of $i$-th input. Since Assumption~\ref{assume iden} holds, then 
\[
\gamma_{it}=\x_{it}^\top\cb_{y_i}=\begin{cases}
    1,\quad t\in\Oc_i\\
    0,\quad t\in\Ocb_i.
\end{cases}
\]
Let $\s^{\W}_i=\sft{\X_i\W\xli}$. The regularization path solution of the ERM problem is defined as follows:  
\[
\bar{\W}_R=\arg\min_{\tf{\W}\leq R}\Lc(\W)\quad\text{where}\quad\Lc(\W)=\frac{1}{n}\sum_{i=1}^n \ell(\cb_{y_i}^\top\X_i^\top \sft{\X_i\W\xli})=\frac{1}{n}\sum_{i=1}^n\ell\left(\sum_{t\in\Oc_i}s^{\W}_{it}\right).
\]
%
%
Additionally, let $\W_R^\perp=\prj_{\Scf^\perp}(\Wb_R)$ and $\W_R^\pl=\prj_{\Scf}(\Wb_R)$. 
Lemma~\ref{lemma reg orth dir} has shown that for any finite $\lim_{R\to\infty}\W_R^\pl$, 
\[
\lim_{R\to\infty}\frac{\Wb_R}{R}=\lim_{R\to\infty}\frac{\W_R^\perp}{\sqrt{R^2-\tf{\W_R^\pl}^2}}=\lim_{R\to\infty}\frac{\W_R^\perp}{\tf{\W_R^\perp}}=\frac{\Wm}{\tf{\Wm}}.
\]
Therefore it remains to prove that $\lim_{R\to\infty}\W_R^\pl\in\Wcf$. 


{Suppose $\lim_{R\to\infty}\W^\pl_R:=\W'\not\in\Wcf$.} Then for any $\W^\pl\in\Wcf$, applying Lemma~\ref{lemma loss eq}, we obtain 
    \[
    \min_{\W^\perp\in\Scf^\perp}\Lc(\W^\perp+\W')>\min_{\W^\perp\in\Scf^\perp}\Lc(\W^\perp+\W^\pl)=\lim_{R\to\infty}\Lc(R\cdot\Wm+\W^\pl).
    \]
Therefore, $\W'$ does not achieve the minimal loss as $R\to\infty$. 
\end{proofof}


\subsection{Proof of Lemma~\ref{lemma opt risk}}

\begin{proof}
%
%
%
Let $\bgam_i=\X_i\cb_{y_i}$ denote the score vector of $i$-th input and $\gamma_{it}=\x_{it}^\top\cb_{y_i}$. Let $\bgammax_i=\eb_{y_i}^\top\cb_{y_i}=\max_{t\in[T_i]}\gamma_{it}$ following Assumptions~\ref{assume realizable} and \ref{assume relax}. What's more, since loss $\ell$ is strictly decreasing, we define the optimal loss as follows:
\[
\Lc_\st:=\frac{1}{n}\sum_{i=1}^n\ell(\bgammax_i).
\]
For any $\W\in\R^{d\times d}$, let $\s_i=\X_i\W\xli$, $i\in[n]$. If $\tf{\W}<\infty$, then $\min_{t\in[T_i],i\in[n]}s_{it}>0$ and for any $i\in[n]$
\[
\s_i^\top\bgam_i=\sum_{t=1}^{T_i}s_{it}\gamma_{it}<\bgammax_i.
\]
Since loss function $\ell$ is strictly decreasing, we get
\[
\Lc(\W)=\frac{1}{n}\sum_{i=1}^n\ell(\s_i^\top\bgam_i)>\frac{1}{n}\sum_{i=1}^n\ell(\bgammax_i)=\Lc_\st.
\]
Let $\W$ be any attention weight satisfying all the ``$\geq1$'' constraints in \eqref{acyc svm}. We next prove that $\lim_{R\to\infty}\Lc(R\cdot\W)=\Lc_\st$. 
Recap $\Oc_i$ and $\Ocb_i$ from \eqref{def Oc Rc}. Since token $\eb_{y_i}$ is always contained in $\X_i$ following Assumption~\ref{assume realizable}, we have $|\Oc_i|\geq1,i\in[n]$, and $\X_i$ contains $|\Oc_i|$ optimal tokens $\eb_{y_i}$. Note that under acyclic data setting, $\W$ separates tokens $\eb_{y_i}$ from the rest of the tokens within $\X_i$. Then $\lim_{R\to\infty}\sft{\X_i(R\cdot\W)\xli}$ will output $1/|\Oc_i|$ for $t\in\Oc_i$ and zero for the left. Specifically, let $\s^R_i:=\sft{\X_i(R\cdot\W)\xli}$, and following the SVM objective \eqref{acyc svm} for any $i\in[n]$, we get
\begin{align*}
&s^R_{it}=\frac{e^{\x_{it}^\top(R\cdot\W)\xli}}{\sum_{t\in[T_i]}e^{\x_{it}^\top(R\cdot\W)\xli}}=\frac{1}{|\Oc_i|+\sum_{t\in\Ocb_i}e^{(\x_{it}-\eb_{y_i})^\top(R\cdot\W)\xli}}\geq\frac{1}{|\Oc_i|+e^{-R}}\quad\text{for all $t\in\Oc_i$}\\
&\text{and then,}\quad\sum_{t\in\Oc_i}s^R_{it}=\frac{|\Oc_i|}{|\Oc_i|+\sum_{t\in\Ocb_i}e^{(\x_{it}-\eb_{y_i})^\top(R\cdot\W)\xli}}\geq\frac{1}{1+e^{-R}}.
\end{align*}
Then $\lim_{R\to\infty}\sum_{t\in\Oc_i}s_{it}^R=1$ and therefore,
\begin{align*}
&\lim_{R\to\infty}s_{it}^R=1/|\Oc_i|\quad\text{for $t\in\Oc_i$},\qquad\text{and}\qquad\lim_{R\to\infty}s_{it}^R=0\quad\text{for $t\in\Ocb_i$}.
\end{align*}
Hence we have
\[
\lim_{R\to\infty}\X_i^\top\s^R_i=\sum_{t\in\Oc_i}\frac{1}{|\Oc_i|}\eb_{y_i}=\eb_{y_i}.
\]
{Since $|\ell'|$ is bounded, then  $\lim_{R\to\infty}\Lc(R\cdot\W)=\frac{1}{n}\sum_{i=1}^n\ell(\cb_{y_i}^\top\eb_{y_i})=\frac{1}{n}\sum_{i=1}^n\ell(\bgammax_i)=\Lc_\st$. }
\end{proof}

\subsection{Proof of Theorem~\ref{thm acyc bias}}

\begin{proof}
Recap the dataset $\data=(\X_i,y_i)_{i=1}^n$. The regularization path solution of the ERM problem (per \ref{algo rp} and \eqref{erm}) is defined as follows:  
\[
\bar{\W}_R=\arg\min_{\tf{\W}\leq R}\Lc(\W)\quad\text{where}\quad\Lc(\W)=\frac{1}{n}\sum_{i=1}^n \ell(\cb_{y_i}^\top\X_i^\top \sft{\X_i\W\xli}).
\]
The proof is similar to the proof of Theorem 2 in \cite{tarzanagh2023transformers} by choosing $\op_i=y_i$. However in our work, we allow each sequence contains more than one optimal tokens, while \cite{tarzanagh2023transformers} forces that the optimal token is unique. 

Following the proof in Lemma~\ref{lemma opt risk}, let $\bgam_i=\X_i\cb_{y_i}$, $\bgammax_i=\eb_{y_i}^\top\cb_{y_i}=\max_{t\in[T_i]}\gamma_{it}$, and the optimal training risk
\[
\Lc_\st:=\frac{1}{n}\sum_{i=1}^n\ell(\bgammax_i).
\]




From Lemma~\ref{lemma opt risk}, we have that for any finite $\W$, $\Lc(\W)<\lim_{R\to\infty}\Lc(R\cdot\Wm)=\Lc^\st$. Then the optimal risk $\Lc_\st$ is achievable and to achieve the limit, $R$ has to be infinite. Then it remains to prove that 
%
$\bar\W_R$ converges in direction to $\Wm$. 

Suppose convergence fails. We will obtain a contradiction by showing that $R\cdot\Wm/\tf{\Wm}$ achieves a strictly superior loss compared to $\bar\W_R$. Suppose $\bar\W_R$ fails to directionally converge towards $\Wm$. For some $\delta>0$, there exists arbitrarily large $R>0$ such that 
\[
\tf{\bar\W_R\cdot\tf{\Wm}/R-\Wm}\geq \delta.
\]
Let $\W'=\bar\W_R\cdot\tf{\Wm}/R$ where we have $\tf{\W'}\leq\tf{\Wm}$ and $\W'\neq\Wm$. Since $\Wm$ is the min-norm solution of \eqref{acyc svm}, then for some $\epsilon:=\epsilon(\delta)$, there exists $i,j,k$ such that 
\[
(\eb_i-\eb_j)^\top\W'\eb_k\leq 1-\epsilon\quad\text{where}\quad(i\Rightarrow j)\in\Gck.
\]
Now, we will argue that this leads to a contradiction by proving $\Lc(R\cdot\Wm/\tf{\Wm})<\Lc(\Wb_R)$ for sufficiently large $R$. Let $\Theta=1/\tf{\Wm}$ and we will show that $\Lc(R\Theta\cdot\Wm)<\Lc(R\Theta\cdot\W')$ for sufficiently large $R$.

To obtain the result, we establish a refined softmax probability control as in the proof of Theorem~\ref{thm general bias} by studying the distance to $\Lc_\st$. 
Let $\ab_i^\st=\X_i(R\Theta\cdot\Wm)\xli$, $\ab_i^R=\X_i(R\Theta\cdot\W')\xli$, $\s_i^\st:=\sft{\ab_i^\st}$, $\s_i^R:=\sft{\ab_i^R}$, $\gamma_i^\st:=\cb_{y_i}^\top\X_i^\top\s_i^\st$, and $\gamma_i^R:=\cb_{y_i}^\top\X_i^\top\s_i^R$. Recap that $\bgammax_i=\eb_{y_i}^\top\cb_{y_i}$. Then 
\begin{align}
    &\sum_{t\in\Oc_i}s_{it}^R=\frac{\sum_{t\in\Oc_i}e^{a^R_{it}}}{\sum_{t\in[T_i]}e^{a^R_{it}}}\leq\frac{|\Oc_i|}{|\Oc_i|+e^{-(1-\epsilon)R\Theta}}\leq\frac{1}{1+e^{-(1-\epsilon)R\Theta}/T},\quad\exists i\in[n]\\
    &\sum_{t\in\Oc_i}s_{it}^\st=\frac{\sum_{t\in\Oc_i}e^{a^\st_{it}}}{\sum_{t\in[T_i]}e^{a^\st_{it}}}\geq\frac{|\Oc_i|}{|\Oc_i|+(T-|\Oc_i|)e^{-R\Theta}}\geq\frac{1}{1+Te^{-R\Theta}},\quad\forall i\in[n].\label{sfx prob wmm}
\end{align}

Since $\ell$ is strictly decreasing and $|\ell'|$ is bounded, 
let $\cl\leq-\ell'\leq\ch$ for some constants $\cl,\ch>0$. Note that $\cl,\ch$ are data-dependent. Additionally, define the score minimal/maximal score gaps as
\[
\cmin=\min_{y,k\in[K],y\neq k}(\eb_{y}-\eb_{k})^\top\cb_{y},\quad\cmax=\max_{y,k\in[K],y\neq k}(\eb_{y}-\eb_{k})^\top\cb_{y}
\]
where $\cmax\geq\cmin>0$. Then we have that there exists $i\in[n]$,
\begin{align}\label{risk diff W}
    \Lc(R\Theta\cdot\W')-\Lc_\st&\geq\frac{1}{n}\left(\ell(\gamma_i^R)-\ell(\bgammax_i)\right)\geq\frac{\cl}{n}\left(\bgammax_i-\gamma_i^R\right)\nn\\
    &\geq\frac{\cl}{n}\cmin\left(1-\sum_{t\in\Oc_i}s_{it}^R\right)\geq\frac{\cl\cmin}{n}\frac{1}{1+Te^{(1-\epsilon)R\Theta}}
\end{align}
and letting $j:=\arg\max_{i\in[n]}\left(\ell(\gamma_i^\st)-\ell(\bgammax_i)\right)$, we can upper-bound the loss difference for $R\Theta\cdot\Wm$ as follows:
\begin{align}\label{risk diff Wmm}
    \Lc(R\Theta\cdot\Wm)-\Lc_\st&\leq\max_{i\in[n]}\left(\ell(\gamma_i^\st)-\ell(\bgammax_i)\right)\leq\ch\left(\bgammax_j-\gamma_j^\st\right)\nn\\
    &\leq\ch\cmax\left(1-\sum_{t\in\Oc_i}s_{it}^\st\right)\leq\ch\cmax\frac{1}{1+e^{R\Theta}/T}\leq\ch\cmax T e^{-R\Theta}.
\end{align}
Combining them together results in that, $\Lc(R\Theta\cdot\W')>\Lc(R\Theta\cdot\Wm)$ whenever 
\[
\frac{\cl\cmin}{n}\frac{1}{1+Te^{(1-\epsilon)R\Theta}}>\ch\cmax T e^{-R\Theta}\Longrightarrow R>\frac{1}{\Theta\cdot\min(\epsilon,1)}\log\left(\frac{2nT^2\ch\cmax}{\cl\cmin}\right).
\]
This completes the proof by contradiction. 
\end{proof}

%% file: supp/add_exp.tex
\ifarxiv
\section{Implementation Details and Additional Experiments}
\else
\section{IMPLEMENTATION DETAILS AND ADDITIONAL EXPERIMENTS}
\fi
\subsection{Implementation Details}

In all the experiments, we train single-layer self-attention layer models using PyTorch and SGD optimizer. We conduct normalized gradient descent method to enhance the increasing of the norm of attention weight, so that softmax can easily saturate. Specifically, at each iteration $\tau$, we update attention weight $\W$ via 
\[
\W(\tau+1)=\W(\tau)-\eta\frac{\nabla\Lc(\W(\tau))}{\tf{\nabla\Lc(\W(\tau))}}.
\]
All the results are averaged over $100$ random trails and in each trail, we create the dataset and its corresponding TPGs, SCCs as follows:
\begin{enumerate}
    \item Given dimension $d$ and vocabulary size $K$, generate random embedding table $\Eb=[\eb_1~\cdots~\eb_K]^\top\in\R^{K\times d}$ such that each $\eb\in\Eb$ is randomly sampled from unit sphere.
    \item Given sample size $n$ and sequence length $T$, create dataset $\data=(\X_i,y_i)_{i=1}^n$ and $\X_i=[\x_{it}~\cdots~\x_{iT}]^\top\in\R^{T\times d}$ where $\x_{it}$ are randomly sampled from $\Eb$. For acyclic setting, label $\eb_{y_i}$ is determined by the token in the $\X_i$ that has the highest priority order; while for general cyclic setting, $\eb_{y_i}$ are also randomly sampled from $\X_i$. 
    
    \item Construct TPGs and apply Tarjan's algorithm \cite{tarjan1972scc} to find SCCs of each TPG. For global convergence experiments (Section~\ref{sec global gd}), TPGs are created based on the token relations between $\x_{it}$'s and $\eb_{y_i}$'s in the dataset $\data$; while for local convergence analysis (Section~\ref{sec:local-gd}), we instead establish the token relations between $\x_{it}$'s and $\hat \eb_{y_i}$'s following the instruction in Section~\ref{sec:local-gd}, where $\hat \eb_{y_i}$ is determined by the GD solution. 
    \item $\bdata$ is created following Definition~\ref{cyc sub} based on the SCCs of the corresponding TPGs.
\end{enumerate}
Here, we set the sequence length to be the same for all the samples in $\data$, and we emphasize that though $\data$ contains inputs with same number of tokens, the randomness in sampling $\x_{it}$ and $\eb_{y_i}$ will still result in a variety of TPGs and SCCs, and $\bdata$ may contain inputs with varying sequence lengths (see Figure~\ref{fig:intro}). 


\smallskip
\noindent$\bullet$ \textbf{Generating $\Wf$ and $\tWf$.} Inspired by the convexity and finiteness of $\Lcb(\W)$ per Definition~\ref{def finite correct} under the setting of Theorem~\ref{thm cyclic gd}, we can derive $\Wf$ via gradient descent. Hence, to obtain $\Wf$, we train separate models but with the same architecture {from zero initialization} on the sub-dataset $\bdata$. As for the experiments shown in Section~\ref{sec:local-gd}, we follow the same method as generating $\Wf$. However, we emphasize that under the local convergence setting, there is no guarantee that gradient descent will converge to the $\tWf$ solution as problem is more general, i.e., with nonconvex head, and dataset $\bdata$ might not be enough to capture the performance of tokens within the same SCCs. Though, our results in Figures~\ref{fig local ls} and \ref{fig local ce} indicate that $\tWf$ can predict the GD convergence performance better than $\Wf$ which is drawn from the dataset-based TPGs. We defer a rigorous definition of local $\tWf$ and guarantees related to gradient descent for future exploration.


\smallskip
\noindent$\bullet$ \textbf{Local convergence experiments (Figures~\ref{fig local ls} and \ref{fig local ce}).} To evaluate our local convergence conjecture, we conduct random experiments with more general head (satisfying Assumption~\ref{assume relax}) and, and consider squared loss $\ell(u)=(1-u)^2$ in Figure~\ref{fig local ls} and cross-entropy loss in Figure~\ref{fig local ce}. In both experiment, we create embedding tabels with $K=8,d=8$ and datasets with $n=4,T=6$. We choose step size $\eta=0.1$ and also conduct normalized gradient descent. Correlations are reported in Figs.~\ref{fig local ls Wm} and \ref{fig local ce Wm} and the distance of $\tf{\prj_{\tScf}(\W(\tau))-\tWf}$ are presented in the orange curves in Figs.~\ref{fig local ls Wf} and \ref{fig local ce Wf}. In both experiments, correlations between $\frac{\W(\tau)}{\tf{\W(\tau)}}\text{ and }\frac{\tWm}{\tf{\tWm}}$ end with \yl{$>0.99$} values. Fig.~\ref{fig local ls Wf} achieves \yl{$0$} distance error since employing squared loss, attention is inclined to select tokens that appear mostly frequently in the labels of the dataset, resulting in $\Rc_i=\Oc_i$ for $i\in[n]$ and $\bdata=\emptyset$. \yl{While in Fig.~\ref{fig local ce Wf}, the global and local norm of difference is around $9.59 \text{ and }0.09$ respectively, where $\bdata\neq\emptyset$. This implies that the distance of $\prj_{\tScf}(\W(\tau))$ is much closer to $\tWf$ compared to the distance between $\prj_{\Scf}(\W(\tau))$ and $\Wf$}.

\smallskip
\begin{figure}[t!] 
\centering
\begin{minipage}{0.45\linewidth}
\begin{tikzpicture}
  \node at (0,0){\includegraphics[height=.6\linewidth, trim={1.3cm 1.3cm 0 0}, clip]{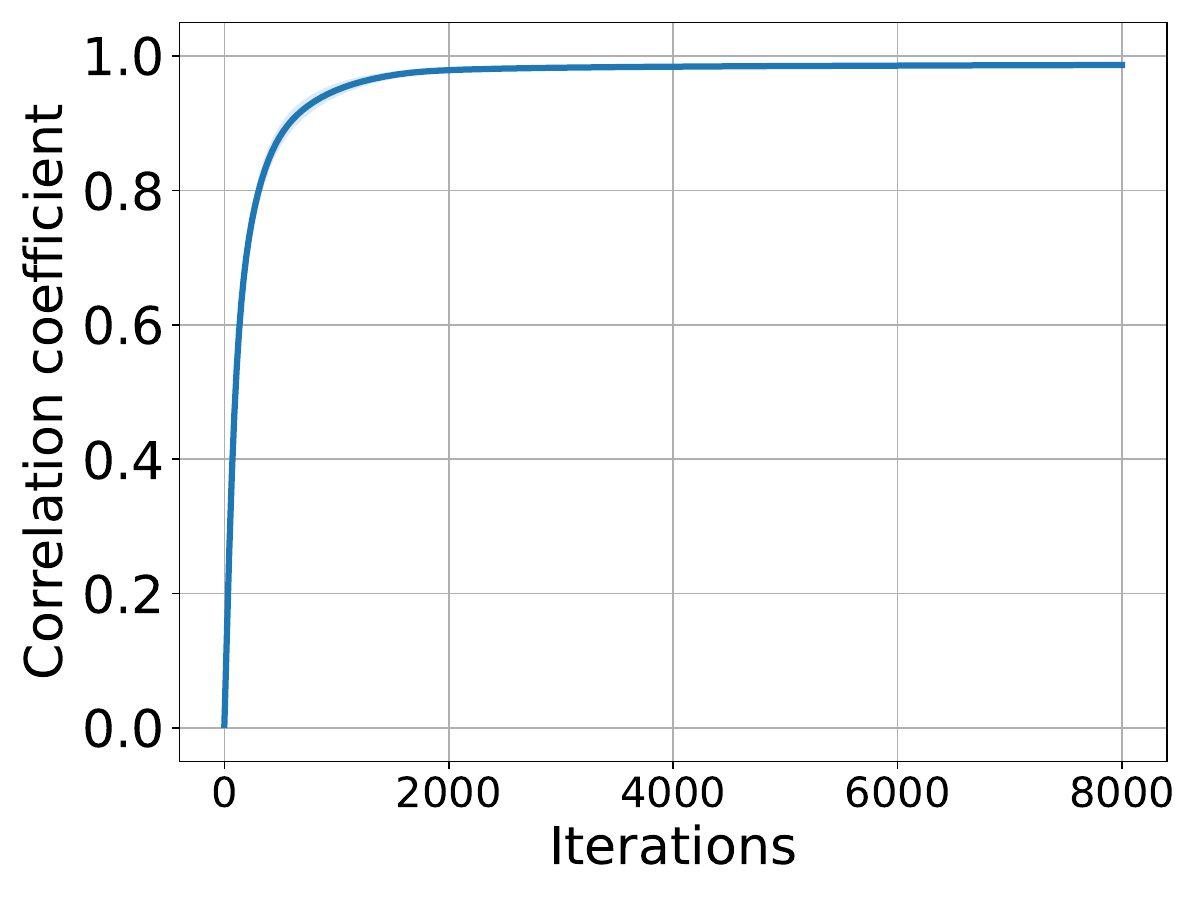}};
  \node at (0.2,-2.5) {{Iterations ($\tau$)}};
  \node[rotate=90] at (-3.4,0) {{Correlation coefficient}};
  \end{tikzpicture}
\vspace{-7pt}
\caption{{
$\frac{\W(\tau)}{\tf{\W(\tau)}}\to\frac{\Wm}{\tf{\Wm}}$
}}\label{fig large corr}
\end{minipage}
\begin{minipage}{0.45\linewidth}
\begin{tikzpicture}
  \node at (0,0){\includegraphics[height=.6\linewidth, trim={1.3cm 1.3cm 0 0}, clip]{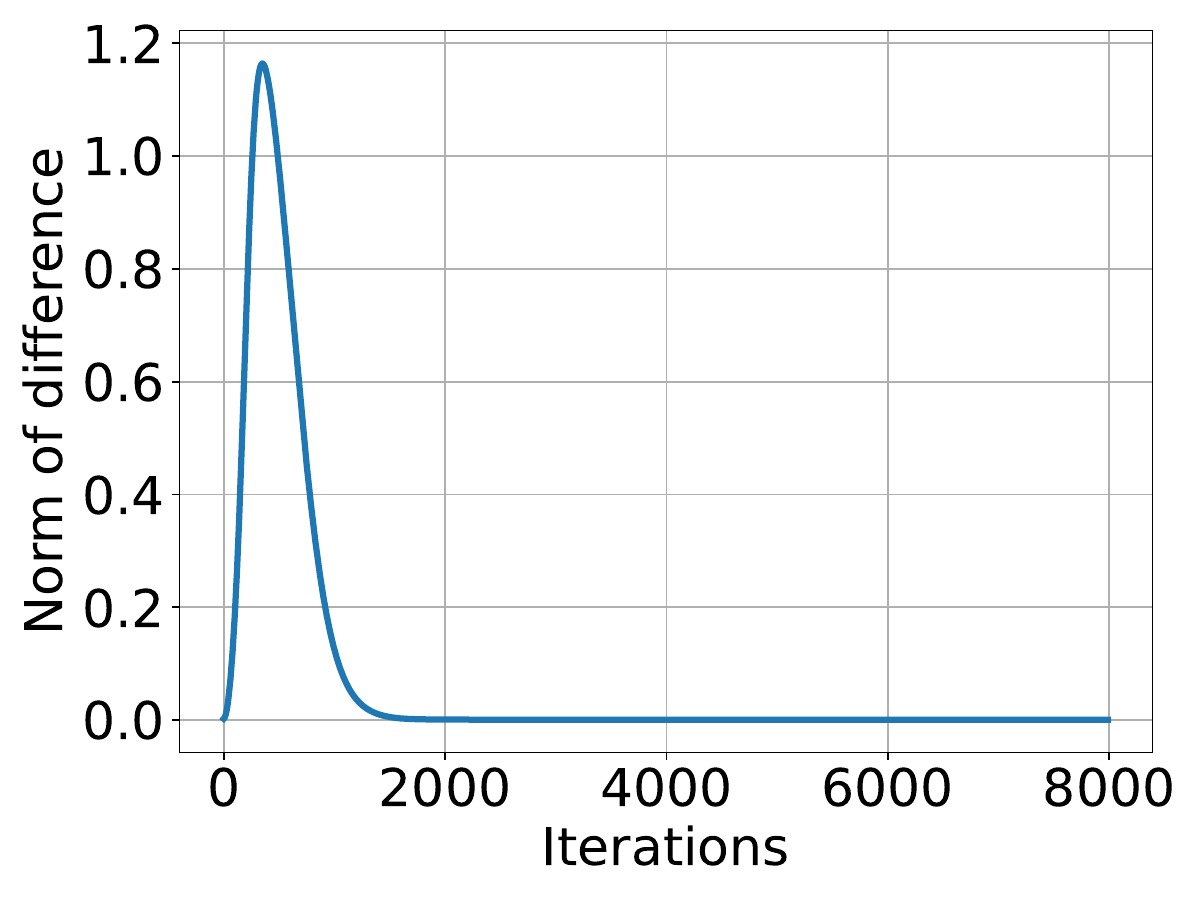}};
  \node at (0.2,-2.5) {{Iterations ($\tau$)}};
  \node[rotate=90] at (-3.7,0) {{$\tf{\hWf-\Wf}$}};
  \end{tikzpicture}
\vspace{-7pt}
\caption{{ 
$\prj_{\Scf}(\W(\tau))\to\Wf$}
}\label{fig large diff}
\end{minipage}

\vspace{14pt}
\begin{minipage}{0.45\linewidth}
\begin{tikzpicture}
  \node at (0,0){\includegraphics[height=.62\linewidth, trim={1.3cm 1.3cm 0 0}, clip]{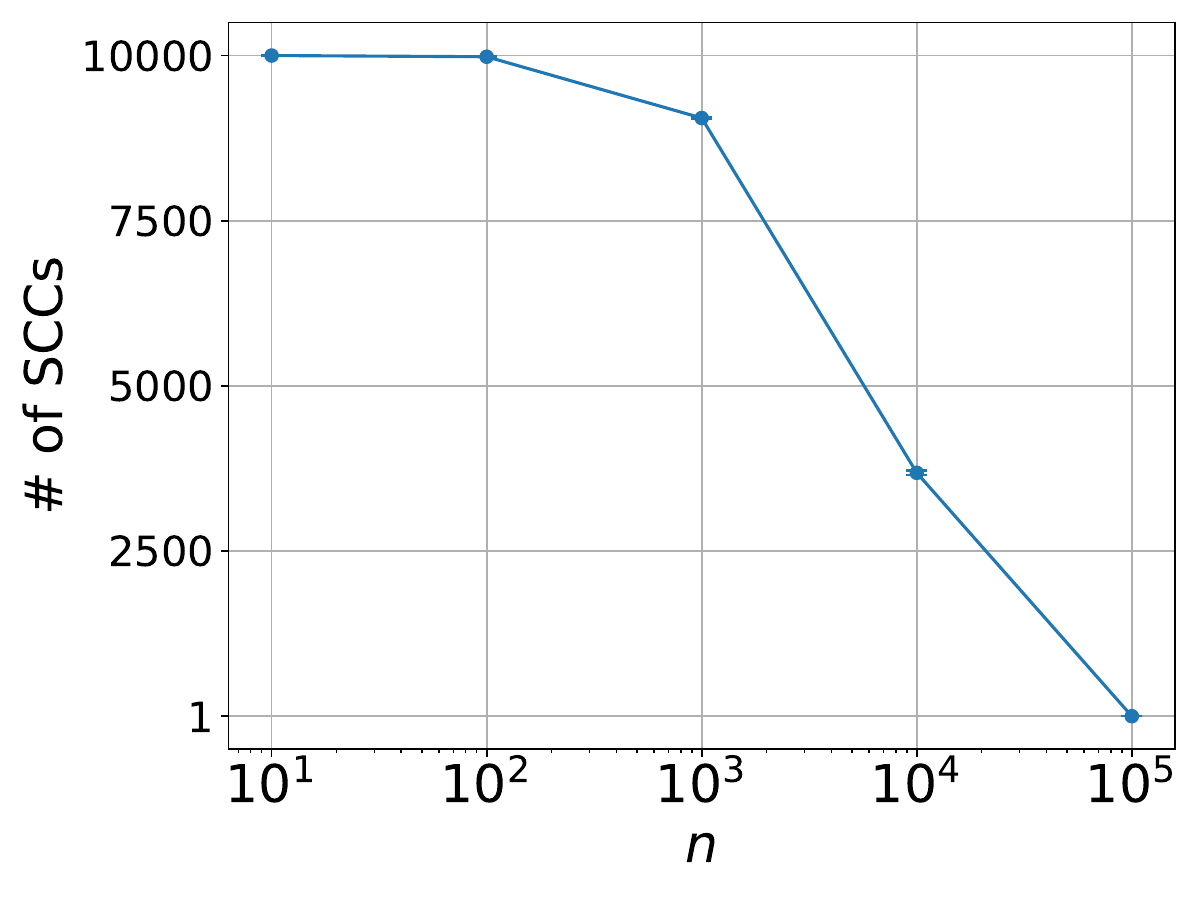}};
  \node at (0.2,-2.5) {{$n$}};
  \node[rotate=90] at (-3.4,0) {{\# of SCCs}};
  \end{tikzpicture}
\vspace{-7pt}
\caption{{Number of of SCCs vs $n$ 
}}\label{fig sccs}
\end{minipage}
\hspace{8pt}
\begin{minipage}{0.45\linewidth}
\begin{tikzpicture}
  \node at (0,0){\includegraphics[height=.61\linewidth, trim={1.3cm 1.3cm 0 0}, clip]{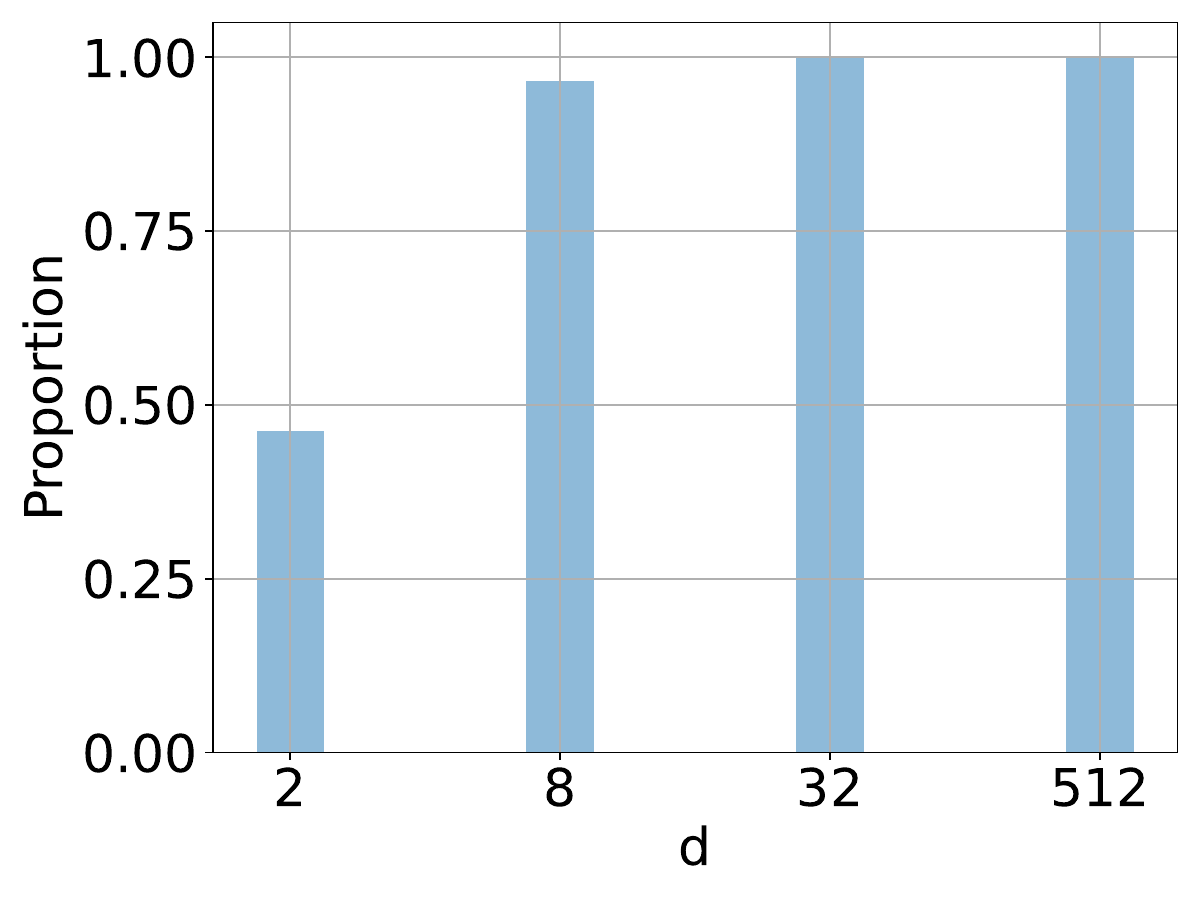}};
  \node at (0.2,-2.5) {{$d$}};
  \node[rotate=90] at (-3.5,0) {{Proportion}};
  \end{tikzpicture}
\vspace{-7pt}
\caption{{Feasibility of $\Wm$
}}\label{fig feasible}
\end{minipage}

\vspace{-15pt}
\end{figure}

\subsection{Additional Experiments}

\paragraph{Global convergence experiments on large $K$ (Figures~\ref{fig large corr} and~\ref{fig large diff}).} 
Assumption~\ref{assume iden} in our work requires $K\leq d$, which helps make the optimization landscape more benign such that global convergence of GD is guaranteed. Unlike previous work \cite{tarzanagh2023margin,tarzanagh2023transformers} that relies on strong equal score conditions to induce global convergence, our assumption is much less strict. Empirically, we argue that this constraint is not necessary as we can apply a mask $\M \in \R^{n \times K \times T}$ to directly collect the attention probability for each distinct token from the attention map without explicitly calculating the linear head. Therefore, we can still impose Assumption~\ref{assume iden} when $K > d$, which aligns more closely with the real-world setting. In Figs.~\ref{fig large corr} and \ref{fig large diff}, we repeat the global convergence experiments by setting $n = 16,T = 64, d=128$ and $K=10000$. Results are averaged over 100 random instances. The averaged correlation is $\approx 0.987$ and the soft component error reaches $0.025$. The results again validate Theorem \ref{thm cyclic gd}.

\paragraph{SCC Structure on large $n$ (Figure~\ref{fig sccs}).} When we increase the sample size $n$ and fix others, more edges and cycles will be added to the graphs. Different SCCs can then be merged into one. As illustrated in Fig. \ref{fig sccs}, the graph eventually collapses to a single SCC.

\paragraph{Feasible condition of \eqref{graph svm} (Figure~\ref{fig feasible}).} To verify that \eqref{graph svm} is feasible when $d \geq K$ (Lemma~\ref{lemma feasible}), in Fig.~\ref{fig feasible}, we run experiments with fixed $n=16, T=128, K=512$ and varying $d$ from $2$ to $512$. Define $\Cc_y$ as the SCC that the label token belongs to. We calculate the proportion of selected tokens that are in $\Cc_y$ to the size of $\Cc_y$, and \eqref{graph svm} is feasible when the value reaches $1$.
The interpretation is that: When $d$ is small, the problem focuses on separating an optimally feasible subset of training data from the others and the empirical SVM bias is captured by a relaxed Graph-SVM solution with constraints based on the subset. As $d$ grows, the exact Graph-SVM becomes feasible. This is similar to the findings in \cite{ji2019risk}.